\documentclass{article}

\usepackage{microtype}
\usepackage{graphicx}
\usepackage{subfigure}
\usepackage{booktabs} 

\usepackage{hyperref}

\usepackage[accepted]{icml2023}

\newcommand{\norm}[1]{\left\lVert#1\right\rVert}

\def\mA{{\mathbf{A}}}

\def\mC{{\mathbf{C}}}
\def\mD{{\mathbf{D}}}

\def\mH{{\mathbf{H}}}
\def\mI{{\mathbf{I}}}
\def\mJ{{\mathbf{J}}}

\def\mL{{\mathbf{L}}}
\def\mM{{\mathbf{M}}}
\def\mN{{\mathbf{N}}}

\def\mP{{\mathbf{P}}}
\def\mQ{{\mathbf{Q}}}
\def\mR{{\mathbf{R}}}
\def\mS{{\mathbf{S}}}
\def\mT{{\mathbf{T}}}
\def\mU{{\mathbf{U}}}
\def\mV{{\mathbf{V}}}
\def\mW{{\mathbf{W}}}
\def\mX{{\mathbf{X}}}
\def\mY{{\mathbf{Y}}}
\def\mZ{{\mathbf{Z}}}

\def\vzero{{\mathbf{0}}}

\def\vb{{\mathbf{b}}}

\def\vp{{\mathbf{p}}}

\def\vs{{\mathbf{s}}}

\def\vu{{\mathbf{u}}}

\def\vx{{\mathbf{x}}}

\def\vz{{\mathbf{z}}}

\def\gE{{\mathcal{E}}}

\def\gG{{\mathcal{G}}}

\def\gN{{\mathcal{N}}}

\def\gT{{\mathcal{T}}}

\def\gV{{\mathcal{V}}}

\def\sR{{\mathbb{R}}}

\usepackage{amsmath}
\usepackage{amssymb}
\usepackage{mathtools}
\usepackage{amsthm}

\usepackage[capitalize,noabbrev]{cleveref}

\theoremstyle{plain}
\newtheorem{theorem}{Theorem}[section]
\newtheorem{proposition}[theorem]{Proposition}
\newtheorem{lemma}[theorem]{Lemma}
\newtheorem{corollary}[theorem]{Corollary}
\theoremstyle{definition}
\newtheorem{definition}[theorem]{Definition}

\theoremstyle{remark}

\usepackage[textsize=tiny]{todonotes}

\icmltitlerunning{Randomized Schur Complement Views for Graph Contrastive Learning}

\begin{document}

\twocolumn[
\icmltitle{Randomized Schur Complement Views for Graph Contrastive Learning}



\icmlsetsymbol{equal}{*}

\begin{icmlauthorlist}
\icmlauthor{Vignesh Kothapalli}{yyy}
\end{icmlauthorlist}

\icmlaffiliation{yyy}{Courant Institute of Mathematical Sciences, New York University, New York, USA}

\icmlcorrespondingauthor{Vignesh Kothapalli}{vk2115@nyu.edu}

\icmlkeywords{Graph Contrastive Learning, Graph Neural Networks, Schur Complements, Randomized Numerical Linear Algebra}

\vskip 0.3in
]



\printAffiliationsAndNotice{}  

\begin{abstract}
We introduce a randomized topological augmentor based on Schur complements for Graph Contrastive Learning (GCL). Given a graph laplacian matrix, the technique generates unbiased approximations of its Schur complements and treats the corresponding graphs as augmented views. We discuss the benefits of our approach, provide theoretical justifications and present connections with graph diffusion. Unlike previous efforts, we study the empirical effectiveness of the augmentor in a controlled fashion by varying the design choices for subsequent GCL phases, such as encoding and contrasting. Extensive experiments on node and graph classification benchmarks demonstrate that our technique consistently outperforms pre-defined and adaptive augmentation approaches to achieve state-of-the-art results.
\end{abstract}

\section{Introduction}

Understanding the structural properties and semantics of graph data typically requires domain expertise and efficient computational tools.  Advances in machine learning techniques such as Graph Neural Networks (GNN) \citep{gori2005new, scarselli2008graph, bruna2014spectral, henaff2015deep, welling2016semi, niepert2016learning, bronstein2017geometric, xu2018powerful, wu2020comprehensive, zhou2020graph} have paved a path for representation learning on internet scale graphs and significantly alleviated the human effort. However, real-world graphs such as social networks, citation networks, supply chains and media networks are continuously evolving, which makes labeling an extremely challenging task to accomplish. Self-supervised learning (SSL) addresses this issue by optimizing pretext objectives and learning generalizable representations for downstream tasks \citep{jin2020self, wu2021self, liu2022graph, xie2022self}. This learning paradigm has been quite popular for vision \citep{gidaris2018unsupervised, hjelm2018learning, chen2020simple, jing2020self}, language domains \citep{mikolov2013efficient, devlin2018bert, radford2018improving, lan2019albert} and is relatively new to graphs.

Contrastive learning on graphs \citep{velickovic2019deep, sun2019infograph, zhu2020deep, you2020graph, hassani2020contrastive} is a variant of SSL, whose pretext task is aimed at maximizing representation agreement across augmented views. A typical GNN-based GCL framework comprises three main components: 1. Data augmentors, 2. GNN encoders and 3. Contrastive objectives with corresponding modes. Data augmentation techniques in the literature can be categorized based on: 1. Feature masking/perturbations \citep{you2020graph, thakoor2021bootstrapped}, 2. Structural perturbations \citep{you2020graph, hassani2020contrastive, zeng2021contrastive}, 3. Sub-graph sampling \citep{hu2019strategies, qiu2020gcc, jiao2020sub, zhu2021transfer}, 4. Adaptive and learnable structural perturbations \citep{zhu2021graph, yin2022autogcl}. Efforts leveraging these techniques tend to focus on the collective comparison of GCL frameworks and lack controlled experimentation pertaining to the augmentation phase. 

\begin{figure*}[ht]
\vskip 0.2in
\begin{center}
\centerline{\includegraphics[width=0.9\textwidth]{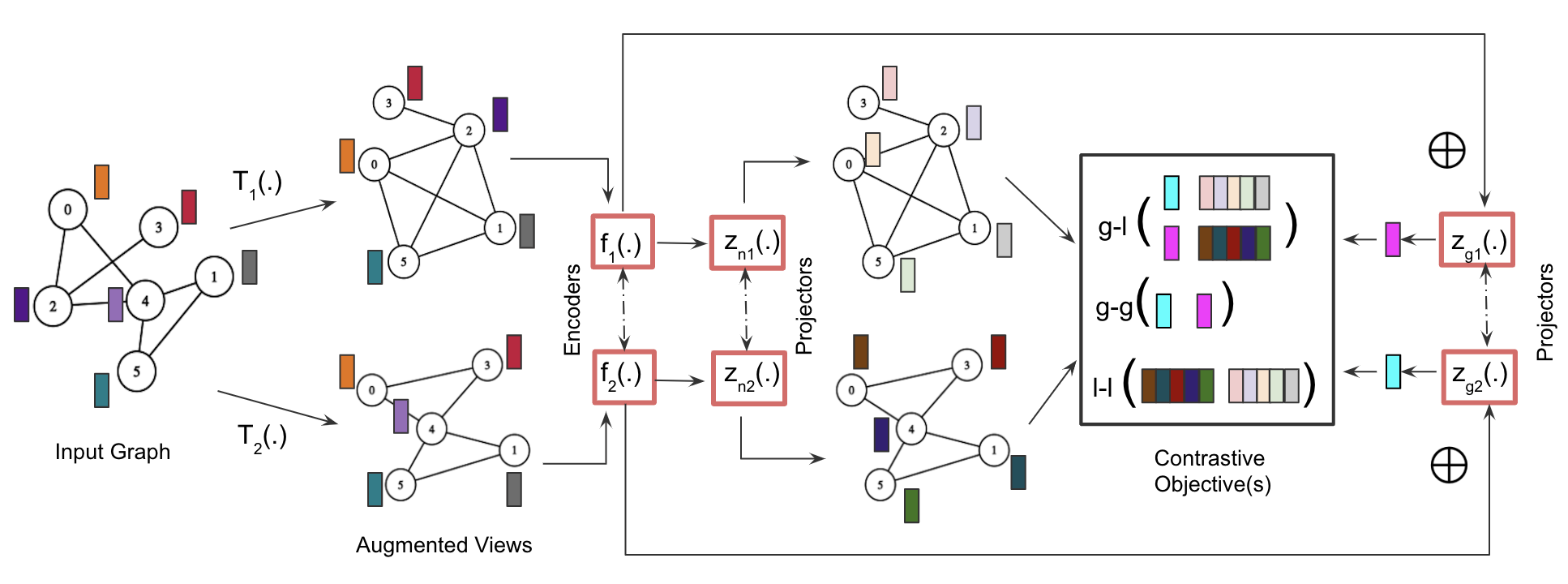}}
\caption{A generalized dual branch GCL framework. The input graph $\gG=(\gV, \gE, w)$ is augmented by $T_1, T_2$ to generate graph views. These graph views are encoded using GNNs $f_1, f_2$ to compute node embeddings. The node embeddings are projected using MLPs $z_{n1}, z_{n2}$ to generate node features. Additionally, the node embeddings are aggregated using a permutation invariant function $\bigoplus$ and projected using $z_{g1}, z_{g2}
$ to compute graph-level features. Depending on the contrastive mode, the graph and node features are contrasted.} 
\label{fig:gcl}
\end{center}
\vskip -0.2in
\end{figure*}

A recent empirical study on GCL by \citet{zhu2021empirical} partially addresses this issue and compares some of the widely used augmentors with a fixed encoder design and contrastive objective. In particular, the study focuses on techniques such as feature masking, node dropping, edge perturbations, sub-graph sampling based on random walks, diffusions based on Personalized Page Rank (PPR), Heat Kernel (HK) and Markov Diffusion Kernel (MDK). From a topological augmentation perspective, pre-defined stochastic techniques such as node dropping and edge perturbations  tend to be faster in practice but fail to preserve structural information of the original graph \citep{hubler2008metropolis}. Surprisingly, views obtained by such simple techniques have been shown to outperform diffusion-based strategies when contrasting solely based on node embeddings. Despite this parity, \citet{zhu2021empirical} show that, when diffusion-based approaches are combined with node dropping, there is a substantial increase in the unsupervised node and graph classification performance. An attempt to avoid such trial and error based design of augmentors was made by \citet{yin2022autogcl}, where learnable augmentors are trained in an end-to-end fashion to generate semantic label preserving views. Such a technique is limited to node-dropping operations and induces non-negligible computational overheads if edge perturbations were to be included. Furthermore, since learnable augmentors modify learning objectives, leveraging them to explore and design new GCL objectives can be challenging. On the other hand, the adaptive augmentation technique proposed by \citet{zhu2021graph} drops edges based on their degree, page rank or eigenvector centrality scores and preserves pre-defined structural information across views. However, metrics such as eigenvector centrality are expensive to compute and negate the performance benefits of augmentors. These observations highlight the persisting augmentation problem of \textbf{identifying} and \textbf{preserving} structural properties across \textbf{stochastically} generated views in a \textbf{computationally efficient} manner.

In our work, we leverage the properties of Schur complements to address these issues. Schur complements are typically obtained during Gaussian Elimination (GE) and are pervasive in numerical analysis, probability and statistics \citep{cottle1974manifestations, ouellette1981schur, zhang2006schur}. From a graph theoretic perspective, GE essentially leads to node dropping and edge perturbations such that the resulting Schur complements preserve random walk transition probabilities of the remaining nodes (with respect to the original graph). Our technique exploits this property and generates randomized yet unbiased approximations of Schur complements as the augmented views. Additionally, we exploit the role of Schur complements in matrix inversions and showcase computational optimizations for diffusion-based augmentors. Overall, the potential benefits of this class of topological augmentors remain unexplored for GCL and we take a step towards understanding them. To summarize, our contributions are as follows:
\begin{itemize}
    \item We design a randomized Schur complement based topological augmentor for GCL which is fast, stochastic and effective on a wide range of framework designs.
    \item We show that by preserving the combinatorial properties of Schur complements in expectation, the augmentor achieves state-of-the-art results on unsupervised node and graph classification tasks.
    \item Using our technique, we present an efficient approach to achieve diffusion followed by sub-graph sampling.
\end{itemize}

\section{Preliminaries}
\subsection{Notations}
Let $\gG=(\gV, \gE, w)$ denote an undirected graph with nodes $\gV = \{v_i\}_{i=1}^N$ and edges $\gE \subseteq \gV \times \gV$. An edge between nodes $v_i, v_j$ is given by $e_{v_iv_j}=e_{ij} \in \gE$. We drop the subscript in $e_{ij}$ when indexing can be avoided. The neighborhood of $v_i$ is given by $\gN(v_i)$. For notational convenience, $v_i \in e$ represents a node $v_i \in \gV$ being on either end of an edge $e \in \gE$. A weight function $w$ maps an edge $e_{ij} \in \gE$ to its weight $w(e_{ij}) = w(v_i, v_j) \in \sR$. Also, $w$ is overloaded to denote the total weight of a node as the sum of weights of edges incident on it: $w(v_i) = \sum_{e\in \gE, v_i \in e}w(e)$. Similarly, the degree of a node $deg(v_i) = |\{e \in \gE: v_i \in e\}|$ is the total number of edges incident on it. The feature matrix of $\gG$ is given by $\mX \in \sR^{N \times d}$ where features of a node $v_i$ are given by the row $\vx_i \in \sR^d$. Finally, we consider the graph $\gG$ to be associated with a weighted positive semi-definite laplacian matrix $\mL \in \sR^{N \times N}$. 


\subsection{Graph Contrastive Learning Frameworks}

In this section, we briefly describe the functionality of GCL framework components (see Figure \ref{fig:gcl}) and defer additional details on the mathematical formulations to Appendix \ref{appendix:experimental_setup}.

\subsubsection{Augmentors}

Without loss of generality, let $\gT$ represent a class of augmentation functions from which $T_1, T_2 \sim \mathcal{T}$ are chosen. These functions transform $\gG$ to the respective augmented views $T_1(\gG) = \widetilde{\gG}_1, T_2(\gG) = \widetilde{\gG}_2$ by perturbing the topology, features or both. \citet{hassani2020contrastive} observed that, when $\gT$ represents a class of diffusion-based augmentors, contrasting more than 2 views didn't improve the performance on the downstream node and graph classification tasks. Additionally, the approach of \citet{velickovic2019deep} uses only a single view to contrast the node level and graph level embeddings. Analyzing the effect of the number of views based on choices of $\mathcal{T}$ is still an open problem. Meanwhile, for the rest of this paper, we default it to 2.

\subsubsection{Encoders and Projectors}

Depending on design choices for the framework, one can choose either a shared encoder \citep{you2020graph} or dedicated encoders \citep{hassani2020contrastive} to compute the node embeddings for $\widetilde{\gG}_1, \widetilde{\gG}_2$. Typically, GNNs based on Graph Convolution Network (GCN) \citep{kipf2016semi} or Graph Isomorphism Network (GIN) \citep{xu2018powerful} are employed. We denote the node level embeddings of $\widetilde{\gG}_1, \widetilde{\gG}_2$ as $\mH_{\widetilde{\gG}_1}, \mH_{\widetilde{\gG}_2} \in \sR^{N \times d_h}$. Next, an MLP-based projector is employed to project the embeddings onto a latent space. We denote the projected node level features of $\widetilde{\gG}_1, \widetilde{\gG}_2$ as $\mZ_{\widetilde{\gG}_1}, \mZ_{\widetilde{\gG}_2} \in \sR^{N \times d_z}$. Additionally, the node embeddings are aggregated using a permutation invariant function $\bigoplus$ (eg: mean) and projected to compute graph level features\footnote{Designs such as MVGRL \cite{hassani2020contrastive} employ MLP projections to improve the graph level representations.} $\vz_{\widetilde{\gG}_1}, \vz_{\widetilde{\gG}_2} \in \sR^{d_z}$. 

\subsubsection{Contrastive modes and objectives}

With $\mZ_{\widetilde{\gG}_1}, \mZ_{\widetilde{\gG}_2} \in \sR^{N \times d_z}$ and $\vz_{\widetilde{\gG}_1}, \vz_{\widetilde{\gG}_2} \in \sR^{d_z}$ being available, a relevant contrastive objective is chosen for optimization. Popular choices include: Noise Contrastive Estimation based InfoNCE \cite{oord2018representation}, Jensen-Shannon Divergence (JSD) \cite{lin1991divergence}, Bootstrapping Latent (BL) \cite{grill2020bootstrap}, Variance-Invariance-Covariance Regularization (VICReg) \cite{bardes2021vicreg} and Barlow Twins (BT) \cite{zbontar2021barlow} losses. Finally, a mode is chosen to determine the granularity of contrasting:

\begin{enumerate}
    \item \textbf{local-local (l-l)}: The objective is solely dependent on the node-level features.
    \item \textbf{global-local (g-l)}: The objective is dependent on both node and graph-level features.
    \item \textbf{global-global (g-g)}: The objective is solely dependent on the graph-level features.
\end{enumerate}

The choice of mode and objective not only affects the GCL framework performance but also plays a key role in computational efficiency. For instance, \citet{zhu2021empirical} empirically show that, in local-local mode, losses such as BL, BT, and VICReg consume $\approx 3\times$ less memory than InfoNCE and JSD while achieving comparable performance on node and graph classification tasks. Finally, when it comes to objectives such as InfoNCE and JSD, choosing the optimal set of negative samples is not straightforward. Especially, the difficulty lies in achieving a balance between computational overheads of negative sampling \citep{chuang2020debiased} (with hard negatives if required by design \citep{robinson2020contrastive}) and performance improvements.

\subsection{Schur Complements and Gaussian Elimination }

Every edge $e_{ij} \in \gE$ of graph $\gG$ forms an elementary laplacian given by the outer product:
$\mathbf{\Delta}_{ij} = (\delta_i-\delta_j)(\delta_i-\delta_j)^*
$. Here, $\delta_i \in \sR^N$ denotes the standard basis vector at index $i$ and $^*$ denotes the adjoint. The weighted Laplacian $\mL$ can now be formulated as follows:
\begin{equation}
\label{eq:lap}
\mL = \sum_{e_{ij} \in \gE}w(e_{ij})\mathbf{\Delta}_{ij}
\end{equation}
Since every $\mathbf{\Delta}_{ij}$ is positive semi-definite, so is their weighted sum $\mL$. The off-diagonal elements of $\mL$ are given by: $(\mL)_{ij} = -w(e_{ij}), i \neq j$ and the diagonal elements by $(\mL)_{ii} = w(v_i) = \sum_{v_j \in \mathcal{N}(v_i)} w(e_{ij})$. Now, without loss of generality, consider the matrix representation of $\mL$ as:
\begin{equation}
\mL = \begin{bmatrix}
a_{11} & \vb^{*}\\
\vb & \mL_2
\end{bmatrix}
\end{equation}
Where $a_{11} \in \sR, \vb \in \sR^{N-1}$ and $\mL_2 = \sR^{N-1 \times N-1}$. The Schur complement of $\mL$ with respect to nodes $\gV \backslash v_1$ is a positive semi-definite matrix that is obtained when $v_1$ is eliminated via a GE step:
\begin{align}
\begin{split}
\label{eq:sc_matrix}
    SC(\mL, \gV \backslash v_1) &= \mL - \frac{1}{a_{11}}\begin{bmatrix}a_{11}\\ \vb \end{bmatrix}\begin{bmatrix}a_{11}\\ \vb \end{bmatrix}^* \\
&= \begin{bmatrix}0 & \vzero^*\\ \vzero & \mL_2 - a_{11}^{-1}\vb\vb^*\end{bmatrix}
\end{split}
\end{align}
Alternatively, a graph theoretic interpretation of GE states that $SC(\mL, \gV \backslash v_1)$ is obtained by removing the $\textit{Star}$ graph corresponding to $v_1$ from $\gG$ and adding the induced $\textit{Clique}$ graph. Formally:
\begin{equation}
\label{eq:sc_star_clique}
    SC(\mL, \gV \backslash v_1) = \mL - \textit{STAR}(\mL, v_1) + \textit{CLIQUE}(\mL, v_1)
\end{equation}
Where $\textit{STAR}(\mL, v_i), \textit{CLIQUE}(\mL, v_i)$ for any node $v_i \in \gV$ are given by:
\begin{align}
\label{eq:star_clique}
\begin{split}
\textit{STAR}(\mL, v_i) & = \sum_{e_{ij}\in \gE, v_j \in \mathcal{N}(v_i)}w(e_{ij})\mathbf{\Delta}_{ij} \\
\textit{CLIQUE}(\mL, v_i) & = \frac{1}{2w(v_i)}\sum_{e_{ij}\in \gE}\sum_{e_{ik}\in \gE}w(e_{ij})w(e_{ik})\mathbf{\Delta}_{jk}    
\end{split}
\end{align}

This result indicates that $SC(\mL, \gV \backslash v_1)$ is a weighted Laplacian matrix with an associated graph. For simplicity, we denote $SC(\mL, \gV \backslash v_1)$ as $SC(\gG, \gV \backslash v_1)$ when the notion of a graph is suitable for the context. Schur complements hold an interesting graph theoretic property that, the random walk transition probabilities of the remaining nodes $\gV \backslash v_1$ through the eliminated vertex $v_1$ (with respect to $\gG$) are preserved in $SC(\gG, \gV \backslash v_1)$. Intuitively, the list of nodes visited by random walks on $SC(\gG, \gV \backslash v_1)$ is equivalent in distribution to the list of nodes in $\gV \backslash v_1$ visited by random walks on $\gG$. Refer to section 3 in \citet{durfee2019fully} and section 4.1 in \citet{gao2022fully} for detailed discussions on this property.

\section{Methodology}

To leverage these combinatorial properties of Schur complements in our contrastive views, two issues need to be addressed: The $O(N^2)$ overhead to compute a clique and a lack of stochasticity in the gaussian elimination procedure. We address both these issues by developing a clique approximation procedure that inherently introduces randomness. Thus, by eliminating nodes via GE and computing unbiased approximations of the induced cliques, our augmentation technique generates randomized Schur complements which preserve the random walk transition probabilities of the remaining nodes in expectation.

The need for stochasticity is based on the empirical results of \citet{zhu2021empirical}, where introducing randomness to the views (eg: combining diffusion with node dropping) improved contrastive learning performance on downstream classification tasks. Although a rigorous analysis of such behavior is not present in the literature, we provide interesting insights on GCL performance due to randomness introduced by our approximation procedure.



\subsection{Randomized Schur Complements}

We present our Schur complement approximation procedure in Algorithm \ref{alg:rlap}.\footnote{The code is available at: \href{https://github.com/kvignesh1420/rlap}{https://github.com/kvignesh1420/rlap}} We name it $rLap$ to denote the laplacian nature of its output. The input to $rLap$ is the graph $\gG$, the fraction of nodes to eliminate $\gamma$, node elimination scheme $o_v$ and the neighbor ordering scheme $o_n$. The output is a randomized Schur complement of $\gG$ after eliminating $\gamma|\gV|$ nodes\footnote{Elimination in this context indicates that a node is disconnected from all its neighbors. The dimensions of $\mR$ remain $\sR^{N \times N}$ with corresponding row and column set to $\vzero$.}. The scheme $o_v$ indicates the order in which $\gamma|\gV|$ nodes are eliminated. The possibilities for such a scheme are huge but we limit our analysis to random ordering and ordering based on node degree. In the `random' scheme, a node is randomly selected at each step of the outer loop. In the `degree' scheme, a priority queue is maintained to eliminate nodes based on their degree, i.e. nodes with lower degrees are eliminated before the highly connected ones.

\begin{algorithm}[tb]
  \caption{$rLap$}
  \label{alg:rlap}
\begin{algorithmic}
  \STATE {\bfseries Input:} graph $\gG = (\gV, \gE, w)$, node drop fraction $\gamma$,  node elimination scheme $o_v$, neighbor ordering scheme $o_n$.
  \STATE \textbf{set} $\mR = \mL$
  \REPEAT
  \STATE $\mathcal{X}_{\mR}(v_i) = o_n(\mathcal{N}_{\mR}(v_i))$
  \STATE $\mC = \textbf{0}$
  \FOR{$l=1$ {\bfseries to} $|\mathcal{X}_{\mR}(v_i)|-1$}
  \STATE // \textit{choose $x_q$ based on the conditional probability}
  \STATE $P(x_q|x_l) = \frac{w_{\mR}(v_i, x_q)}{w_{\mR}(v_i) - \sum\limits_{k=1}^l w_{\mR}(v_i, x_k)}; x_{l<q} \in \mathcal{X}_{\mR}(v_i)$
  \STATE // \textit{compute weight of the new edge between $x_l, x_q$}
  \STATE $w_{\mR}(x_l, x_q) = \frac{w_{\mR}(v_i, x_l) \cdot(w_{\mR}(v_i) - \sum\limits_{k=1}^l w_{\mR}(v_i, x_k))}{w_{\mR}(v_i)}$
  \STATE $\mC = \mC + w_{\mR}(x_l, x_q)\cdot \mathbf{\Delta}_{x_lx_q}$
  \ENDFOR
  \STATE $\mR = \mR - \sum\limits_{e_{ij}\in \gE_{\mR}, v_j \in \mathcal{N}_{\mR}(v_i)}w_{\mR}(e_{ij})\mathbf{\Delta}_{ij}$ 
  \STATE $\mR = \mR + \mC$
  \UNTIL{all the $\gamma|\gV|$ nodes are eliminated based on $o_v$}
  \STATE \textbf{return} $\mR$
\end{algorithmic}
\end{algorithm}

Now, without loss of generality, consider the $i^{th}$ iteration of the outer loop where node $v_i$ is being eliminated. The first step is to order the neighbors $\mathcal{N}_{\mR}(v_i)$ and store the result in $\mathcal{X}_{\mR}(v_i)$ \footnote{$\gE_{\mR}, w_{\mR}, \mathcal{N}_{\mR}$ are indexed by $\mR$ to denote the current state.}. The order decided by $o_n$ affects the arrangement of edges in the approximated clique $\mC$. Specifically, while iterating over the neighbors $x_l \in \mathcal{X}_{\mR}(v_i), l \in \{1, \cdots, |\mathcal{X}_{\mR}(v_i)|-1\}$ in the inner loop, we compute a conditional probability $P(x_q|x_l)$ of choosing a neighbor $x_q \in \mathcal{X}_{\mR}(v_i), q \in \{l+1, \cdots, |\mathcal{X}_{\mR}(v_i)|\}$, based on which, a new edge between $x_l, x_q$ is created. The choice of $o_n$ affects this conditional probability and can lead to sparser or denser variants of approximated cliques. Now, the weighted elementary laplacian  of the newly formed edge $w_{\mR}(x_l, x_q) \mathbf{\Delta}_{x_lx_q}$ is computed and added to $\mC$ \footnote{Here $x_l, x_q$ are indexed on $\mathcal{X}_{\mR}(v_i)$ for notational convenience, but have a one-to-one correspondence with our original node set $\gV$}. Note that, the inner loop iterates over $O(N)$ nodes for approximating the clique and avoids the $O(N^2)$ overhead for exact computation. Finally, $\textit{STAR}(\mR, v_i)$ is subtracted and $\mC$ is added to $\mR$ to continue the elimination process.

\begin{theorem}
\label{thm:rlap}
Given an undirected graph $\gG=(\gV, \gE, w)$, node dropping fraction $\gamma$, node elimination scheme $o_v$ and neighbor ordering scheme $o_n$, the output of $rLap(\gG, \gamma, o_v, o_n)$ is an unbiased estimator of the schur complement $SC(\gG, \gV_{\mR_{\gamma|\gV|}})$, where $\gV_{\mR_{\gamma|\gV|}} \subset \gV$ denotes the set of nodes that remain after $\gamma|\gV|$ eliminations.
\end{theorem}
\begin{proof} Without loss of generality, let $\mR_{i-1}$ be the state after eliminating $i-1$ nodes. Let $v_i$ indicate the node which is being eliminated in the $i^{th}$ iteration of the outer loop. The proof is based on the loop invariant that $\mathbb{E}[\mR_i] = SC(\gG, \gV_{\mR_{i-1}}\backslash v_i)$ after the end of this iteration. With this guarantee, we continue the elimination process for $\gamma|\gV|$ iterations and achieve the desired randomized Schur complement which equals $SC(\gG, \gV_{\mR_{\gamma|\gV|}})$ under expectation.  The proof is available in Appendix \ref{app:poc_rlap} along with the tail bounds of deviation for the laplacian matrix martingale.
\end{proof}


\subsection{Connections with Graph Diffusion}

For a graph $\gG = (\gV, \gE, w)$, the diffusion operator is a polynomial filter on $\gG$ which mitigates the effect of noisy neighbors. The diffused output $\mS$ is given by:
\begin{equation}
\label{eq:diffusion}
    \mS = \sum_{i=0}^{\infty}\rho_i \mT^i
\end{equation}
Where $\mT$ is a transition matrix \citep{klicpera2019diffusion} and can be chosen as either $\mT_{rw} = \mA\mD^{-1}$ or $\mT_{sym} = \mD^{-1/2}\mA\mD^{-1/2}$, $\mA, \mD$ are the adjacency and degree matrices for $\gG$ respectively,  $\rho_i$ is the scaling coefficient with $\rho_i^{PPR} = \alpha \cdot (1 - \alpha)^i, \alpha \in (0,1)$ and $\rho_i^{HK} = e^{-t}\cdot \frac{t^i}{i!}$ representing the coefficients for personalized page rank \citep{page1999pagerank} and heat kernels with diffusion time $t$ \citep{kondor2002diffusion, chung2007heat} respectively. When $\mT = \mT_{rw} = \mA\mD^{-1}$, the closed forms of diffusions are given by:
\begin{align}
\label{eq:diffusion_rw}
\begin{split}
\mS^{PPR} &= \alpha(\mI - (1-\alpha)\mA\mD^{-1})^{-1} \\
\mS^{HK} & = \exp(t\cdot \mA\mD^{-1} - t\mI)  
\end{split}
\end{align}

Note that, for $\mS^{PPR}$, computing $(\mI - (1-\alpha)\mA\mD^{-1})^{-1}$ is an expensive operation over the entire graph and leads to undesired bottlenecks when the desired augmented view during contrasting is just a relatively smaller sub-graph. We address this issue in the following theorem using $rLap$.

\begin{theorem}
\label{thm:theta_schur}
Let $\rho_i = c\cdot \beta^i \in \sR, \beta \to 1$. The sub-graph with nodes $\gV_s$ sampled from $\mS=(\mI - (1-\alpha)\mA\mD^{-1})^{-1}$ is given by $\mS_{\gV_s} = \bigg(\lim\limits_{K \rightarrow \infty} \sum_{k = 0}^K c \cdot \beta^k \cdot (\mA\mD^{-1})^k \bigg)_{\gV_s} = \mathbb{E} \bigg[ \lim\limits_{K \rightarrow \infty} \sum_{k = 0}^K c \cdot \beta^k \cdot \mD_{\gV_s}\widetilde{\mD}_{\mR}^{-1}(\widetilde{\mA}_{\mR}\widetilde{\mD}_{\mR}^{-1})^k \bigg] $, where $\mR$ is the approximation of schur complement $SC(\gG, \gV_s), \gV_s \subset \gV$ given by $rLap$, $\mD_{\gV_s}$ is a diagonal matrix formed by selecting entries corresponding to $\gV_s$ from $\mD$ and $\widetilde{\mA}_{\mR}, \widetilde{\mD}_{\mR}$ represent the adjacency and degree matrices w.r.t $\mR$.
\end{theorem}
\begin{proof}
The proof is a simple expansion of the limit:

$\mathbb{E} \bigg[ \lim\limits_{K \rightarrow \infty} \sum_{k = 0}^K c\cdot \beta^k \mD_{\gV_s}\widetilde{\mD}_{\mR}^{-1} (\widetilde{\mA}_{\mR}\widetilde{\mD}_{\mR}^{-1})^k \bigg]$ \\ 
$= c \mD_{\gV_s} \mathbb{E} \big[ \widetilde{\mD}_{\mR}^{-1}\lim\limits_{K \rightarrow \infty} \sum_{k = 0}^K(\beta \widetilde{\mA}_{\mR}\widetilde{\mD}_{\mR}^{-1})^k \big]$\\ 
$= c \mD_{\gV_s} \mathbb{E} \big[ \widetilde{\mD}_{\mR}^{-1}\cdot(\mI - \beta \widetilde{\mA}_{\mR}\widetilde{\mD}_{\mR}^{-1})^{-1} \big]$\\
$= c \mD_{\gV_s} \mathbb{E} \big[ (\widetilde{\mD}_{\mR} - \beta \widetilde{\mA}_{\mR})^{-1} \big]$\\
$= c \mD_{\gV_s} \cdot SC(\mD - \beta \mA, \gV_s)^{-1}$ \\
$= c \mD_{\gV_s} (\mD - \beta \mA)^{-1}_{\gV_s} = \mS_{\gV_s}$

The equality $\mathbb{E}[(\widetilde{\mD}_{\mR} - \beta \widetilde{\mA}_{\mR})^{-1}] = SC(\mD - \beta \mA, \gV_s)^{-1}$
is based on the guarantees of Theorem \ref{thm:rlap} when $\beta \to 1$ and $\mD - \beta \mA \to \mL$. The equality $SC(\mD - \beta \mA, \gV_s)^{-1} = (\mD - \beta \mA)^{-1}_{\gV_s}$ is based on the standard matrix inversion lemma. Refer to \citet{gallier2010notes} and Appendix C.4 in \citet{boyd2004convex} for further details.
\end{proof}

When $\beta$ is not close to $1$, $\mD - \beta \mA$ doesn't represent a graph laplacian but instead represents a symmetric diagonally dominant matrix (SDDM). Similar to the laplacian property of Schur complements, it can be shown that Schur complements of SDDM matrices are also SDDM (refer Appendix A, Lemma A.1 in \citet{fahrbach2020faster} for the proof). To better understand the implications of Theorem \ref{thm:theta_schur}, consider a graph $\gG = (\gV, \gE, w)$ on which the PPR diffusion operator needs to be applied, followed by a sub-graph sampling operation with respect to nodes $\gV_s \subset \gV$. Such a requirement is pervasive in GCL when views are based on diffusion matrices. However, the overheads of computing $\mS$ for medium-large scale graphs can be significant. Based on the empirical analysis of \citet{klicpera2019diffusion}, the optimal choice of $\alpha$ for $\rho_i = \alpha \cdot (1 - \alpha)^i$ typically lies in a close range of $\alpha \in [0.05, 0.2]$. This implies $\beta = 1 - \alpha$ is close to $1$, which justifies our assumption in Theorem \ref{thm:theta_schur} for practical settings. Thus, instead of diffusion followed by sub-graph sampling, one can apply $rLap$ on $\gG$ to obtain the randomized Schur complement which is relatively sparse, followed by the diffusion operator (up to $\mD_{\gV_s}\widetilde{\mD}_{\mR}^{-1}$ scaling). See Figure \ref{fig:rlap_diffusion} for an illustration.

\begin{figure}[ht]
\vskip 0.2in
\centering
\centerline{\includegraphics[width=\columnwidth]{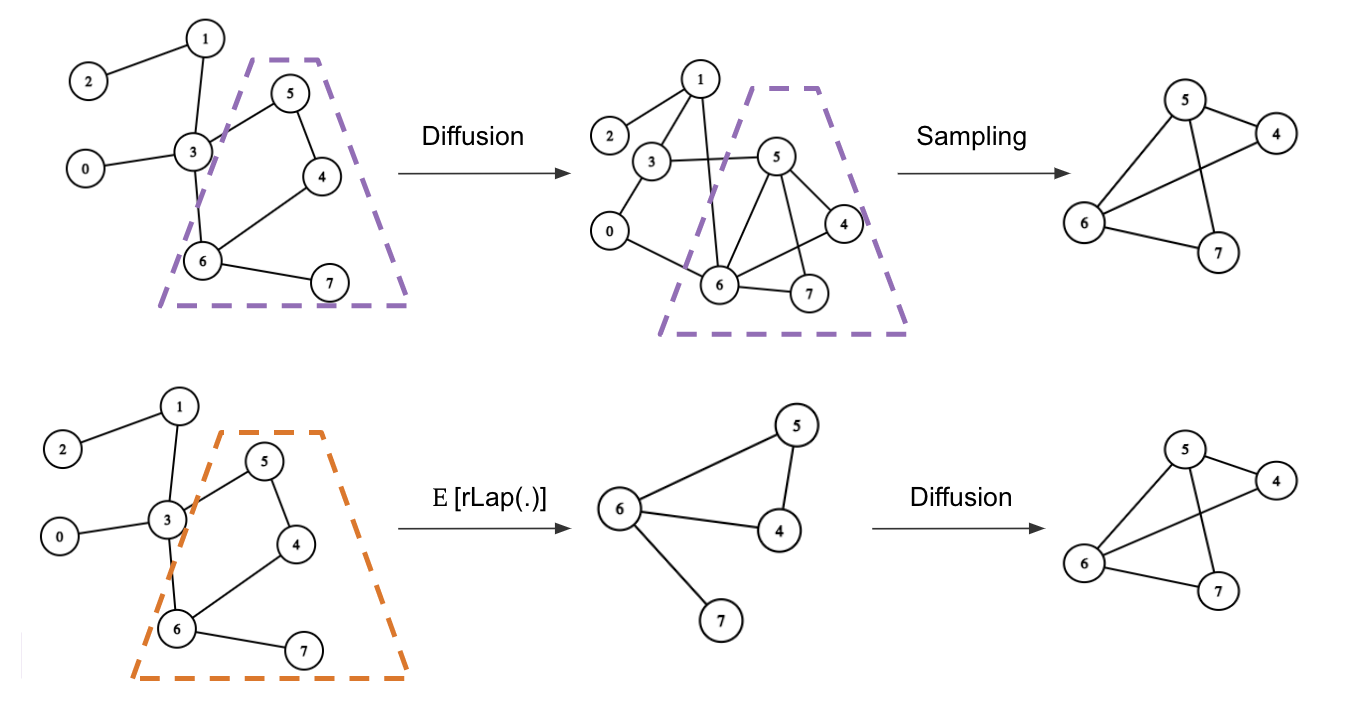}}
\caption{Illustration of diffusion + sampling vs $rLap$ + diffusion with a toy example. In the first row, diffusion is applied on a graph $\gG$ followed by sampling the sub-graph corresponding to nodes: $\gV_s = \{4,5,6,7\}$. In the second row, the expected output of $rLap$ on $\gG$ with $\gV_s = \{4,5,6,7\}$ followed by diffusion on the randomized Schur complement gives the same desired sub-graph.} 
\label{fig:rlap_diffusion}
\vskip -0.2in
\end{figure}

\subsection{Augmentation with rLap}

\begin{algorithm}[tb]
  \caption{A generalized GCL framework with $rLap$.}
  \label{alg:rlap_gcl}
\begin{algorithmic}
  \STATE {\bfseries Input:} Graph(s)  $G = \{\gG^i, i \in \{1, \cdots, M\}$, node elimination scheme $o_v$, neighbor ordering scheme $o_n$, fraction of nodes to eliminate for both views $\gamma_1, \gamma_2$, encoders $f_1, f_2$, node level projectors $z_{n1}, z_{n2}$, aggregation function $\bigoplus$, graph level projectors $z_{g1}, z_{g2}$,  mode $m$.
  \REPEAT
  \FOR{ $\gG^i$ in $G$}
  \STATE $\widetilde{\gG}^{i}_1 = rLap(\gG^i, \gamma_1, o_v, o_n)$ \hfill // first view
  \STATE $\mH_{\widetilde{\gG}^{i}_1} = f_1(\widetilde{\gG}^{i}_1)$ \hfill //node embeddings
  \STATE $\mZ_{\widetilde{\gG}^{i}_1}, \vz_{\widetilde{\gG}^{i}_1} = z_{n1}(\mH_{\widetilde{\gG}^{i}_1}), z_{g1}(\bigoplus(\mH_{\widetilde{\gG}^{i}_1}))$ \hfill // features
  \STATE $\widetilde{\gG}^{i}_2 = rLap(\gG^i, \gamma_2, o_v, o_n)$ \hfill // second view
  \STATE $\mH_{\widetilde{\gG}^{i}_2} = f_2(\widetilde{\gG}^{i}_2)$ \hfill //node embeddings
  \STATE $\mZ_{\widetilde{\gG}^{i}_2}, \vz_{\widetilde{\gG}^{i}_2} = z_{n2}(\mH_{\widetilde{\gG}^{i}_2}), z_{g2}(\bigoplus(\mH_{\widetilde{\gG}^{i}_2}))$ \hfill // features
  \ENDFOR
  \STATE // compute objectives based on contrastive modes
  \STATE $\text{\textbf{set} } \ell = 0$
  \FOR{$j = 1, 2, \cdots, M$ and $k = 1, 2, \cdots, M$}
  \IF{$m = ``\textit{l-l}"$}
   \STATE  $\ell = \ell + \mathcal{L}_{l-l}(\mZ_{\widetilde{\gG}^{j}_1}, \mZ_{\widetilde{\gG}^{j}_2}, \mZ_{\widetilde{\gG}^{k}_1}, \mZ_{\widetilde{\gG}^{k}_2})$ 
   \ELSIF{$m = ``\textit{g-l}"$} 
   \STATE  $\ell = \ell + \mathcal{L}_{g-l}(\mZ_{\widetilde{\gG}^{j}_1}, \mZ_{\widetilde{\gG}^{k}_2}, \vz_{\widetilde{\gG}^{j}_1}, \vz_{\widetilde{\gG}^{k}_2})$ 
   \ELSIF{$m = ``\textit{g-g}"$} 
   \STATE  $\ell = \ell + \mathcal{L}_{g-g}(\vz_{\widetilde{\gG}^{j}_1}, \vz_{\widetilde{\gG}^{j}_2}, \vz_{\widetilde{\gG}^{k}_1}, \vz_{\widetilde{\gG}^{k}_2})$ 
   \ENDIF
  \ENDFOR
  \STATE // compute gradients with appropriate normalization $\mathcal{Z}$
  \STATE $\nabla_{f_1, f_2, z_{n1}, z_{n2}, z_{g1}, z_{g2}} \frac{1}{\mathcal{Z}} (\ell)$
  \UNTIL{convergence/tolerance}
\end{algorithmic}
\end{algorithm}

We present a generalized GCL framework with $rLap$ as the augmentor in Algorithm \ref{alg:rlap_gcl}. Given a collection of graphs $G = \{\gG^i, i \in \{1, \cdots, M\}, M \ge 1$, the blueprint presented in Algorithm \ref{alg:rlap_gcl} can be used for learning node and graph features via different design choices for encoders (including projectors) and contrastive objectives. For every graph $\gG^i \in G$, $rLap$ generates the randomized schur complement views $\widetilde{\gG}^{i}_1, \widetilde{\gG}^{i}_2$, for which, the graph features $\vz_{\widetilde{\gG}^{i}_1}, \vz_{\widetilde{\gG}^{i}_2}$ and node features $\mZ_{\widetilde{\gG}^{j}_1}, \mZ_{\widetilde{\gG}^{j}_2}$ are learnt by the encoders and projectors. The learning objective depends on the contrastive mode $m$, which can be either `l-l', `g-l' or `g-g' as defined in the preliminaries. Most of the GCL frameworks in the literature can be derived from this blueprint. For instance, by replacing $rLap$ with a uniform node dropping augmentor, using shared GCN-based encoders $f_1=f_2=f(.)$, MLP-projectors $z_{g1}=z_{g2}=z_g(.)$, `g-g' contrastive mode and InfoNCE loss $\mathcal{L}_{l-l}$, we obtain the widely used GraphCL framework \citep{you2020graph}. See Appendix \ref{appendix:experimental_setup} for illustrations and details.

However, recent efforts in developing augmentation techniques tend to restrict the design choices of encoders and objectives when comparing the overall GCL performance with prior works. Such an approach fails to isolate the effectiveness and limitations of an augmentor from other design choices. We break this norm and follow established benchmark practices of \citet{zhu2021empirical} to perform controlled experiments in the following section.

\section{Experiments}

Owing to the framework-agnostic nature of augmentors, we conduct controlled experiments and evaluate their effectiveness on unsupervised node and graph classification tasks under the linear evaluation protocol \citep{velickovic2019deep}. The controlled settings pertain to fixing the shared/dedicated nature of encoder(s), contrastive modes, and objectives (see Table \ref{table:exp_settings}). Especially, we leverage the designs of 4 widely used GCL frameworks: GRACE \citep{zhu2020deep}, MVGRL \citep{hassani2020contrastive}, GraphCL \citep{you2020graph} and BGRL \citep{thakoor2021bootstrapped}. Comprehensive details on experimental settings, datasets, augmentors, design choices, evaluation protocols  and baselines are available in Appendix \ref{appendix:experimental_setup}. In the following results for $rLap$, we employ a `random' node selection scheme as $o_v$ and an ordering scheme based on increasing order of edge weights as $o_n$. A detailed ablation study on these schemes is presented in Appendix \ref{app:rlap_ablation}.

\begin{table}[ht]
\caption{Control settings for evaluating augmentors.}
\label{table:exp_settings}
\vskip 0.15in
\begin{center}
\begin{small}
\begin{sc}
\begin{tabular}{lcccr}
\toprule
Setting & design choices\\
\midrule
Dataset     & \textbf{eval} \\
Augmentor   & \textbf{eval} \\
Encoders    &  Shared, dedicated \\
Mode    & \textit{l-l, g-l, g-g} \\
Objective   & InfoNCE, JSD, BL\\
\bottomrule
\end{tabular}
\end{sc}
\end{small}
\end{center}
\vskip -0.1in
\end{table}

\subsection{Node Classification Results}

\begin{table*}[ht!]
\centering
\caption{Evaluation (in accuracy) on benchmark node datasets with \textbf{GRACE} based design.}
\label{table:results_grace}
\vskip 0.15in
\begin{center}
\begin{small}
\begin{sc}
\begin{tabular}{c|c|c|c|c|c}
\toprule
Augmentor & CORA & Amazon-Photo & PubMed & Coauthor-CS & Coauthor-Phy \\
\midrule
EdgeAddition & $83.34 \pm 1.69$ &  $86.99 \pm 0.57$ & $84.03 \pm 0.97$ & $89.94 \pm 0.72$ & $95.53 \pm 0.35$ \\
EdgeDropping & $82.87 \pm 2.39$ & $90.8 \pm 0.55$ & $84.24 \pm 0.87$  & $92.53 \pm 0.4$ & $95.31 \pm 0.43$ \\
EdgeDroppingDegree & \underline{$83.42 \pm 3.57$}  & $88.58 \pm 1.13$ & $83.88 \pm 0.77$ & $92.49 \pm 0.49$ & $95.47 \pm 0.38$ \\
EdgeDroppingEVC & $82.32 \pm 1.81$ & $90.74 \pm 0.9$ & $84.2 \pm 0.6$  & $92.3 \pm 0.31$ & $95.60 \pm 0.37$\\
EdgeDroppingPR & $82.39 \pm 1.31$ & $92.01 \pm 0.87$  & $84.22 \pm 0.74$ & $92.45 \pm 0.51$ &  $95.51 \pm 0.25$ \\
MarkovDiffusion & $82.1 \pm 3.11$  &  $90.91 \pm 0.89$  & $83.96 \pm 0.91$ & $92.69 \pm 0.59$ &  $95.13 \pm 0.43$\\
NodeDropping & $82.9 \pm 1.8$  & \underline{$92.15 \pm 1.33$} & $84.02 \pm 0.9$ & \underline{$92.81 \pm 0.89$} & \underline{$95.63 \pm 0.32$}  \\
PPRDiffusion & $79.85 \pm 2.07$  & $91.16 \pm 0.91$ & $83.31 \pm 1.26$ & $92.75 \pm 0.58$ & $95.09 \pm 0.34$ \\
RandomWalkSubgraph & $82.13 \pm 2.89$ & $89.76 \pm 1.26$ & \underline{$84.37 \pm 0.59$}  & $92.51 \pm 0.29$ & $95.12 \pm 0.33$\\
rLap & $\mathbf{83.75 \pm 2.64}$  & $\mathbf{92.59 \pm 1.05}$ & $\mathbf{84.56 \pm 0.71}$ & $\mathbf{93.1 \pm 0.54}$  & $\mathbf{95.83 \pm 0.44}$ \\
\bottomrule
\end{tabular}
\end{sc}
\end{small}
\end{center}
\vskip -0.1in
\end{table*}

\begin{table*}[ht!]
\centering
\caption{Evaluation (in accuracy) on benchmark node datasets with \textbf{MVGRL} based design.}
\label{table:results_mvgrl}
\vskip 0.15in
\begin{center}
\begin{small}
\begin{sc}
\begin{tabular}{c|c|c|c|c|c}
\toprule
Augmentor & CORA & Amazon-Photo & PubMed & Coauthor-CS & Coauthor-Phy \\
\midrule
EdgeAddition & $81.62 \pm 4.01$  & $87.64 \pm 1.68$ & $83.35 \pm 0.75$  & $91.44 \pm 0.65$ & $94.23 \pm 0.21$ \\
EdgeDropping &  $83.49 \pm 1.32$ & $87.87 \pm 1.33$ & $83.73 \pm 1.14$  &   $91.18 \pm 0.43$ & \underline{$94.68 \pm 0.38$}\\
EdgeDroppingDegree & $83.57 \pm 2.29$  & $88.11 \pm 1.25$ & $84.13 \pm 0.79$ & $91.32 \pm 0.52$ & $94.57 \pm 0.39$\\
EdgeDroppingEVC & $82.79 \pm 2.73$  & $88.81 \pm 1.46$ & \underline{$84.14 \pm 0.77$} & \underline{$91.71 \pm 0.46$} & $94.62 \pm 0.56$ \\
EdgeDroppingPR & $83.42 \pm 2.47$ & $88.04 \pm 1.36$  & $83.76 \pm 0.91$ & $91.52 \pm 0.47$ & $94.51 \pm 0.35$\\
MarkovDiffusion & $\mathbf{84.3 \pm 2.91}$  & \underline{$90.2 \pm 0.98$} & $84.0 \pm 1.03$ & $91.61 \pm 0.49$ & $94.22 \pm 0.33$ \\
NodeDropping & \underline{$84.16 \pm 1.69$} & $86.46 \pm 1.51$ & $83.71 \pm 1.22$ & $91.6 \pm 0.58$ & $94.54 \pm 0.29$  \\
PPRDiffusion & $84.05 \pm 2.72$  & $\mathbf{90.84 \pm 1.67}$  & $82.7 \pm 0.85$ & $90.9 \pm 1.06$ & $94.03 \pm 0.5$ \\
RandomWalkSubgraph & $83.53 \pm 2.46$  & $88.31 \pm 1.01$ & $83.36 \pm 0.94$ & $91.70 \pm 0.49$ & $94.6 \pm 0.49$\\
rLap & $83.68 \pm 2.04$ & $87.14 \pm 1.34$ &  $\mathbf{84.21 \pm 0.46}$ & $\mathbf{91.73 \pm 0.53}$ & $\mathbf{94.81 \pm 0.31}$ \\
\bottomrule
\end{tabular}
\end{sc}
\end{small}
\end{center}
\vskip -0.1in
\end{table*}

For node classification tasks, we report the performance using GRACE and MVGRL-based designs in Table \ref{table:results_grace}, \ref{table:results_mvgrl} respectively. The GRACE design corresponds to shared encoders, `l-l' contrastive mode, and InfoNCE objective. In this setting, $rLap$ achieves the best performance across all datasets, followed by EdgeDropping variants and NodeDropping. It is interesting to observe that PPRDiffusion and MarkovDiffusion approaches tend to perform relatively poorly in this setting. This is due to the fixed structure of the augmented views during training that a shared encoder is attempting to contrast. Intuitively, the lack of stochasticity fails to present noisy views to the model and improve its contrastive ability. On the other hand, for every node that is eliminated by $rLap$, it incorporates the lost information in the remaining sub-graph and preserves the random walk-based combinatorial properties in expectation. This approach of leveraging global information in augmented views seems to benefit the GNN encoders. Additionally, since randomized Schur complements are based on node and edge perturbations, one can intuitively consider it as an effective combination of the two.

On the other hand, MVGRL design corresponds to dedicated encoders, `g-l' contrastive mode, and JSD objective. In this setting, diffusion-based methods performed relatively better on datasets such as CORA and AMAZON-PHOTO when compared to $rLap$ and NodeDropping. During our experiments on PUBMED, COAUTHOR-CS, and COAUTHOR-PHY, the PPRDiffusion technique led to out-of-memory (OOM) issues on a 32GB GPU, which we addressed using sub-graph sampling approaches. One can observe that our technique is outperformed by diffusion and adaptive augmentors such as EdgeDroppingEVC on CORA and AMAZON-PHOTO and the margin of improvement with $rLap$ on PUBMED, COAUTHOR-CS and COAUTHOR-PHY is minimal. However, we emphasize the observation that computational overheads of PPRDiffusion, MarkovDiffusion and EdgeDroppingEVC techniques are significantly large when compared to $rLap$ (see Table \ref{table:node_aug_stats}). In this setting, $rLap$ achieves the best trade-off between performance and computational overheads. Furthermore, based on Theorem \ref{thm:theta_schur}, one can combine $rLap$ with PPRDiffusion to reduce the computational overheads (see Appendix \ref{app:add_exp}). An interesting observation related to MVGRL design and non-diffusion based augmentors is that a perturbation ratio $\gamma_1$ close to $0$ for one of the views was preferred in most of the experiments. Intuitively, one can think of this setting as contrasting a highly perturbed graph with a relatively less perturbed one, which reflects the essence of multi-view contrasting.

\subsection{Graph Classification Results}

\begin{table*}[ht!]
\centering
\caption{Evaluation (in accuracy) on benchmark graph datasets with \textbf{GraphCL} based design.}
\label{table:results_graphcl}
\vskip 0.15in
\begin{center}
\begin{small}
\begin{sc}
\begin{tabular}{c|c|c|c|c|c}
\toprule
Augmentor & PROTEINS & IMDB-BINARY & MUTAG & IMDB-MULTI & NCI1 \\
\midrule
EdgeAddition &  $71.34 \pm 3.28$ & $68.7 \pm 3.55$   &  $81.0 \pm 8.89$  &  $47.13 \pm 4.33$ & $73.36 \pm 1.81$ \\
EdgeDropping & $71.79 \pm 3.37$  & $70.4 \pm 5.37$  &  $83.0 \pm 5.15$ &   $46.4 \pm 4.81$ & $71.29 \pm 3.08$ \\
EdgeDroppingDegree & $71.96 \pm 3.93$ & $67.4 \pm 4.76$ & $85.0 \pm 11.62$ & \underline{$48.03 \pm 5.42$} & $74.4 \pm 1.65$   \\
EdgeDroppingEVC & $74.2 \pm 2.97$ & \underline{$71.3 \pm 4.71$}  & $80.5 \pm 10.11$ &  $45.94 \pm 5.86$ & $72.65 \pm 2.32$  \\
EdgeDroppingPR &  $74.11 \pm 4.24$ & $\mathbf{71.4 \pm 2.87}$  & $83.5 \pm 8.38$ & $46.8 \pm 2.81$ & \underline{$74.6 \pm 2.29$} \\
MarkovDiffusion &  $72.68 \pm 4.43$ &  $64.9 \pm 5.54$ & $80.0 \pm 6.71$ &  $41.47 \pm 3.69$ & $71.83 \pm 2.5$\\
NodeDropping & $73.57 \pm 1.84$  & $69.0 \pm 4.27$  &  $82.0 \pm 6.03$  &  $47.33 \pm 4.86$ &  $73.65 \pm 2.61$\\
PPRDiffusion &  $73.48 \pm 5.15$  & $69.8 \pm 2.96$   & \underline{$85.0 \pm 6.71$}  & $42.87 \pm 3.63$ & $73.36 \pm 2.34$ \\
RandomWalkSubgraph & \underline{$74.55 \pm 4.59$}  &  $70.0 \pm 3.35$ &  $84.0 \pm 9.17$  & $43.93 \pm 3.55$ & $74.01 \pm 2.47$ \\
rLap &  $\mathbf{75.27 \pm 3.34}$ & $70.9 \pm 5.01$  & $\mathbf{87.5 \pm 9.86}$  & $\mathbf{48.66 \pm 2.68}$ & $\mathbf{75.06 \pm 1.65}$  \\
\bottomrule
\end{tabular}
\end{sc}
\end{small}
\end{center}
\vskip -0.1in
\end{table*}

\begin{table*}[ht!]
\centering
\caption{Evaluation (in accuracy) on benchmark graph datasets with \textbf{BGRL} based design. }
\label{table:results_bgrl}
\vskip 0.15in
\begin{center}
\begin{small}
\begin{sc}
\begin{tabular}{c|c|c|c|c|c}
\toprule
Augmentor & PROTEINS & IMDB-BINARY & MUTAG &  IMDB-MULTI & NCI1\\
\midrule
EdgeAddition & $81.57 \pm 6.12$  & $76.1 \pm 5.63$  & $73.5 \pm 11.63$  & $59.0 \pm 9.58$ & $72.77 \pm 3.01$ \\
EdgeDropping & $80.5 \pm 7.94$  &  \underline{$78.3 \pm 10.83$} & $72.5 \pm 10.55$ & $56.13 \pm 8.31$ & $74.89 \pm 2.36$ \\
EdgeDroppingDegree & $81.5 \pm 6.51$ &  $\mathbf{78.5 \pm 6.32}$  & $74.0 \pm 5.39$ & \underline{$59.07 \pm 5.0$} & $73.67 \pm 3.98$ \\
EdgeDroppingEVC &  $82.8 \pm 3.94$ & $75.8 \pm 6.26$  & $79.5 \pm 9.63$  & $55.0 \pm 7.01$ & $74.5 \pm 1.86$ \\
EdgeDroppingPR & $80.86 \pm 8.49$  & $76.8 \pm 7.48$   & $74.5 \pm 8.52$ &  $57.8 \pm 9.96$ & $72.7 \pm 2.69$ \\
MarkovDiffusion & $81.84 \pm 7.48$  & $75.3 \pm 8.36$   & \underline{$79.6 \pm 8.79$} &  $56.8 \pm 7.71$ & $68.69 \pm 2.99$  \\
NodeDropping & $80.59 \pm 6.12$   & $70.3 \pm 7.29$  & $72.0 \pm 9.81$ & $57.53 \pm 8.5$ & $71.78 \pm 2.4$ \\
PPRDiffusion &  $79.32 \pm 7.69$ & $73.5 \pm 8.22$  & $76.5 \pm 10.26$ &  $57.13 \pm 7.72$ & $70.49 \pm 2.23$ \\
RandomWalkSubgraph & \underline{$83.09 \pm 3.85$}  & $69.5 \pm 10.31$& $73.5 \pm 10.51$ &  $57.8 \pm 9.47$ & \underline{$74.92 \pm 2.88$} \\
rLap & $\mathbf{84.34 \pm 4.03}$ & $74.8 \pm 10.17$  & $\mathbf{81.5 \pm 5.39}$ & $\mathbf{59.47 \pm 7.42}$ & $\mathbf{75.15 \pm 3.57}$   \\
\bottomrule
\end{tabular}
\end{sc}
\end{small}
\end{center}
\vskip -0.1in
\end{table*}

For graph classification tasks, we report the performance using GraphCL and BGRL designs in Table \ref{table:results_graphcl}, \ref{table:results_bgrl} respectively. GraphCL design corresponds to shared encoders, `g-g' contrastive mode and InfoNCE objective. Similar to our observations with shared encoders and InfoNCE loss in the node classification setting, $rLap$ outperformed other augmentors on PROTEINS, MUTAG, IMDB-MULTI and NCI1 datasets and is quite close to the EdgeDroppingPR technique on IMDB-BINARY dataset. EdgedroppingPR was the best adaptive edge augmentor on PROTEINS, IMDB-BINARY and NCI1 whereas EdgeDroppingDegree performed well on MUTAG and IMDB-MULTI. One should note that leveraging PageRank information in EdgeDroppingPR to perform stochastic augmentations consistently led to better performance than the PPRDiffusion approach, which underscores the benefits of stochasticity. Interestingly, we found that this particular design preferred augmentation ratios ($\gamma_1, \gamma_2$) to be nearly equal (see Table \ref{table:hp_graphcl}), indicating that in `g-g' mode, the shared encoders prefer to contrast similar sized graphs.

In the next setting, we use a modified BGRL design which corresponds to shared encoders, `g-l' contrastive mode and BL objective. Unlike previously mentioned designs, BGRL employs an online and target encoder with weight-sharing based on a moving average technique. The original BGRL design \citep{thakoor2021bootstrapped} uses an `l-l' contrastive mode without negative samples and focuses on node classification settings. In our work, we leverage the modularity of the PyGCL library and incorporate the `g-l' contrastive mode for graph classification. With this design, we report some of the highest observed performances on the PROTEINS dataset in unsupervised settings, with $rLap$ achieving the highest classification accuracy of $\approx 84.34 \%$. Additionally, it achieves the best results on MUTAG, IMDB-MULTI and NCI1 but continues to suffer on IMDB-BINARY. However, in our ablation studies (see Appendix \ref{app:rlap_ablation}), we observed that by changing the node elimination scheme from being a `random' selection to a minimum `degree' selection, this $rLap$ variant achieves $\approx 79.5\%$ accuracy on IMDB-BINARY and surpasses other techniques.

\section{Future Research}

\textbf{$\bullet$ Towards distributed augmentation:} Current approaches to augment large-scale graphs tend to rely on  simpler techniques such as edge dropping \cite{thakoor2021bootstrapped}. Thus, by extending $rLap$ augmentation to a distributed setting and leveraging its connection with graph diffusion, future efforts can explore and unlock the potential benefits of a richer class of augmentations for large-scale GCL. However, as a current challenge to achieving this goal, we observed that the resulting views of $rLap+\text{PPRDiffusion}+sampling$ can be denser than the $\text{PPRDiffusion} + sampling$ approach. Thus, $rLap$ can address the bottlenecks of PPRDiffusion by saving memory during augmentation, but the GNN encoders would eventually consume additional memory due to the message-passing operations on relatively more edges. Although this overhead might not be significant for small-medium scale graphs, we believe that the reader should be aware of these practical trade-offs when dealing with large-scale graphs.

\textbf{$\bullet$ Towards randomized second-order optimizers:} Computing the inverse of the Hessian for large neural networks poses significant computational overheads to optimizers that aim to leverage the curvature of the loss landscape. A popular technique to address this issue is to compute layer-wise block diagonal approximations of the hessian \citep{martens2015optimizing, osawa2019large, hoefler2021sparsity}, which is relatively easy to invert. Alternatively, to capture the second-order information for a subset of parameters (for example, pertaining to a single layer), one can leverage the matrix inversion lemma and compute randomized Schur complements with respect to these parameters to negate the need to invert large Hessian matrices. Especially, by extending $rLap$, this approach can be computationally efficient while also incorporating unbiased curvature information into the parameter updates. This strategy can be of significant interest when only a subset of neural network layers need to be fine-tuned for downstream tasks.

\textbf{$\bullet$ Towards fine-grained augmentor benchmarks:} Our paper presents an extensive augmentor benchmark with respect to various GCL design choices. Thus, future efforts can leverage and extend these fine-grained comparisons for reproducible and fair comparison of augmentors.

\section{Conclusion}

In this work, we designed a fast, stochastic augmentor based on randomized Schur complements and demonstrated its flexibility and effectiveness on a variety of GCL designs and benchmark datasets. Our theoretical analysis and extensive experiments have proved that $rLap$ achieves a perfect balance with respect to performance and computational overheads while achieving state-of-the-art accuracies on unsupervised node and graph classification tasks. Especially, we achieved an unsupervised graph classification accuracy of $\approx 84.34 \%$ on the PROTEINS dataset under linear evaluation and exceeded current state-of-the-art results. To conclude, we emphasize the potential of this new class of augmentation techniques that enable GCL frameworks to leverage the stochasticity and combinatorial properties of randomized Schur complements.

\section*{Acknowledgements}

VK would like to thank Prof. Jonathan Weare, Zhengdao Chen,  and the anonymous reviewers for their constructive feedback and prompt discussions during the preparation of this manuscript.




\bibliography{main}
\bibliographystyle{icml2023}

\newpage
\appendix
\onecolumn

\section{Related Work}
\label{app:related_work}

Earlier works on representation learning techniques on graphs with limited labels leveraged exploration-based approaches such as DeepWalk \citep{perozzi2014deepwalk} and node2vec \citep{grover2016node2vec}. Inspired by the skip-gram model of word2vec \citep{mikolov2013efficient}, such techniques generate similar embeddings for nodes that co-occur in a random walk. On the other hand, techniques such as Large-scale Information Network Embedding (LINE) \citep{tang2015line} and Predictive Text Embedding (PTE) \citep{tang2015pte} learn representations by contrasting nodes with negative samples drawn from a noisy distribution. Eventually, in the work of \citet{qiu2018network}, DeepWalk, node2vec, LINE and PTE were unified under a common matrix factorization scheme for computing the network/graph embeddings. However, their limitations on scalability and incorporating global information in the learning process persisted.

Recent efforts in GCL leverage GNNs as scalable feature encoders and are inspired by techniques/frameworks in vision \citep{gidaris2018unsupervised, hjelm2018learning, chen2020simple, jing2020self} and language domains \citep{mikolov2013efficient, devlin2018bert, radford2018improving, lan2019albert}. Some of the notable efforts by \citet{velickovic2019deep, sun2019infograph, zhu2020deep, you2020graph, hassani2020contrastive} employ these GCL frameworks and aim to minimize objectives based on mutual information. With this growing line of work, a critical aspect of most GCL frameworks that is devised based on trial and error is the augmentation phase. The node/edge perturbation techniques introduced in \citet{you2020graph} have been widely employed due to their simplicity and are empirically effective on unsupervised, self-supervised node and graph classification tasks \citep{zhu2021empirical, thakoor2021bootstrapped}. To the contrary, \citet{hassani2020contrastive} showed that augmentations such as graph diffusions can outperform the stochastic techniques with dedicated encoders. Since then, efforts have been made to devise augmentations that can leverage structural information during the training process either in a pre-defined \citep{zhu2021graph} or learnable fashion \citep{yin2022autogcl}. The adaptive augmentation technique by \citet{zhu2021graph} takes a step in this direction and focuses on edge-dropping mechanisms based on centrality measures. We observed in our experiments that, it is difficult to choose a particular centrality measure that can work for both node and graph-based tasks. Furthermore, measures such as eigenvector centrality are computationally expensive, which makes such adaptive techniques orders of magnitude slower than simple edge dropping. 

To address these issues, we designed the $rLap$ algorithm, which is inspired by the approximate gaussian elimination (AGE) technique introduced by \citet{kyng2016approximate}. The AGE technique leverages the idea of effective resistances \citep{spielman2008graph} and falls in the line of work on fast approximate linear system solvers for SDDM matrices \citep{spielman2003solving, spielman2004nearly, koutis2011nearly, cohen2014solving, kyng2016sparsified}. It has been influential in efficiently solving linear systems of equations \citep{cohen2018solving, peng2021solving, cohen2021solving, chen2021rchol}, sampling random spanning trees \citep{durfee2017sampling} and matrix scaling \citep{cohen2017matrix}. The technique of \citet{kyng2016approximate} was later adopted by \citet{fahrbach2020faster} for generating node embeddings using sparse matrix factorization techniques. However, to the best of our knowledge, techniques along this line of work have not been explored for GCL.

\section{Proof for Theorem \ref{thm:rlap}}
\label{app:poc_rlap}

Recall from the preliminaries that: 
\begin{equation*}
\textit{CLIQUE}(\mL, v_i) = \frac{1}{2w(v_i)}\sum_{e_{ij}\in \gE}\sum_{e_{ik}\in \gE}w(e_{ij})w(e_{ik})\mathbf{\Delta}_{jk} 
\end{equation*}

The factor $2$ in the denominator addresses duplicate laplacians in the quadratic sum. We assume that $\gG$ doesn't have self-loops and eliminate these redundant summations to rewrite $\textit{CLIQUE}(\mL, v_i)$ as:
\begin{equation}
\label{eq:new_clique}
\textit{CLIQUE}(\mL, v_i)  = \sum_{x_l \in \mathcal{N}(v_i)} \sum_{x_q \in \mathcal{N}(v_i) \backslash \{x_1, \cdots, x_l\}} \frac{w(v_i, x_q) \cdot w(v_i, x_l)}{w(v_i)} \cdot \mathbf{\Delta}_{x_lx_q}
\end{equation}

To employ this reformulation in Algorithm \ref{alg:rlap}, we assume $\mR_{i-1}$ as the Schur complement at the beginning of $i^{th}$ iteration of the outer loop and $v_i$ as the selected node for elimination. Then:

\begin{equation}
\label{eq:new_clique_sc}
\textit{CLIQUE}(\mR_{i-1}, v_i)  = \sum_{x_l \in \mathcal{N}_{\mR_{i-1}}(v_i)} \sum_{x_q \in \mathcal{N}_{\mR_{i-1}}(v_i) \backslash \{x_1, \cdots, x_l\}} \frac{w_{\mR_{i-1}}(v_i, x_q) \cdot w_{\mR_{i-1}}(v_i, x_l)}{w_{\mR_{i-1}}(v_i)} \cdot \mathbf{\Delta}_{x_lx_q}
\end{equation}


In Eq. \ref{eq:new_clique_sc}, we eliminated half the effort needed to compute $\textit{CLIQUE}(\mR_{i-1}, v_i)$, but introduced an ordering of the neighbors $\{x_1, \cdots, x_l\} \in \mathcal{N}_{\mR_{i-1}}(v_i)$. This ordering is determined by $o_n$ in Algorithm \ref{alg:rlap} and leads to different conditional probabilities $P(x_q|x_l)$ while computing the approximate clique $\mC_i = C(\mR_{i-1}, v_i)$. The ordered nodes are given by $\mathcal{X}(v_i) = o_n(\mathcal{N}_{\mR_{i-1}}(v_i))$ and $w_{\mR_{i-1}} = \hat{w}$ for notational convenience. Then, $\mathbb{E}C(\mR_{i-1}, v_i)$ is computed as follows:

\begin{multline*}
    \mathbb{E}C(\mR_{i-1}, v_i) = \sum_{x_q \in \mathcal{X}(v_i) \backslash \{x_1\}} P(x_q|x_1)\cdot \hat{w}(x_1, x_q) \cdot \mathbf{\Delta}_{x_1x_q} + \sum_{x_q \in \mathcal{X}(v_i) \backslash \{x_1, x_2\}} P(x_q|x_2)\cdot \hat{w}(x_2, x_q) \cdot \mathbf{\Delta}_{x_2x_q} + \cdots \\
    + \sum_{x_q \in \mathcal{X}(v_i) \backslash \{x_1, \dots, x_{|\mathcal{X}(v_i)|-1}\}}  P(x_q|x_{|\mathcal{X}(v_i)|-1})\cdot \hat{w}(x_{|\mathcal{X}(v_i)|-1}, x_q) \cdot \mathbf{\Delta}_{x_{|\mathcal{X}(v_i)|-1}x_q}
\end{multline*}

By substituting the conditional probabilities and edge weights as per Algorithm \ref{alg:rlap}, we get:

\begin{multline*}
   \mathbb{E}C(\mR_{i-1}, v_i) = 
     \sum_{x_q \in \mathcal{X}(v_i) \backslash \{x_1\}} \frac{\hat{w}(v_i, x_q)}{\hat{w}(v_i) - \hat{w}(v_i, x_1)} \cdot \frac{\hat{w}(v_i, x_1)\cdot (\hat{w}(v_i) - \hat{w}(v_i, x_1))}{\hat{w}(v_i)} \cdot \mathbf{\Delta}_{x_1x_q} \\
  +  \sum_{x_q \in \mathcal{X}(v_i) \backslash \{x_1, x_2\}} \frac{\hat{w}(v_i, x_q)}{\hat{w}(v_i) - \hat{w}(v_i, x_1) - \hat{w}(v_i, x_2)} \cdot \frac{\hat{w}(v_i, x_2)\cdot (\hat{w}(v_i) - \hat{w}(v_i, x_1) - \hat{w}(v_i, x_2))}{\hat{w}(v_i)} \cdot \mathbf{\Delta}_{x_2x_q} \\
  + \dots \\
  + \sum_{x_q \in \mathcal{X}(v_i) \backslash \{x_1, \dots, x_{|\mathcal{X}(v_i)-1|}\}} \frac{\hat{w}(v_i, x_q)}{\hat{w}(v_i) - \sum\limits_{k=1}^{|\mathcal{X}(v_i)|-1} \hat{w}(v_i, x_k)} \cdot \frac{\hat{w}(v_i, x_{|\mathcal{X}(v_i)|-1})\cdot (\hat{w}(v_i) - \sum\limits_{k=1}^{|\mathcal{X}(v_i)|-1} \hat{w}(v_i, x_k))}{\hat{w}(v_i)} \cdot \mathbf{\Delta}_{x_{|\mathcal{X}(v_i)|-1} x_q}
\end{multline*}

The factors cancel out and simplify the expectation to:

\begin{multline*}
   \mathbb{E}C(\mR_{i-1}, v_i) = 
     \sum_{x_q \in \mathcal{X}(v_i) \backslash \{x_1\}} \frac{\hat{w}(v_i, x_q) \cdot \hat{w}(v_i, x_1)}{w(v_i)} \cdot \mathbf{\Delta}_{x_1x_q}
  +  \sum_{x_q \in \mathcal{X}(v_i) \backslash \{x_1, x_2\}} \frac{\hat{w}(v_i, x_q) \cdot \hat{w}(v_i, x_2)}{\hat{w}(v_i)} \cdot \mathbf{\Delta}_{x_2x_q} + \dots \\
  +  \sum_{x_q \in \mathcal{X}(v_i) \backslash \{x_1, \dots, x_{|\mathcal{X}(v_i)-1|}\}} \frac{\hat{w}(v_i, x_q) \cdot \hat{w}(v_i, x_{|\mathcal{X}(v_i)|-1})}{\hat{w}(v_i)} \cdot \mathbf{\Delta}_{x_{|\mathcal{X}(v_i)|-1}x_q}
\end{multline*}
\begin{equation}
    \implies \mathbb{E}\mC_i = \mathbb{E}C(\mR_{i-1}, v_i) =  \sum_{x_l \in \mathcal{X}(v_i)} \sum_{x_q \in \mathcal{X}(u_i) \backslash \{x_1, \cdots, x_l\}} \frac{\hat{w}(v_i, x_q) \cdot \hat{w}(v_i, x_l)}{\hat{w}(v_i)} \cdot \mathbf{\Delta}_{x_lx_q}
\end{equation}

Thus, after node $v_i$ has been eliminated, $\mC_i = C(\mR_{i-1}, v_i)$ is an unbiased estimator of the actual clique. Now, by removing the star and adding $\mC_i$ to $\mR_{i-1}$, we get $\mR_{i}$, which is the unbiased schur complement approximation after $i$ iterations:
\begin{equation}
\label{eq:R_star_C_relation}
    \mR_{i} = \mR_{i-1} - \textit{STAR}(\mR_{i-1}, v_i) + \mC_i
\end{equation}
By repeating this process for $\gamma|\gV|$ iterations, we get the desired randomized schur complement $\mR_{\gamma|\gV|}$.

\subsection{Deviation Analysis}

Considering $\mL_0 = \mR_0 = \mL$, one can observe that $\mR_i$ is a laplacian matrix, satisfying the following laplacian approximations:

\begin{equation}
\label{eq:L_R_relation}
\mL_i = \mR_i + \sum_{k=1}^i \vs_k \vs_k^*, \forall i = {1, 2, \dots, N}
\end{equation}

This implies that at every step we remove a rank 1 matrix corresponding to $\vs_k = \frac{1}{\sqrt{(\mR_{k-1})_{v_kv_k}}}\mR_{k-1}\delta_{v_k}$. Thus, the sequence $\{\mL_i\}$ forms a random process such that:
\begin{equation}
\label{eq:lap_clique_diff_seq}
\mathbb{E}_{i-1}[\mL_i - \mL_{i-1}] = \mathbb{E}_{i-1}\mathbb{E}_{i-1}[(\mR_i - \mR_{i-1} + \vs_i \vs_i^*) |v_i] = \mathbb{E}_{i-1}\mathbb{E}_{i-1}[(\mC_i - \textit{CLIQUE}(\mR_{i-1},v_i)) |v_i] =  0 
\end{equation}

The last equality is based on the result that $\mathbb{E}_{i-1}[\mC_i | v_i] = \textit{CLIQUE}(\mR_{i-1}, v_i)$ as shown above. This implies, $\mathbb{E}\mL_i = \mL$ for $i = 1,2,3, \dots, \gamma|\gV|$ and the `deviation' sequence $\{\mL_i - \mL_0\}$ is a \textit{\textbf{matrix martingale}} with $\textbf{0}$ initial value. The analysis for the `deviation' martingale $\{\mL_i - \mL_0\}$ aids in understanding the randomness that is induced due to a series of Schur complement approximations.

Our line of analysis is inspired by the approach of \citet{kyng2016approximate} which was later presented in \citet{tropp2019matrix}. We relax the assumptions and transformations for spectral approximations and bounding the effective resistances. These assumptions were critical in the analysis of \citet{kyng2016approximate, tropp2019matrix} as the motivation was to solve the laplacian system of linear equations. However, since we are primarily concerned with augmentations of graphs, we avoid such constraints.

\textbf{Analysis Sketch:} We primarily leverage the corrector process approach introduced in \citet{tropp2019matrix}. By building the corrector process for the martingale $\{\mL_i - \mL_0\}$, we provide a tail bound for the singular values of these deviations, which is a direct application of Theorem 7.4 in \citet{tropp2019matrix}. We start by building the individual corrector matrices and then extend them to build the corrector process.

\subsubsection{Correctors}

\begin{definition}
\textbf{Corrector process} \citep{tropp2019matrix}(Sec 7.2): Consider a function $g: [0, \infty] \rightarrow [0, \infty]$, a martingale $\{ \mY_k \in \sR^{N \times N} : k = 0, 1, \cdots\}$ and a predictable random process $\{ \mW_k \in \sR^{N \times N} : k = 0, 1, \cdots\}$ of self-adjoint matrices. We define $\{g \mW_k\}$ as a corrector process of martingale $\{\mY_k\}$ if the real-valued random process $\{ \text{tr}(\exp(\theta \mY_k - g(\theta)\mW_k)) \}, \forall \theta \ge 0$ is a positive supermartingale \citep{williams1991probability, chow2003probability}. 
\end{definition}

\begin{definition}
\textbf{Corrector} \citep{tropp2019matrix}(Sec 7.3): Consider a function $g: [0, \infty] \rightarrow [0, \infty]$, a random self-adjoint matrix $\mU \in \sR^{N \times N}$, a fixed matrix $\mV \in \sR^{N \times N}$. We  define the corrector $g\mV$ as the matrix satisfying:

\begin{equation}
\mathbb{E} \big[ \text{ tr}\big( \exp(\mM + \theta \mU - g(\theta)\mV) \big) \big] \le \text{ tr}\big(\exp(\mM)\big), \theta \ge 0, \forall \mM \in \mathbb{H}_N
\end{equation}

This bound must hold true for all fixed matrices $\mM \in \mathbb{H}_N$, where $\mathbb{H}_N$ denotes a space of real $N \times N$ self-adjoint matrices.
\end{definition}

\begin{corollary}
\label{cor:tropp_lieb}
\citep{tropp2019matrix}(Corollary 7.7): Let $\mM \in \sR^{N \times N}$ be a self-adjoint matrix and $\mU \in \sR^{N \times N}$ be a random self-adjoint matrix. Then for $\theta \in \sR$, 
\begin{equation}
    \mathbb{E} [ \text{ tr}( \exp(\mM + \theta\mU - \log \mathbb{E}e^{\theta\mU}) )] \le \mathbb{E}[ \text{ tr}( \exp(\mM)) ]
\end{equation}

\end{corollary}

This corollary is specific to corrector matrices of the form $g\mV = \log \mathbb{E}e^{\theta\mU}$ and will be used to build the Bernstein and Chernoff correctors for random matrices with arbitrary spectral norm.

\subsubsection{The Bernstein Corrector}

From these definitions, let $\mU_i = \mC_i - \mathbb{E}\mC_i$ be the random self-adjoint matrix. Let $\norm{\mU_i}$ indicate the spectral norm of $\mU_i$ ($l_2$ operator norm), which indicates its largest singular value. Since $\mathbb{E}[\mU_i] = \vzero$, we can directly apply a generalized version of the Bernstein corrector for $\mU_i$. The result in Proposition 7.8 of \citet{tropp2019matrix} assumes $\norm{\mU_i} \le 1$ and requires some form of normalization in the analysis. We avoid such assumptions and extend the result to any arbitrary bound $\norm{\mU_i} \le \mathcal{B}_{\mU_i}$. 

We begin with a scalar case and extend it to the matrix case. If $|x| \le b$, then the exponential $e^{\theta x}$ can be bounded by:
\begin{align*}
\begin{split}
    e^{\theta x} = 1 + \theta x + \sum_{j=2}^\infty \frac{\theta^j}{j!}x^j \le 1 + \theta x + \bigg( \sum_{j=2}^\infty \frac{b^{j-2}|\theta|^j}{2\cdot 3^{j-2}} \bigg) x^2 = 1 + \theta x + \frac{\theta^2}{2 (1 - b|\theta|/3)}x^2
\end{split}
\end{align*}
Similarly, if $\norm{\mU_i} \le \mathcal{B}_{\mU_i}$ then, we can extend the above bound to matrix form as follows:
\begin{align*}
\begin{split}
    e^{\theta \mU_i} \preccurlyeq  \mI + \theta \mU_i + \frac{\theta^2}{2 (1 - \mathcal{B}_{\mU_i}|\theta|/3)}\mU_i^2 
    \implies \log \mathbb{E}e^{\theta \mU_i}  \preccurlyeq \frac{\theta^2}{2 (1 - \mathcal{B}_{\mU_i}|\theta|/3)}\mathbb{E}\mU_i^2
\end{split}
\end{align*}

Where $\preccurlyeq$ indicates inequality in the semi-definite order. For example: if $\mJ \preccurlyeq \mQ$, then $\vx^\top \mJ\vx \le \vx^\top \mQ \vx, \forall \vx $. Also, the second inequality is based on the operator monotone property of logarithms (see Chapter 8 in \citet{tropp2015introduction}). From Corollary \ref{cor:tropp_lieb}, observe that by replacing $\log \mathbb{E}e^{\theta \mU_i}$ with $\frac{\theta^2}{2 (1 - \mathcal{B}_{\mU_i}|\theta|/3)}\mathbb{E}\mU_i^2$ the inequality is still preserved. Thus, we can use $\frac{\theta^2}{2 (1 - \mathcal{B}_{\mU_i}|\theta|/3)}\mathbb{E}\mU_i^2$ as our corrector matrix for $\mU_i$. The variance term $\mathbb{E}\mU_i^2$ can be given by:

\begin{equation}
\label{eq:clique_var}
\mathbb{E}\mU_i^2 = \mathbb{E}[\mC_i - \mathbb{E}[\mC_i]]^2  = \mathbb{E}[\mC_i^2] - (\mathbb{E}[\mC_i])^2 \preccurlyeq \mathbb{E}[\mC_i^2] \preccurlyeq \mathbb{E}[\norm{\mC_i}\cdot \mC_i]
\end{equation}

The second equality is due to the zero mean of $\mC_i - \mathbb{E}[\mC_i]$. The last semi-definite order inequality is straightforward since $\norm{\mC_i}$ indicates the largest singular value of $\mC_i$, which bounds $\vx^\top \mC_i\mC_i \vx \le \norm{\mC_i} \vx^\top \mC_i \vx, \forall \vx$. Additionally, if we assume the spectral norm of the approximated cliques to be bounded $\norm{\mC_i} \le \mathcal{B}_{\mC_i}$, we get $\mathbb{E}\mU^2 \preccurlyeq \mathbb{E}[\norm{\mC_i}\cdot \mC_i]  \preccurlyeq  \mathcal{B}_{\mC_i} \cdot \textit{CLIQUE}(\mR_{i-1}, v_i)$. Putting all these results together, we get the Bernstein corrector as follows:
\begin{equation}
\label{eq:bernstein_corrector}
    \bigg ( \frac{\theta^2}{2 (1 - \mathcal{B}_{\mU_i}|\theta|/3)}\mathcal{B}_{\mC_i}  \bigg ) \textit{CLIQUE}(\mR_{i-1}, v_i)
\end{equation}

The Bernstein corrector handles the randomness associated with the clique approximation after a node $v_i$ is selected (i.e, due to the ordering scheme $o_n$ and its implications on the edge formation probabilities). Now, to handle the randomness associated with this node selection operation, we build the Chernoff corrector.
\subsubsection{The Chernoff Corrector}

Based on the formulations of $\textit{CLIQUE}(\mR_{i-1}, v_i), \textit{STAR}(\mR_{i-1}, v_i)$, we obtain a relation between them as follows:
\begin{equation}
    \textit{CLIQUE}(\mR_{i-1}, v_i) = \textit{STAR}(\mR_{i-1}, v_i) - \vs_i\vs_i^*
\end{equation}

Where $\vs_i = \frac{1}{\sqrt{(\mR_{i-1})_{v_iv_i}}}\mR_{i-1}\delta_{v_i}$ as introduced before. Since $\vs_i\vs_i^*$ is positive semi-definite and $\textit{STAR}(\mR_{i-1}, v_i)$ is formed by a subset of edges of $\mR_{i-1}$, we obtain the following semi-definite order inequalities:
\begin{equation}
\label{eq:clique_star_sc_ineq}
    \vzero \preccurlyeq \textit{CLIQUE}(\mR_{i-1}, v_i) \preccurlyeq \textit{STAR}(\mR_{i-1}, v_i) \preccurlyeq \mR_{i-1}
\end{equation}

This result can now be used to bound $\mathbb{E}_{v_i}[\textit{CLIQUE}(\mR_{i-1}, v_i)]$ over the possible choices of $v_i \in \gV_{\mR_{i-1}}$. Here $\gV_{\mR_{i-1}}$ represents the set of nodes from which $v_i$ is chosen uniformly at random for elimination. This setting corresponds to the case of the `random' node elimination scheme $o_v$.
\begin{equation}
\mathbb{E}_{v_i}[\textit{CLIQUE}(\mR_{i-1}, v_i)] \preccurlyeq \mathbb{E}_{v_i}[\textit{STAR}(\mR_{i-1}, v_i)] = \frac{1}{|\gV_{\mR_{i-1}}|} \sum_{v_i \in \gV_{\mR_{i-1}}} \sum_{e \in \textit{STAR}(\mR_{i-1}, v_i)}w_{\mR_{i-1}}(e)\mathbf{\Delta}_e
\end{equation}

Where $\mathbf{\Delta}_e$ is the elementary laplacian formed by the nodes at either end of edge $e$. The last term in the above equation indicates that, for every node $v_i \in \gV_{\mR_{i-1}}$, we add the weighted laplacians for all edges in $\textit{STAR}(\mR_{i-1}, v_i)$. Thus, we are iterating over all the edges in $\gE_{\mR_{i-1}}$ and adding the weighted elementary laplacians twice. Formally:
\begin{equation}
\label{eq:exp_clique_bound}
\mathbb{E}_{v_i}[\textit{CLIQUE}(\mR_{i-1}, v_i)] \preccurlyeq \frac{2}{|\gV_{\mR_{i-1}}|} \sum_{e \in \gE_{\mR_{i-1}}} w_{\mR_{i-1}}(e)\mathbf{\Delta}_e = \frac{2}{|\gV_{\mR_{i-1}}|} \mR_{i-1}
\end{equation}

Now, we prove an extension of the Chernoff bound lemma to random matrices as follows:

\begin{lemma}
\label{lemma:extended_chernoff}
Let $\mN$ be a random self-adjoint matrix satisfying: $0 \preccurlyeq \mN \preccurlyeq r\mI$, Then $\log \mathbb{E}e^{\theta \mN} \preccurlyeq \frac{(e^{r\cdot \theta} - 1)}{r}(\mathbb{E}\mN)$, $\forall \theta \in \mathbb{R}$.
\end{lemma}
\begin{proof}
The classical version on this statement uses $r = 1$ and has been proved in \citet{tropp2012user} and Lemma 1.12 in \citet{tropp2019matrix}. The proof sketch for the custom case is straightforward and can be derived from the observation that the function $\mN \rightarrow e^{\theta \mN}$ is convex, singular values of $\mN$ lie in $[0, r]$ and $\log(1+x) \le x$ for valid $x$. This gives us:
\begin{equation*}
    e^{\theta \mN} \preccurlyeq \mI + \frac{(e^{r \cdot \theta} - 1)}{r}\mN 
    \implies \log \mathbb{E} e^{\theta \mN} \preccurlyeq \frac{(e^{r \cdot \theta} - 1)}{r}(\mathbb{E}\mN)
\end{equation*}
\end{proof}
By leveraging Lemma \ref{lemma:extended_chernoff} and Eq. \ref{eq:clique_star_sc_ineq} , we can bound $\textit{CLIQUE}(\mR_{i-1}, v_i)$ as $0 \preccurlyeq \textit{CLIQUE}(\mR_{i-1}, v_i) \preccurlyeq \norm{\mR_{i-1}}\cdot \mI$. Now, if we assume a bound on the largest singular values of $\mR_{i-1}$ as $\norm{\mR_{i-1}} \le \mathcal{B}_{\mR_{i-1}}$, the Chernoff corrector can be given using Eq. \ref{eq:exp_clique_bound} as follows:

\begin{equation}
\label{eq:chernoff_corrector}
\frac{(e^{\mathcal{B}_{\mR_{i-1}}\cdot \theta} - 1)}{\mathcal{B}_{\mR_{i-1}}} \cdot \frac{2}{|\gV_{\mR_{i-1}}|} \mR_{i-1}
\end{equation}

\subsubsection{Composite Corrector}

Putting together the Bernstein corrector in Eq. \ref{eq:bernstein_corrector} and the Chernoff corrector in Eq. \ref{eq:chernoff_corrector}, we can create a composite corrector that accounts for the randomness determined by $o_n$ and $o_v$. The composition rule stated and proved in section 7.3.7 of \citet{tropp2019matrix} is defined here for context.

\begin{proposition}
\label{prop:corrector_composition}
\citep{tropp2019matrix} : The composite rule states that, given the sigma fields $\mathcal{F}_0 \subset \mathcal{F}_1 \subset \mathcal{F}_2$, Let $\mP$ be a random matrix measurable over $\mathcal{F}_2$, $\mV_1, \mM_1$ are measurable over $\mathcal{F}_1$, $\mV_0, \mM_0$ are measurable over $\mathcal{F}_0$. For $\theta > 0$, suppose:
\begin{align*}
\begin{split}
\mathbb{E} \big[ \textnormal{tr} \big(\exp (\mM_1 + \theta \mP - g(\theta)\mV_1)\big) | \mathcal{F}_1 \big] &\le \textnormal{tr}\big(\exp(\mM_1)\big) \\
\mathbb{E} \big[ \textnormal{tr}\big(\exp (\mM_0 + \theta \mV_1 - h(\theta)\mV_0)\big)\big] &\le \textnormal{tr}\big(\exp(\mM_0)\big)
\end{split}
\end{align*}
holds true, then $(h \circ g)\mV_0$ is the corrector for $\mP$.
\end{proposition}

Where $(h \circ g)$ represents function composition. To use this proposition, we can compose the Bernstein corrector $g(\theta)V_1 = \big( \frac{\theta^2}{2 (1 - \mathcal{B}_{\mU_i}|\theta|/3)}\mathcal{B}_{\mC_i}  \big ) \textit{CLIQUE}(\mR_{i-1}, v_i)$ from Eq. \ref{eq:bernstein_corrector} and the Chernoff corrector $h(\theta)V_0 = \frac{(e^{\mathcal{B}_{\mR_{i-1}}\theta} - 1)}{\mathcal{B}_{\mR_{i-1}}} \cdot \frac{2}{|\gV_{\mR_{i-1}}|} \mR_{i-1} $ from Eq. \ref{eq:chernoff_corrector} to define the composite corrector as follows:

\begin{equation}
\label{eq:composite_corrector}
\frac{2}{\mathcal{B}_{\mR_{i-1}}|\gV_{\mR_{i-1}}|} \cdot \bigg[ \exp\bigg( \frac{\mathcal{B}_{\mR_{i-1}}\mathcal{B}_{\mC_i} \theta^2}{2 (1 - \mathcal{B}_{\mU_i}|\theta|/3)} \bigg) - 1 \bigg]  \cdot \mR_{i-1}
\end{equation}

\subsubsection{Deviation bound}

\begin{definition}
\label{def:corrector_to_corrector_process}
\textbf{Corrector to Corrector process:} Proposition 7.10 in \citet{tropp2019matrix}: For a function $\widetilde{g} : [0, \infty] \to [0, \infty]$, let $\{\widetilde{\mY}_k\}$ be a matrix martingale of self-adjoint matrices with difference sequence $\{\widetilde{\mU}_k\} = \{\widetilde{\mY}_k - \widetilde{\mY}_{k-1} , k=1,2,\dots\}$. Let $\{\widetilde{\mV}_k\}$ be a predictable sequence of self-adjoint matrices. For each $k$, suppose $\widetilde{g}\widetilde{\mV}_k$ is a corrector for $\widetilde{\mU}_k$, conditional on the sigma field $\mathcal{F}_{k-1}$, then the predictable process $\widetilde{\mW}_k = \sum_{i=1}^k \widetilde{\mV}_i$ generates a corrector $\{\widetilde{g}\widetilde{\mW}_k\}$ for the martingale $\{\widetilde{\mY}_k\}$.
\end{definition}

Since we are interested in the $\{\mL_i - \mL_0\}$ martingale, each matrix in this sequence can be expanded as:
\begin{align*}
\begin{split}
    \mL_i -  \mL_0 &= \mL_i - \mL_{i-1} + \mL_{i-1} - \mL_{i-2} + \cdots + \mL_1 - \mL_0 \\
    &= \mC_i - \mathbb{E}[\mC_i|v_i] + \mC_{i-1} - \mathbb{E}[\mC_{i-1}|v_{i-1}] + \cdots + \mC_1 - \mathbb{E}[\mC_1|v_1]
\end{split}
\end{align*}
Where the second equality follows from Eq. \ref{eq:lap_clique_diff_seq}. Therefore, based on Definition \ref{def:corrector_to_corrector_process}, the corrector process $\{\widetilde{g}\widetilde{\mW}_i, i = 1, \dots\}$ for martingale $\{\mL_i - \mL_0\}$ is given by:
\begin{equation}
\label{eq:composite_corrector_process}
    \widetilde{g}\widetilde{\mW}_i = \sum_{j=1}^i \frac{2}{\mathcal{B}_{\mR_{j-1}}|\gV_{\mR_{j-1}}|} \cdot \bigg[ \exp\bigg( \frac{\mathcal{B}_{\mR_{j-1}}\mathcal{B}_{\mC_j} \theta^2}{2 (1 - \mathcal{B}_{\mU_j}|\theta|/3)} \bigg) - 1 \bigg]  \cdot \mR_{j-1}
\end{equation}

Finally, we state the master tail bound theorem for matrix martingales (see Theorem 7.4 in \citet{tropp2019matrix} and \citet{freedman1975tail} for details on the proof based on martingale stopping time).
\begin{theorem}
\label{thm:matrix_martingale_master_tail_bound}
\citep{tropp2019matrix} If $\{\widetilde{g}\widetilde{\mW}_k\}$ is a corrector process for a martingale $\{\widetilde{\mY}_k \in \sR^{N \times N}\}$ of self-adjoint matrices and $\sigma_{max}(.)$ returns the largest singular value, then:

\begin{equation}
\label{eq:master_tail_bound}
P(\exists k \ge 0 : \sigma_{max}(\widetilde{\mY}_k) \ge t \text{ and } \sigma_{max}(\widetilde{\mW}_k) \le q ) \le N \cdot \inf_{\theta > 0} e^{-\theta t + \widetilde{g}(\theta)q} , \hfill \forall t, q \in \mathbb{R}
\end{equation}
\end{theorem}

The deviation tail bound for the martingale $\{L_i - L_0\}$ can now be given as:
\begin{align}
\begin{split}
\label{eq:martingale_deviation_bound}
P(\norm{\mL_i - \mL_0} > \epsilon ) &\le P( \exists j: \norm{\mL_j - \mL_0} > \epsilon \} \\
 &\le P( \exists j: \sigma_{max}(\mL_j - \mL_0) \ge \epsilon) + P( \exists j: \sigma_{max}(-(\mL_j - \mL_0)) \ge \epsilon) \\
 &\le 2N \cdot \inf_{\theta > 0} \exp( -\epsilon \theta + \widetilde{g}(\theta) \eta^2)
\end{split}
\end{align}
Where $\sigma_{max}(\widetilde{\mW}_i) \le \eta^2$. The first inequality is a split based on the singular values so that Theorem \ref{thm:matrix_martingale_master_tail_bound} can be applied, to obtain the tail bound. The factor $2N$ instead of $N$ (as per Theorem \ref{thm:matrix_martingale_master_tail_bound}) is obtained due to the two cases of singular value bounds in the second inequality of Eq. \ref{eq:martingale_deviation_bound} and based on the validity of the corrector for the negation of the martingale $-(\mL_j - \mL_0)$. The $\widetilde{g}(\theta)$ term in the bound is obtained from the $\theta$ dependent scalar terms of the composite corrector in Eq. \ref{eq:composite_corrector_process}. Observe that, based on Eq. \ref{eq:R_star_C_relation} and \ref{eq:clique_star_sc_ineq}, it is sufficient if we can track $\mathcal{B}_{\mR_{i-1}}$ during the elimination process as $\mR_{i}$ essentially captures the clique approximations $\mC_i$. Thus, by controlling $\mathcal{B}_{\mR_{i-1}}$, we can have a provable bound on the randomness induced during the view generation steps using $rLap$.  

The analysis by \citet{kyng2016approximate, tropp2019matrix} provides one such approach. Their work applies a normalization map to the clique approximations $ \Phi(\mC_i) = \mL^{-1/2}\mC_i\mL^{-1/2}$ and leverages the concept of effective resistances \citep{spielman2008graph} to split the edges in the graph into $\xi$ multi-edges (each with weight $w_{\mL}(e)/\xi$) and bound the spectral norm $\norm{\Phi(\mR_{i})}$. A key takeaway from their approach is that the probability of deviation can be reduced by increasing $\xi$. However, note that this line of analysis aims at better pre-conditioning of the laplacian linear system with $\mL$. In our analysis, since we are only concerned with introducing randomness into our graph augmentations and preserving the Schur complement properties in expectation, we present a generalized analysis and avoid strict assumptions.

\textbf{Runtime:} To analyse the runtime complexity of Algorithm \ref{alg:rlap}, observe that at iteration $i$, when a node $v_i$ is being eliminated, the corresponding $\textit{STAR}(v_i)$ can be removed from $\mR_{i-1}$ in $O(deg(v_i))$ with efficient sparse tensor data types. Next, applying $o_n$ can lead to $O(deg(v_i)\log(deg(v_i)))$ overheads. The clique sampling loop runs $O(deg(v_i))$ times with $O(1)$ overheads to sample a neighbor and add the new edge. This can be achieved by using a bi-directional linked list representation of the graph with an additional list for node pointer lookup. Finally, assuming a random node elimination scheme $o_v$ (i.e, our default variant of $rLap$), $|\gV|=N, |\gE|=M, deg(v_i)=O(\frac{M}{N-i+1})$, the time complexity to eliminate $k = \gamma|\gV|$ nodes is:

\begin{align*}
\begin{split}
    O \bigg( \sum_{i=1}^k \frac{M}{N - i+1} \log \big(\frac{M}{N - i+1} \big) \bigg) 
    &= O \bigg( \sum_{i=0}^{N-1} \frac{M}{N - i} \log \big(\frac{M}{N - i} \big) \bigg) \\
    &= O \bigg( \sum_{i=0}^{N-1} \frac{M \log M}{N - i} - \frac{M \log (N - i)}{N - i} \bigg)
    \\ &= O \bigg( M \log M \sum_{i=0}^{N-1} \frac{1}{N - i} - M \sum_{i=0}^{N-1} \frac{\log (N - i)}{N - i} \bigg) \\
    &=  O \bigg( M \log M \log N - M \big( \frac{\log N}{N} + \cdots \frac{\log 1}{1} \big) \bigg) \\
    &= O \bigg( M \log M \log N \bigg) 
\end{split}
\end{align*}

Note that if one employs a random neighbor ordering scheme for $o_n$, the time complexity can be reduced to $O(M \log N)$. Additionally, one can carefully design a priority queue that maintains the degrees of nodes and looks up the elimination node in $O(deg(v_i))$. Thus, $o_v$ based on min-degree elimination can still be achieved in $ O \bigg( M \log M \log N \bigg) $ time.

\section{Ablation study on randomized Schur complements}
\label{app:rlap_ablation}

In the following ablation study, we consider two variants of $o_v$ and three variants of $o_n$ as mentioned in Table \ref{table:rlap_variants} for $rLap$. The scheme $o_v=rand$ indicates a random selection of nodes for elimination and $o_v=deg$ indicates a selection based on the minimum degree nodes which haven't been eliminated. Formally, by employing a priority queue of nodes based on their current degree values, the $o_v$ scheme returns a node with the lowest degree at each iteration for elimination. Once a node $v_i$ has been selected, the functionality of $o_n$ is to order the neighbors $x_l \in \mathcal{N}_{\mR_{i-1}}(v_i)$ based on weights of edges $w_{\mR_{i-1}}(v_i, x_l)$. The schemes $o_n=asc, o_n=desc, o_n=rand$ indicate sorting of $x_l \in \mathcal{N}_{\mR_{i-1}}(v_i)$ in the ascending, descending and random order of edge weights respectively. The variants of $rLap$ follow the $rLap-o_v-o_n$ naming convention.

\begin{table*}[ht!]
\centering
\caption{Variants of $rLap$ based on $o_v, o_n$.}
\label{table:rlap_variants}
\vskip 0.15in
\begin{center}
\begin{small}
\begin{sc}
\begin{tabular}{c|c|c|c|c}
\toprule
Variant & $o_v$ & $o_v$ description & $o_n$ & $o_n$ description\\
\midrule
rLap-rand-asc & rand & random selection & asc & ascending order of $w_{\mR_{i-1}}(v_i, x_l)$  \\
rLap-rand-desc & rand & random selection & desc & descending order of $w_{\mR_{i-1}}(v_i, x_l)$  \\
rLap-rand-rand & rand & random selection & rand & random order of $w_{\mR_{i-1}}(v_i, x_l)$  \\
rLap-deg-asc & deg & min-degree selection & asc & ascending order of $w_{\mR_{i-1}}(v_i, x_l)$  \\
rLap-deg-desc & deg & min-degree selection & desc & descending order of $w_{\mR_{i-1}}(v_i, x_l)$  \\
rLap-deg-rand & deg & min-degree selection & rand & random order of $w_{\mR_{i-1}}(v_i, x_l)$  \\
\bottomrule
\end{tabular}
\end{sc}
\end{small}
\end{center}
\vskip -0.1in
\end{table*}

On a related note, the $vertex-coarsening$ approach by \citet{fahrbach2020faster} attains the same goal of approximating Schur complements in expectation and is closely related to our work. However, it differs from $rLap$ in its clique approximation approach and chooses the $o_v=deg$ scheme for node elimination. The notion of $o_n$ doesn't apply to their edge sampling procedure when approximating the clique. To understand the benefits of our approach and the role of $o_v, o_n$, we compare the variants of $rLap$ (including $vertex-coarsening$) on all the benchmark datasets.

Table \ref{table:rlap_ablation_grace}, \ref{table:rlap_ablation_mvgrl} compare the unsupervised node classification performance of $rLap$ variants and $vertex-coarsening$ using the GRACE and MVGRL design. The setup and evaluation protocols are the same as per Appendix \ref{appendix:experimental_setup}. In a majority of cases, the $rLap-rand$ variants outperformed $rLap-degree$ variants, followed by $vertex-coarsening$. Thus, showcasing the effectiveness of our clique sampling strategy. A similar pattern can be observed for unsupervised graph classification tasks in Table \ref{table:rlap_ablation_graphcl}, \ref{table:rlap_ablation_bgrl} using the GraphCL and BGRL designs respectively. However, the $vertex-coarsening$ approach dominates the $rLap$ variants when the BGRL design is employed for IMDB-BINARY, i.e the minimum degree-based elimination scheme for $o_v$ seems to outperform the random variants in this case.

\begin{table*}[ht!]
\centering
\caption{Evaluation (in accuracy) on benchmark node datasets with \textbf{GRACE} design and $rLap$ variants.}
\label{table:rlap_ablation_grace}
\vskip 0.15in
\begin{center}
\begin{small}
\begin{sc}
\begin{tabular}{c|c|c|c|c|c}
\toprule
Augmentor & CORA & Amazon-Photo & PubMed & Coauthor-CS & Coauthor-Phy \\
\midrule
rLap-rand-asc & $\mathbf{83.75 \pm 2.64}$ & $\mathbf{92.59 \pm 1.05}$ & $84.56 \pm 0.71$ &  $\mathbf{93.1 \pm 0.54}$  & $\mathbf{95.83 \pm 0.44}$ \\
rLap-rand-desc & $80.4 \pm 3.43$  & $92.37 \pm 0.91$ & $\mathbf{85.49 \pm 0.88}$ & $92.71 \pm 0.57$ &  \underline{$95.73 \pm 0.25$} \\
rLap-rand-rand & $82.94 \pm 3.88$ & \underline{$92.52 \pm 1.18$} & \underline{$84.91 \pm 0.7$} & $92.67 \pm 0.71$  & $95.56 \pm 0.39$ \\
rLap-deg-asc  & $80.37 \pm 1.8$   & $89.7 \pm 0.58$ & $84.22 \pm 0.85$ & \underline{$92.87 \pm 0.6$} & $95.39 \pm 0.4$  \\
rLap-deg-desc & \underline{$83.23 \pm 2.35$} &  $90.98 \pm 1.28$ & $84.45 \pm 0.66$ & $92.84 \pm 0.25$ & $95.24 \pm 0.4$ \\
rLap-deg-rand & $81.62 \pm 3.6$ & $88.27 \pm 1.47$ & $84.8 \pm 0.77$ & $92.63 \pm 0.75$ & $94.79 \pm 0.38$ \\
Vertex coarsening & $78.31 \pm 2.25$ & $90.58 \pm 1.09$ & $84.18 \pm 1.0$ & $92.77 \pm 0.78$ & $94.58 \pm 0.36$ \\
\bottomrule
\end{tabular}
\end{sc}
\end{small}
\end{center}
\vskip -0.1in
\end{table*}

\begin{table*}[ht!]
\centering
\caption{Evaluation (in accuracy) on benchmark node datasets with \textbf{MVGRL} design and $rLap$ variants.}
\label{table:rlap_ablation_mvgrl}
\vskip 0.15in
\begin{center}
\begin{small}
\begin{sc}
\begin{tabular}{c|c|c|c|c|c}
\toprule
Augmentor & CORA & Amazon-Photo & PubMed & Coauthor-CS & Coauthor-Phy \\
\midrule
rLap-rand-asc & $\mathbf{83.68 \pm 2.04}$ & $87.14 \pm 1.34$ &  $\mathbf{84.21 \pm 0.46}$ &  $91.73 \pm 0.53$ & $\mathbf{94.81 \pm 0.31}$ \\
rLap-rand-desc & $82.24 \pm 1.66$ & $86.14 \pm 0.89$  & $83.88 \pm 0.51$ & $91.6 \pm 0.48$ & \underline{$94.53 \pm 0.2$}  \\
rLap-rand-rand & $81.84 \pm 1.26$ & $87.28 \pm 1.29$ &  $83.79 \pm 0.97$  &  $91.57 \pm 0.49$ & $94.24 \pm 0.43$ \\
rLap-deg-asc & $82.32 \pm 2.29$ &  $\mathbf{87.52 \pm 1.0}$ & $83.61 \pm 0.74$ & $91.74 \pm 0.64$ & $94.18 \pm 0.34$  \\
rLap-deg-desc & \underline{$83.38 \pm 1.65$} & $86.47 \pm 1.25$ & $83.54 \pm 0.64$ & \underline{$91.82 \pm 0.58$}  & $94.5 \pm 0.35$  \\
rLap-deg-rand & $81.73 \pm 2.47$ & $87.11 \pm 1.06$ & \underline{$84.02 \pm 0.49$} & $\mathbf{91.85 \pm 0.81}$ & $94.23 \pm 0.33$  \\
Vertex coarsening & $82.02 \pm 2.18$ &  \underline{$87.36 \pm 1.0$} & $83.62 \pm 0.67$ & $90.79 \pm 0.49$ & $94.12 \pm 0.35$ \\
\bottomrule
\end{tabular}
\end{sc}
\end{small}
\end{center}
\vskip -0.1in
\end{table*}

\begin{table*}[ht!]
\centering
\caption{Evaluation (in accuracy) on benchmark graph datasets with \textbf{GraphCL} design and $rLap$ variants.}
\label{table:rlap_ablation_graphcl}
\vskip 0.15in
\begin{center}
\begin{small}
\begin{sc}
\begin{tabular}{c|c|c|c|c|c}
\toprule
Augmentor & PROTEINS & IMDB-BINARY & MUTAG & IMDB-MULTI & NCI1\\
\midrule
rLap-rand-asc & $\mathbf{75.27 \pm 3.34}$ & $\mathbf{70.9 \pm 5.01}$  & $\mathbf{87.5 \pm 9.86}$  & $\mathbf{48.66 \pm 2.68}$ & $\mathbf{75.06 \pm 1.65}$  \\
rLap-rand-desc & $72.61 \pm 4.75$ & $68.3 \pm 4.29$ & \underline{$85.5 \pm 7.83$} & \underline{$48.23 \pm 2.73$} & $74.14 \pm 1.9$ \\
rLap-rand-rand & $73.34 \pm 4.54$ & $68.8 \pm 4.12$ & $83.5 \pm 5.5$ & $47.8 \pm 4.16$ & $74.43 \pm 2.15$\\
rLap-deg-asc & \underline{$74.29 \pm 2.76$} & $69.3 \pm 3.85$  & $85.0 \pm 7.42$ & $47.2 \pm 4.31$ & $74.2 \pm 2.34$  \\
rLap-deg-desc & $73.04 \pm 2.76$ & \underline{$70.7 \pm 5.88$} & $82.0 \pm 8.43$  & $46.33 \pm 4.91$ & \underline{$74.89 \pm 1.43$}  \\
rLap-deg-rand & $73.66 \pm 2.86$ & $67.5 \pm 4.2$ & $81.5 \pm 8.5$ & $47.13 \pm 4.15$ & $73.6 \pm 2.17$ \\
Vertex Coarsening & $71.61 \pm 4.51$ & $67.6 \pm 3.85$  & $80.0 \pm 9.75$ & $45.6 \pm 4.32$ & $74.09 \pm 2.5$\\
\bottomrule
\end{tabular}
\end{sc}
\end{small}
\end{center}
\vskip -0.1in
\end{table*}

\begin{table*}[ht!]
\centering
\caption{Evaluation (in accuracy) on benchmark graph datasets with \textbf{BGRL} design and $rLap$ variants.}
\label{table:rlap_ablation_bgrl}
\vskip 0.15in
\begin{center}
\begin{small}
\begin{sc}
\begin{tabular}{c|c|c|c|c|c}
\toprule
Augmentor & PROTEINS & IMDB-BINARY & MUTAG & IMDB-MULTI & NCI1\\
\midrule
rLap-rand-asc & $\mathbf{84.34 \pm 4.03}$ & $74.8 \pm 10.17$  & $\mathbf{81.5 \pm 5.39}$ & $\mathbf{59.47 \pm 7.42}$ & $\mathbf{75.15 \pm 3.57}$ \\
rLap-rand-desc & \underline{$83.72 \pm 4.3$} & $75.1 \pm 8.87$ & $80.6 \pm 8.66$ & $58.73 \pm 7.13$ & $73.77 \pm 3.34$  \\
rLap-rand-rand & $83.14 \pm 6.38$ & $74.2 \pm 7.74$ & $80.3 \pm 5.39$ & $57.07 \pm 5.52$ & $74.72 \pm 2.97$ \\
rLap-deg-asc & $83.62 \pm 5.44$ & \underline{$79.4 \pm 5.94$} & \underline{$81.0 \pm 10.26$} & $58.6 \pm 4.73$  & \underline{$74.8 \pm 2.89$} \\
rLap-deg-desc & $83.48 \pm 4.91$ & $79.2 \pm 7.3$  & $80.5 \pm 5.22$ & $58.53 \pm 5.2$ & $73.36 \pm 4.17$ \\
rLap-deg-rand & $83.55 \pm 5.62$ &  $79.0 \pm 8.75$ & $80.8 \pm 10.2$ & \underline{$59.36 \pm 5.39$} & $74.3 \pm 2.36$ \\
Vertex coarsening & $82.3 \pm 6.72$ & $\mathbf{80.5 \pm 5.71}$ & $78.5 \pm 7.76$ & $56.67 \pm 5.38$ & $73.58 \pm 2.37$ \\
\bottomrule
\end{tabular}
\end{sc}
\end{small}
\end{center}
\vskip -0.1in
\end{table*}

\subsection{Maximum singular value and edge count analysis}

For a better understanding of the effectiveness of these variants, we set $\gamma=0.5$ and vary the choices of $o_v, o_n$ to track $\sigma_{max}(\mR_i)$ as an indicator of $\mathcal{B}_{\mR_i}$. As analyzing hundreds of Schur complements in graph classification settings defeats the purpose of our analysis, we restricted our analysis to single graph node classification datasets to systematically track the Schur complements. Figure \ref{fig:rlap_ablation_cora}(a) plots the trend of $\sigma_{max}(\mR_i)$ using $vertex-coarsening$ and $rLap$ variants with a default setting of $o_n=asc$ on the CORA dataset. Figure \ref{fig:rlap_ablation_cora}(b) plots the edge count trend of $\mR_i$ for these three techniques. Since $rLap$ with $o_v=rand$ gave the best performance on most of the benchmark datasets, we tracked $\sigma_{max}(\mR_i)$ for $rLap$ with $o_v=rand$ and $o_n=asc, desc, rand$ in Figure \ref{fig:rlap_ablation_cora}(c), with an edge count trend being presented in Figure \ref{fig:rlap_ablation_cora}(d). A similar analysis is conducted for AMAZON-PHOTO (see Figure \ref{fig:rlap_ablation_amazon_photo}), PUBMED (see Figure \ref{fig:rlap_ablation_pubmed}), COAUTHOR-CS (see Figure \ref{fig:rlap_ablation_coauthor_cs}) and COAUTHOR-PHY (see Figure \ref{fig:rlap_ablation_coauthor_phy}).

The $\sigma_{max}(\mR_i)$ values for $rLap$ with $o_v=rand, o_n=asc$ tend to be relatively smaller than $o_v=deg, o_n=asc$ and $vertex-coarsening$ approaches on CORA and AMAZON-PHOTO. However, the variance in their values is larger than the latter two. On the contrary, the Schur complements have relatively higher $\sigma_{max}(\mR_i)$ values for PUBMED, COAUTHOR-CS and COAUTHOR-PHY datasets. This indicates higher stochasticity in the augmented graph views when compared to $rLap$ with $o_v=deg, o_n=asc$, and $vertex-coarsening$. Also, note that degree-based elimination techniques lead to relatively dense $\mR_i$ as they generate Schur complements corresponding to a set of highly connected nodes at every iteration. Thus, $\widetilde{\gG}_1, \widetilde{\gG}_2$ generated by such approaches lead to extra computation overheads for message passing in GNN encoders. For instance, $\mR_i$ computed by $rLap$ with $o_v=rand, o_n=asc$ on COAUTHOR-PHY and AMAZON-PHOTO with $\gamma=0.5$ is sparser by $\approx 50,000, \approx 60,000$ edges respectively when compared to $\mR_i$ obtained using $rLap$ with $o_v=deg, o_n=asc$. Since $vertex-coarsening$ is based on minimum degree-based elimination, we see a close resemblance in its trends of $\sigma_{max}(\mR_i)$ and edge counts with $rLap$ based on $o_v=deg, o_n=asc$.

On the other hand, when $o_v=rand$ is fixed and $o_n$ is varied, the $rLap$ variants tend to generate $\mR_i$ with similar trends of $\sigma_{max}(\mR_i)$ across datasets. However, we observe a noticeable difference in the edge count trend as $\gamma \to 0.5$. For instance, $\mR_i$ with $o_v=rand, o_n=asc$ is sparser by $\approx 20,000$ edges when compared to $o_v=rand, o_n=desc$ on the AMAZON-PHOTO dataset. The relatively higher average degree of nodes in the AMAZON-PHOTO graph is one of the reasons for this observation. On the other hand, PUBMED has a relatively lower average degree which explains a minor difference of $\approx 2000$ edges in Schur complements of these $rLap$ variants. Intuitively, nodes with higher weights (indicative of connectivity) have a higher conditional probability of having an edge in the approximate clique. When $o_n = desc$, such nodes are processed in the initial iterations of the inner loop and later iterations lead to new edges between nodes with relatively less degree. The scenario is flipped when $o_n = asc$ and the sampled edges tend to frequently merge with existing ones (note that we merge edges to avoid multi-graphs). Thus, leading to relatively sparser $\mR_i$.

\begin{figure}[H]
\vskip 0.2in
\begin{center}
\begin{tabular}{cc}
\centering
  \includegraphics[width=70mm, height=40mm]{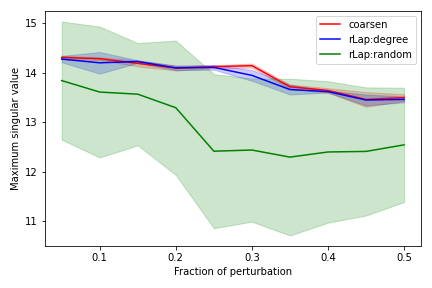} &   \includegraphics[width=70mm, height=40mm]{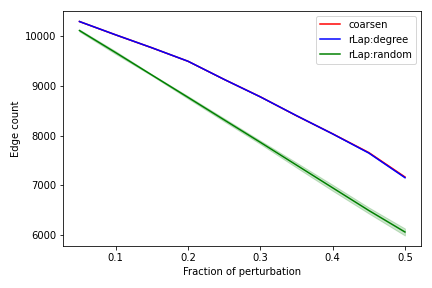} \\
(a) Trend of $\sigma_{max}(\mR_i)$ on CORA & (b)  Trend of $\mR_i$ edge count on CORA \\[2pt]
 \includegraphics[width=70mm, height=40mm]{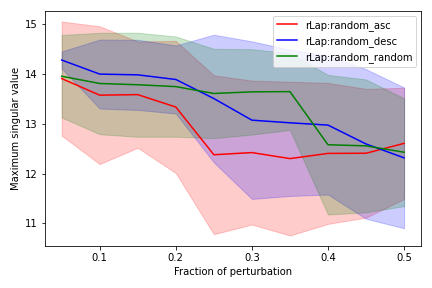} &   \includegraphics[width=70mm, height=40mm]{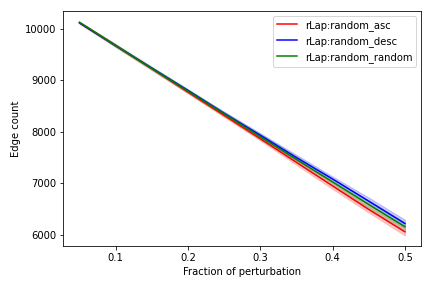} \\
(c)  Trend of $\sigma_{max}(\mR_i)$ on CORA  & (d)  Trend of $\mR_i$ edge count on CORA \\[2pt]
\end{tabular}
\caption{Plots of $\sigma_{max}(\mR_i)$ and edge count vs fraction of perturbation $\gamma$ for $rLap$ variants and $vertex-coarsening$ on CORA. For the $rLap$ variants, plots (a), (b) illustrate the trends when $o_n=asc$ and plots (c), (d) illustrate the trends when $o_v=rand$. }
\label{fig:rlap_ablation_cora}
\end{center}
\vskip -0.2in
\end{figure}


\begin{figure}[H]
\vskip 0.2in
\begin{center}
\begin{tabular}{cc}
\centering
  \includegraphics[width=70mm, height=40mm]{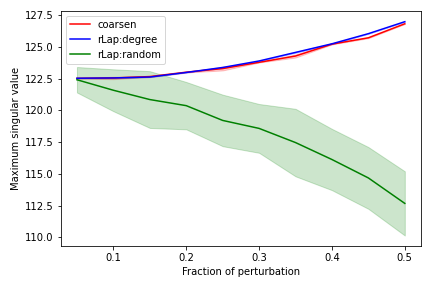} &   \includegraphics[width=70mm, height=40mm]{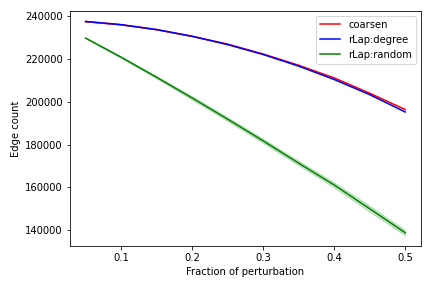} \\
(a) Trend of $\sigma_{max}(\mR_i)$ on AMAZON-PHOTO & (b)  Trend of $\mR_i$ edge count on AMAZON-PHOTO \\[2pt]
 \includegraphics[width=70mm, height=40mm]{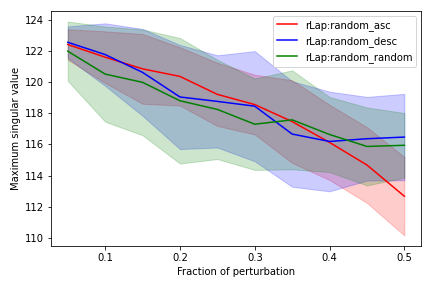} &   \includegraphics[width=70mm, height=40mm]{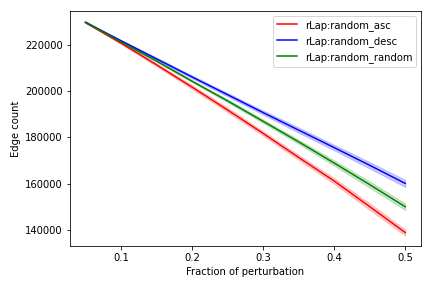} \\
(c)  Trend of $\sigma_{max}(\mR_i)$ on AMAZON-PHOTO  & (d)  Trend of $\mR_i$ edge count on AMAZON-PHOTO  \\[2pt]
\end{tabular}
\caption{Plots of $\sigma_{max}(\mR_i)$ and edge count vs fraction of perturbation $\gamma$ for $rLap$ variants and $vertex-coarsening$ on AMAZON-PHOTO. For the $rLap$ variants, plots (a), (b) illustrate the trends when $o_n=asc$ and plots (c), (d) illustrate the trends when $o_v=rand$.  }
\label{fig:rlap_ablation_amazon_photo}
\end{center}
\vskip -0.2in
\end{figure}


\begin{figure}[H]
\vskip 0.2in
\begin{center}
\begin{tabular}{cc}
\centering
  \includegraphics[width=70mm, height=40mm]{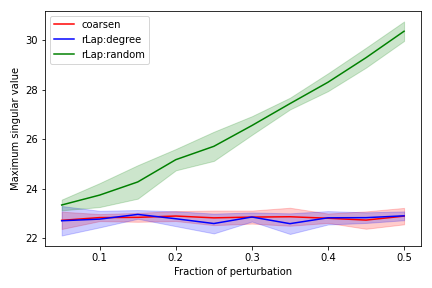} &   \includegraphics[width=70mm, height=40mm]{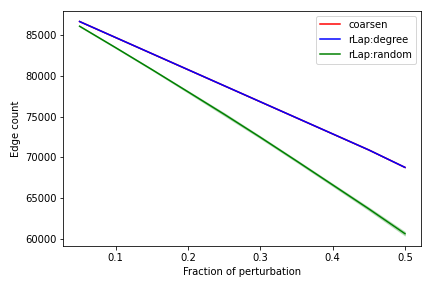} \\
(a) Trend of $\sigma_{max}(\mR_i)$ on PUBMED & (b)  Trend of $\mR_i$ edge count on PUBMED\\[2pt]
 \includegraphics[width=70mm, height=40mm]{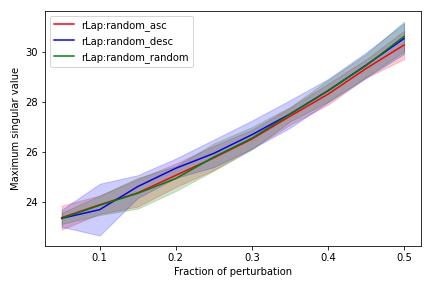} &   \includegraphics[width=70mm, height=40mm]{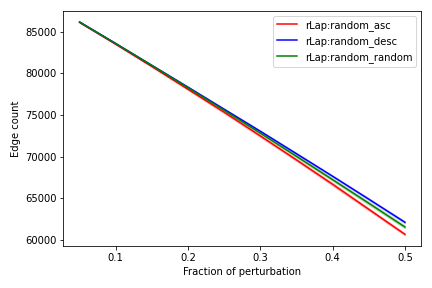} \\
(c)  Trend of $\sigma_{max}(\mR_i)$ on PUBMED  & (d)  Trend of $\mR_i$ edge count on PUBMED \\[2pt]
\end{tabular}
\caption{Plots of $\sigma_{max}(\mR_i)$ and edge count vs fraction of perturbation $\gamma$ for $rLap$ variants and $vertex-coarsening$ on PUBMED. For the $rLap$ variants, plots (a), (b) illustrate the trends when $o_n=asc$ and plots (c), (d) illustrate the trends when $o_v=rand$.   }
\label{fig:rlap_ablation_pubmed}
\end{center}
\vskip -0.2in
\end{figure}


\begin{figure}[H]
\vskip 0.2in
\begin{center}
\begin{tabular}{cc}
\centering
  \includegraphics[width=70mm, height=40mm]{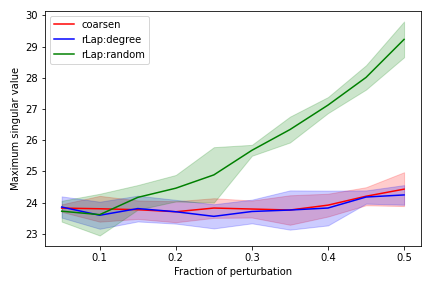} &   \includegraphics[width=70mm, height=40mm]{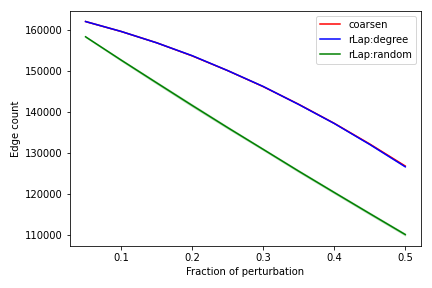} \\
(a) Trend of $\sigma_{max}(\mR_i)$ on COAUTHOR-CS & (b)  Trend of $\mR_i$ edge count on COAUTHOR-CS\\[2pt]
 \includegraphics[width=70mm, height=40mm]{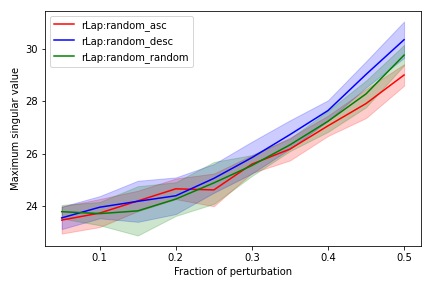} &   \includegraphics[width=70mm, height=40mm]{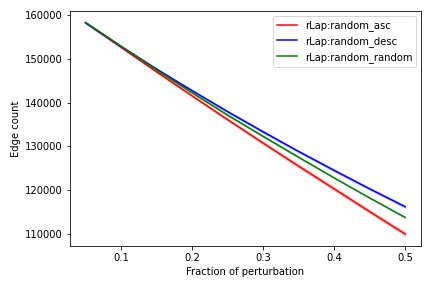} \\
(c)  Trend of $\sigma_{max}(\mR_i)$ on COAUTHOR-CS  & (d)  Trend of $\mR_i$ edge count on COAUTHOR-CS \\[2pt]
\end{tabular}
\caption{Plots of $\sigma_{max}(\mR_i)$ and edge count vs fraction of perturbation $\gamma$ for $rLap$ variants and $vertex-coarsening$ on COAUTHOR-CS. For the $rLap$ variants, plots (a), (b) illustrate the trends when $o_n=asc$ and plots (c), (d) illustrate the trends when $o_v=rand$.   }
\label{fig:rlap_ablation_coauthor_cs}
\end{center}
\vskip -0.2in
\end{figure}


\begin{figure}[H]
\vskip 0.2in
\begin{center}
\begin{tabular}{cc}
\centering
  \includegraphics[width=70mm, height=40mm]{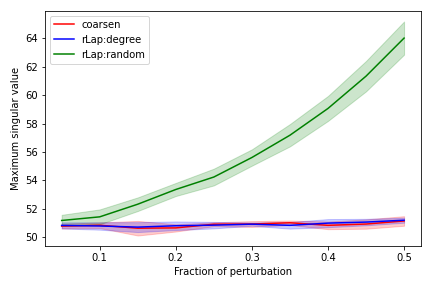} &   \includegraphics[width=70mm, height=40mm]{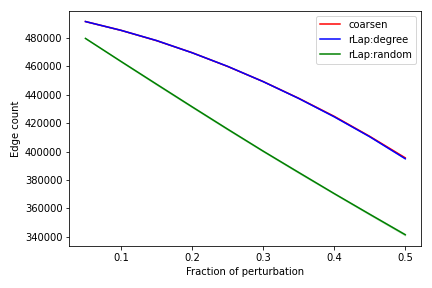} \\
(a) Trend of $\sigma_{max}(\mR_i)$ on COAUTHOR-PHY & (b)  Trend of $\mR_i$ edge count on COAUTHOR-PHY \\[2pt]
 \includegraphics[width=70mm, height=40mm]{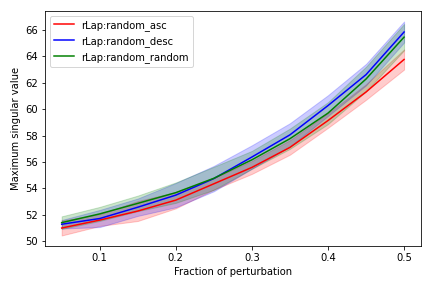} &   \includegraphics[width=70mm, height=40mm]{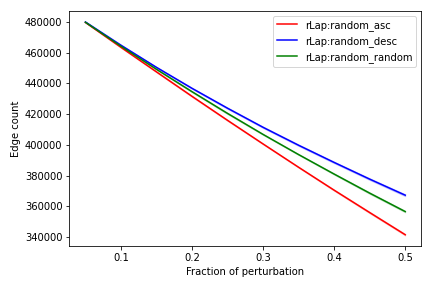} \\
(c)  Trend of $\sigma_{max}(\mR_i)$ on COAUTHOR-PHY  & (d)  Trend of $\mR_i$ edge count on COAUTHOR-PHY \\[2pt]
\end{tabular}
\caption{Plots of $\sigma_{max}(\mR_i)$ and edge count vs fraction of perturbation $\gamma$ for $rLap$ variants and $vertex-coarsening$ on COAUTHOR-PHY. For the $rLap$ variants, plots (a), (b) illustrate the trends when $o_n=asc$ and plots (c), (d) illustrate the trends when $o_v=rand$.   }
\label{fig:rlap_ablation_coauthor_phy}
\end{center}
\vskip -0.2in
\end{figure}

\section{Experimental setup}
\label{appendix:experimental_setup}

\subsection{Platform specifications}
We conduct experiments on a virtual machine with 8 Intel(R) Xeon(R) Platinum 8268 CPUs, 24GB of RAM and 1 Quadro RTX 8000 GPU with 32GB of allocated memory. For reproducible experiments, we leverage the open-source PyGCL framework by \citet{zhu2021empirical} along with PyTorch 1.12.1 \citep{paszke2019pytorch} and PyTorch-Geometric (PyG) 2.1.0 \citep{Fey/Lenssen/2019}. To ensure low latencies, the $rLap$ algorithm is written in C++, uses the Eigen library for tensor operations, is built using Bazel and has Python bindings for wider adaptability.

\subsection{Datasets}
We experiment on 10 widely used node and graph classification datasets. For node classification, we experiment on CORA \citep{mccallum2000automating}, PUBMED \citep{sen2008collective}, AMAZON-PHOTO, COAUTHOR-CS and COAUTHOR-PHY \citep{shchur2018pitfalls}.
For graph classification, we use the MUTAG \citep{debnath1991structure}, PROTEINS-full (PROTEINS) \citep{borgwardt2005protein}, IMDB-BINARY, IMDB-MULTI \citep{yanardag2015deep} and NCI1 \citep{wale2008comparison} datasets. All the datasets are available in PyG and can be imported directly. Table \ref{table:datasets} provides a summary of their statistics.

\begin{table*}[ht!]
\centering
\caption{Statistics of benchmark datasets from PyG}
\label{table:datasets}
\vskip 0.15in
\begin{center}
\begin{small}
\begin{sc}
\begin{tabular}{lccccccr}
\toprule
Dataset & Domain & Task &  \#Graphs & Avg \#nodes & Avg \#edges & \#feat & \#classes \\
\midrule
CORA & CS & Node & $1$ & $2,708$ & $10,556$ & $1,433$ & $7$\\
AMAZON-PHOTO & E-Commerce & Node & $1$ & $7,650$ & $238,162$ & $745$ & $8$ \\
PUBMED &  Medical & Node & $1$ & $19,717$ & $88,648$ & $500$ & $3$ \\
COAUTHOR-CS & CS & Node & $1$ & $18,333$ & $163,788$ & $6,805$ & $15$ \\
COAUTHOR-PHY & Phy & Node & $1$ & $34,493$ & $495,924$ & $8,415$ & $5$ \\
\midrule
MUTAG & Bio & Graph & $188$ & $\approx 17.9$ & $\approx 39.6$ & $7$ & $2$ \\
PROTEINS & Bio & Graph & $1,113$ & $\approx 39.1$ & $\approx 145.6$ & $3$ & $2$ \\
IMDB-BINARY & Movie & Graph & $1,000$ & $\approx 19.8$ & $\approx 193.1$ & $0$ & $2$ \\
IMDB-MULTI & Social & Graph & $1,500$ & $\approx 13.0$ & $\approx 131.88$ & $0$ & $3$ \\
NCI1 & Bio & Graph & $4,110$ & $\approx 29.87$ & $\approx 64.6$ & $0$ & $2$ \\
\bottomrule
\end{tabular}
\end{sc}
\end{small}
\end{center}
\vskip -0.1in
\end{table*}

\subsection{Augmentors}

\textbf{EdgeAddition:} To randomly add edges to the graph $\gG$, a mask $\mathcal{M} \in \mathbb{R}^{N \times N}$ is sampled using the Bernoulli distribution as follows: $\mathcal{M}[i,j] \sim Bern(\gamma)$, where $\gamma$ is the probability of adding an edge. Since we use two views for contrasting, masks $\mathcal{M}_1, \mathcal{M}_2$ are sampled using edge addition probabilities $\gamma_1, \gamma_2$ respectively and applied to the adjacency matrix $\mA$ of $\gG$ to obtain: $\widetilde{\mA}_1 = \mA \odot \mathcal{M}_1, \widetilde{\mA}_2 = \mA \odot \mathcal{M}_2$. Here $\odot$ indicates a bit-wise OR operator.  As per the notation defined in preliminaries, the resulting views $\widetilde{\gG}_1, \widetilde{\gG}_2$ are the graphs corresponding to $\widetilde{\mA}_1, \widetilde{\mA}_2$.

\textbf{EdgeDropping:} This technique uses the same mask sampling process as EdgeAddition for $\mathcal{M}_1, \mathcal{M}_2$ but employs a bit-wise AND operator to mask and drop the edges. 

\textbf{EdgeDroppingDegree:} Unlike the above-mentioned EdgeDropping scheme, this adaptive technique takes into account the degree of nodes at either end of an edge and performs a weighted sampling. The probability of dropping an edge between nodes $v_i, v_j$ is given by:

\begin{equation}
\label{eq:ada_aug_edd}
\gamma_{ij} = \min \bigg ( \frac{(\log s^{deg})_{max} - \log s^{deg}_{ij} }{(\log s^{deg})_{max} - (\log s^{deg})_{avg}} \cdot \gamma_e , \gamma_{\tau} \bigg)
\end{equation}

Where $s^{deg}_{ij} = [deg(v_i) + deg(v_j)]/2$ and $(\log s^{deg})_{max}, (\log s^{deg})_{avg}$ correspond to the maximum and average of log degree centrality scores for edges in the graph. Furthermore, $\gamma_e$ is the desired (user-provided) probability of dropping an edge and $\gamma_{\tau}$ is the cut-off probability to truncate high probabilities of removal \citep{zhu2021graph}.

\textbf{EdgeDroppingEVC:} The leading eigen vector centrality of a node $v_i$ is equal to $\vu[i]$ where $\mA\vu = \lambda_{max} \vu$ and $\vu$ is the eigen vector corresponding to the largest eigen value $\lambda_{max}$. By replacing $s^{deg}_{ij}$ in Eq. \ref{eq:ada_aug_edd} with $s^{evc}_{ij} = [\vu[i] + \vu[j]]/2$, we get the eigenvector centrality based edge dropping scheme.

\textbf{EdgeDroppingPR:} The page-rank centrality vector represents the weights computed by the page rank algorithm \citep{page1999pagerank} and is given by $\vp = \alpha \mA \mD^{-1}\vp + \mI$ \citep{zhu2021graph}. By replacing $s^{deg}_{ij}$ in Eq. \ref{eq:ada_aug_edd} with $s^{pr}_{ij} = [\vp[i] + \vp[j]]/2$, we get the page-rank centrality based edge dropping scheme.

\textbf{NodeDropping:} Similar to the edge dropping scheme, this technique drops a node $v_i$ from the graph based on the $Bern(\gamma)$ distribution, where $\gamma$ is the probability of dropping a node. Computationally, this is equivalent to setting the row and column values corresponding to the dropped node to $\vzero$.

\textbf{RandomWalkSubgraph:} This is a sub-graph sampling technique where a random walk with restart is performed on $\gG$ and the sub-graph induced by this random walk is taken as the augmented view \citep{zhu2021empirical}.

\textbf{PPRDiffusion, MarkovDiffusion:} The diffusion operator on a graph (introduced in Eq. \ref{eq:diffusion}) is an effective way of capturing global structural information. The PPRDiffusion and MarkovDiffusion schemes are based on this formulation as follows:
\begin{align}
\label{eq:aug_ppr_mdk}
\begin{split}
\mS^{PPR} &= \alpha(\mI - (1-\alpha)\mT)^{-1} \\
\mS^{MD} &= \frac{1}{K}\sum_{i=1}^K (\alpha \mI + (1-\alpha) \mT^i)
\end{split}
\end{align}
Where $K$ represents the maximum step count for the markov diffusion \citep{zhu2020simple}, $\alpha$ in $\mS^{PPR}, \mS^{MD}$ indicates the dampening factor and $\mT = \mA\mD^{-1}$. We use $\alpha = 0.2$ and $K = 10$ in our experiments. Next, to address the increase in edge connectivity of the diffused graph, we perform sparsification based on thresholding with $\epsilon=10^{-4}$ as proposed by \citet{klicpera2019diffusion}. Note that for diffusion based augmentors, the second view is the original graph without any augmentation and the notion of perturbation $\gamma$ doesn't apply.

\textbf{rLap}: Based on our setup and Algorithm \ref{alg:rlap}, the augmented view is generated by $rLap(\gG, \gamma, o_v, o_n)$, where $\gG$ is the input weighted graph, $\gamma$ indicates the fraction of nodes to eliminate, $o_v$ is the node elimination scheme and $o_n$ is the neighbor ordering scheme. The algorithm leverages the weighted laplacian $\mL$ to perform GE-based elimination for $\gamma|\gV|$ iterations and returns the randomized Schur complement matrix (which is also a laplacian) and represents the augmented view of $\gG$.

\textbf{Feature Masking}: Based on our setup, a graph $\gG$ is associated with a feature matrix $\mX = [\vx_1, \cdots, \vx_N]^{\top} \in \sR^{N \times d}$. For every row $\vx_i \in \sR^d, i \in {1, \cdots, N}$, a mask $m_i \in \sR^d$ is sampled from a Bernoulli distribution as follows: $m_i[j] \sim Bern(1-\gamma), \forall j \in {1, \cdots, d}$, where $\gamma$ is the probability of masking an entry. These masks are then applied on the rows of $\mX$ to obtain the masked features $\widehat{\mX} = [\vx_1 \odot m_1 , \cdots, \vx_N \odot m_N]^{\top} \in \sR^{N \times d}$. Here $\odot$ indicates a bit-wise AND operator.

\subsection{Contrastive objectives}

We employ Information Noise Contrastive Estimation (InfoNCE), Jensen-Shannon Divergence (JSD), and Bootstrap Latent (BL) objectives in our experiments.

\textbf{InfoNCE} \citep{oord2018representation}: Without loss of generality, consider a sample $u_i$ for which $\mathcal{S}(u_i)$ and $\mathcal{D}(u_i)$ indicate the set of similar and dissimilar samples respectively. Samples can indicate nodes or graphs as per the context. The InfoNCE loss is now given by:
\begin{equation}
    \ell_{\textit{InfoNCE}} = -\frac{1}{|\mathcal{S}(u_i)|} \sum\limits_{s_j \in \mathcal{S}(u_i)} \log \Bigg(\frac{e^{\varphi(u_i, s_j)/t}}{e^{\varphi(u_i, s_j)/t} + \sum\limits_{d_j \in \mathcal{D}(u_i)}e^{\varphi(u_i, d_j)/t}} \Bigg)
\end{equation}
Where $t \in \sR$ is the scaling/temperature parameter and $\varphi(.)$ measures the similarity of features $\psi(.)$ learned by the GNN encoder (followed by projector as per design) as follows:
\begin{equation}
    \varphi(u_i, s_j) = \frac{\psi(u_i)^{\top}\psi(s_j)}{\norm{\psi(u_i)}_2\norm{\psi(s_j)}_2}
\end{equation}
\textbf{JSD} \citep{lin1991divergence}: With the same setup as InfoNCE, we use JSD loss based on softplus and sigmoid function $sig(.)$ as:

\begin{equation}
    \ell_{\textit{JSD}} = -\frac{1}{|\mathcal{S}(u_i)|} \sum\limits_{s_j \in \mathcal{S}(u_i)} \log \big( 1 + e^{-sig(\psi(u_i)^{\top}\psi(s_j))} \big)  - \frac{1}{|\mathcal{D}(u_i)|} \sum\limits_{d_j \in \mathcal{D}(u_i)} \log \big( 1 + e^{sig(\psi(u_i)^{\top}\psi(d_j))} \big) 
\end{equation}

\textbf{BL} \citep{thakoor2021bootstrapped}: The BL loss is independent of the dissimilar or negative samples $\mathcal{D}(u_i)$ and attempts to maximize the cosine similarity between embeddings of similar samples. We follow the design presented in \citet{thakoor2021bootstrapped} and employ an online and target encoder to generate embeddings and minimize the following:
\begin{equation}
    \ell_{\textit{BL}} = -\frac{1}{|\mathcal{S}(u_i)|} \sum\limits_{s_j \in \mathcal{S}(u_i)} \frac{\varphi_{\textit{on}}(u_i)^{\top}\varphi_{\textit{tar}}(s_j) }{\norm{\varphi_{\textit{on}}(u_i)}_2\norm{\varphi_{\textit{tar}}(s_j)}_2 }
\end{equation}

We follow the standard benchmarking procedure by \citet{zhu2021empirical} and incorporate additional constraints presented in  \citet{thakoor2021bootstrapped}, such as the exponential moving average-based update of target encoder parameters.

\subsection{GCL framework designs}

\textbf{GRACE:} Figure \ref{fig:gcl_grace} illustrates the design for a GRACE \citep{zhu2020deep} inspired framework. The input graph $\gG=(\gV, \gE, w)$ is augmented by $T_1, T_2$ to generate views $\widetilde{\gG}_1, \widetilde{\gG}_2$. These graphs are encoded using a shared GNN $f$ and a shared MLP projector $z_n$ to generate node features. The node-level features across these two views are contrasted using the InfoNCE objective under `l-l' mode. We employ this design for node classification where congruent counterparts of nodes form the positive/similar samples and the rest of the nodes are selected as negative/dissimilar samples. We deviate from the design of \citet{zhu2020deep} by incorporating only inter-view dissimilar samples to reduce the computational overheads.

\begin{figure*}[ht]
\vskip 0.2in
\begin{center}
\centerline{\includegraphics[width=0.9\textwidth]{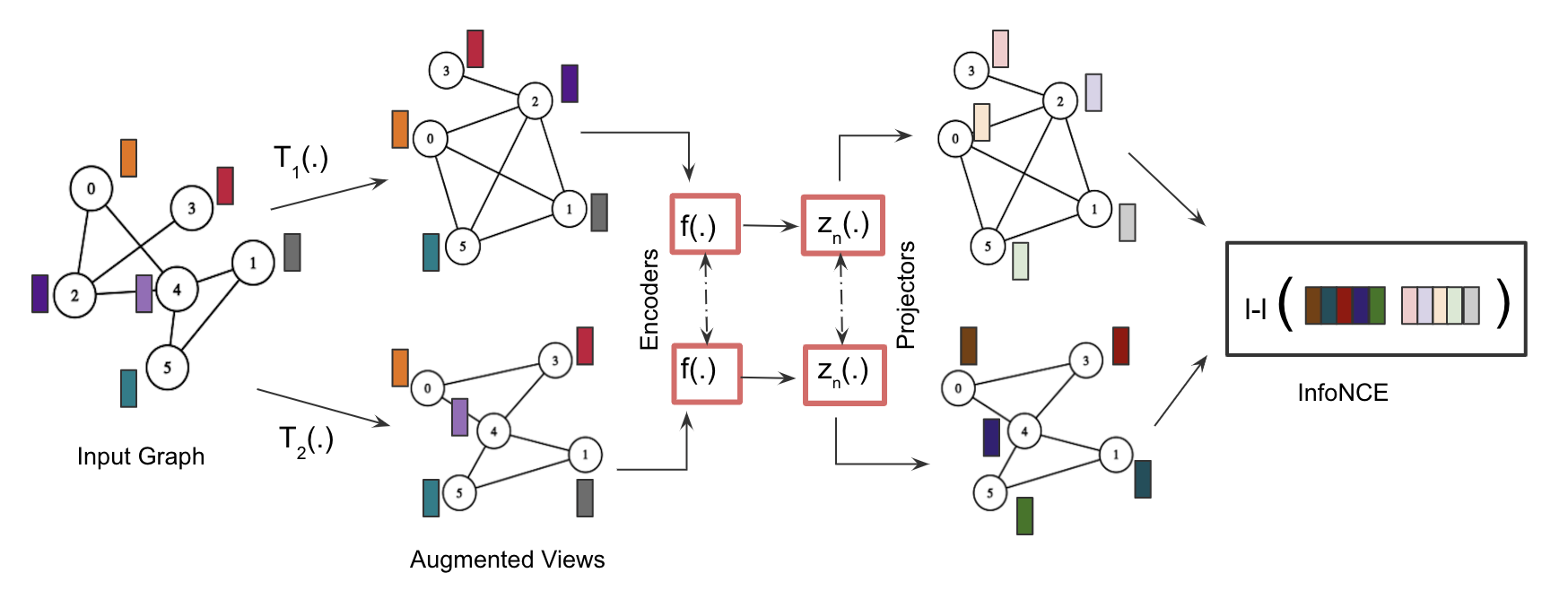}}
\caption{GRACE design with shared encoder $f$, shared node feature projector $z_n$, l-l contrastive mode and InfoNCE objective.} 
\label{fig:gcl_grace}
\end{center}
\vskip -0.2in
\end{figure*}

\textbf{MVGRL:} Figure \ref{fig:gcl_mvgrl} illustrates the design for a MVGRL \citep{hassani2020contrastive} inspired framework. The input graph $\gG=(\gV, \gE, w)$ is augmented by $T_1, T_2$ to generate views $\widetilde{\gG}_1, \widetilde{\gG}_2$. These graphs are encoded using dedicated GNNs $f_1, f_2$ and projected using a shared MLP $z_n$ to compute node features. The node embeddings computed by $f_1, f_2$ are aggregated using a permutation invariant function $\bigoplus$ (eg: mean) and projected using a shared MLP $z_g$ to obtain graph features. Finally, these graph and node-level features and contrasted using the JSD objective under `g-l' mode. In this mode, when only a single graph $\gG$ is available in the dataset, the positive samples for graph features of $\widetilde{\gG}_1$ include the node features of $\widetilde{\gG}_2$ with negative samples being their noisy versions. The same procedure is applied to the graph features of $\widetilde{\gG}_2$ and node features of $\widetilde{\gG}_1$ \citep{velickovic2019deep, hassani2020contrastive, zhu2021empirical}.

\begin{figure*}[ht]
\vskip 0.2in
\begin{center}
\centerline{\includegraphics[width=0.9\textwidth]{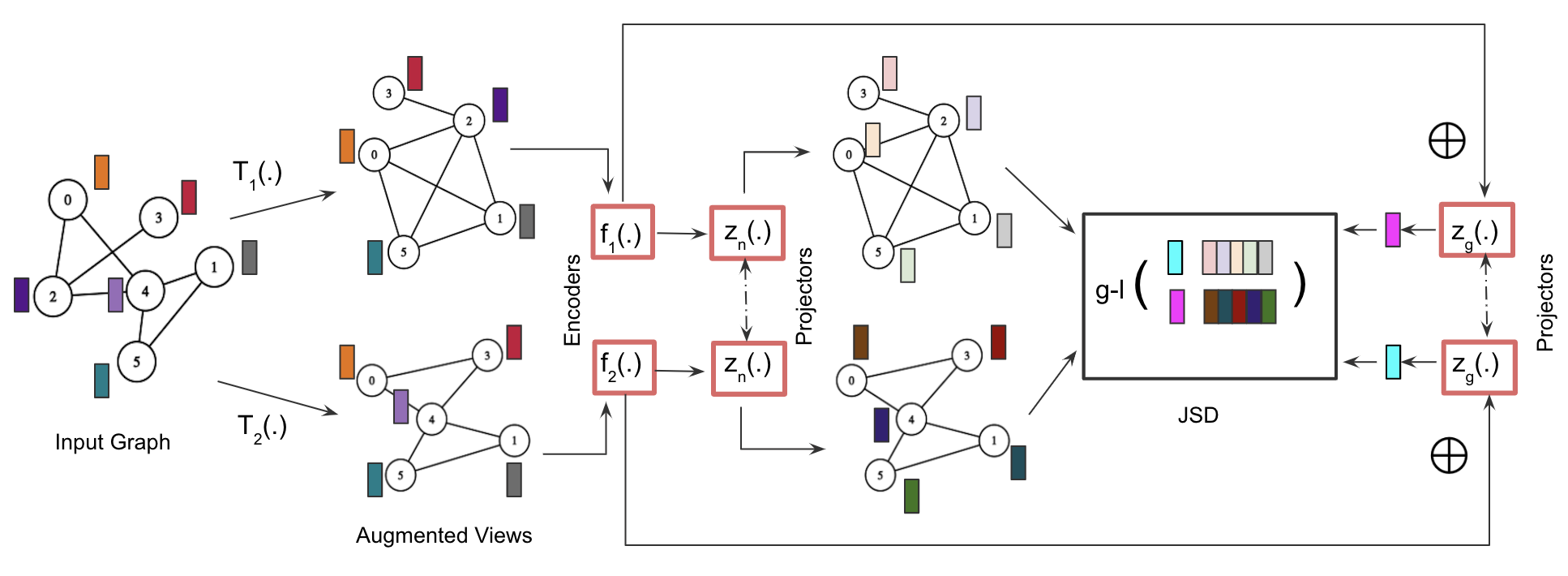}}
\caption{MVGRL design with dedicated encoders $f_1, f_2$, shared node feature projector $z_n$, shared graph feature projector $z_g$, g-l contrastive mode and JSD objective.} 
\label{fig:gcl_mvgrl}
\end{center}
\vskip -0.2in
\end{figure*}

\textbf{GraphCL:} Figure \ref{fig:gcl_graphcl} illustrates the design for a GraphCL \citep{you2020graph} inspired framework. The input graph $\gG=(\gV, \gE, w)$ is augmented by $T_1, T_2$ to generate views $\widetilde{\gG}_1, \widetilde{\gG}_2$. These graphs are encoded using a shared GNN $f$ to obtain the node embeddings. These node embeddings are aggregated using a permutation invariant function $\bigoplus$ (eg: mean) and projected using a shared MLP $z_g$ to obtain graph features. Finally, these graph features are contrasted using the InfoNCE objective under `g-g' mode. For the graph features of $\widetilde{\gG}_1$, the graph features of $\widetilde{\gG}_2$ are considered as positive samples and the graph features of other graphs in the training batch are treated as negative samples.
 
\begin{figure*}[ht]
\vskip 0.2in
\begin{center}
\centerline{\includegraphics[width=0.9\textwidth]{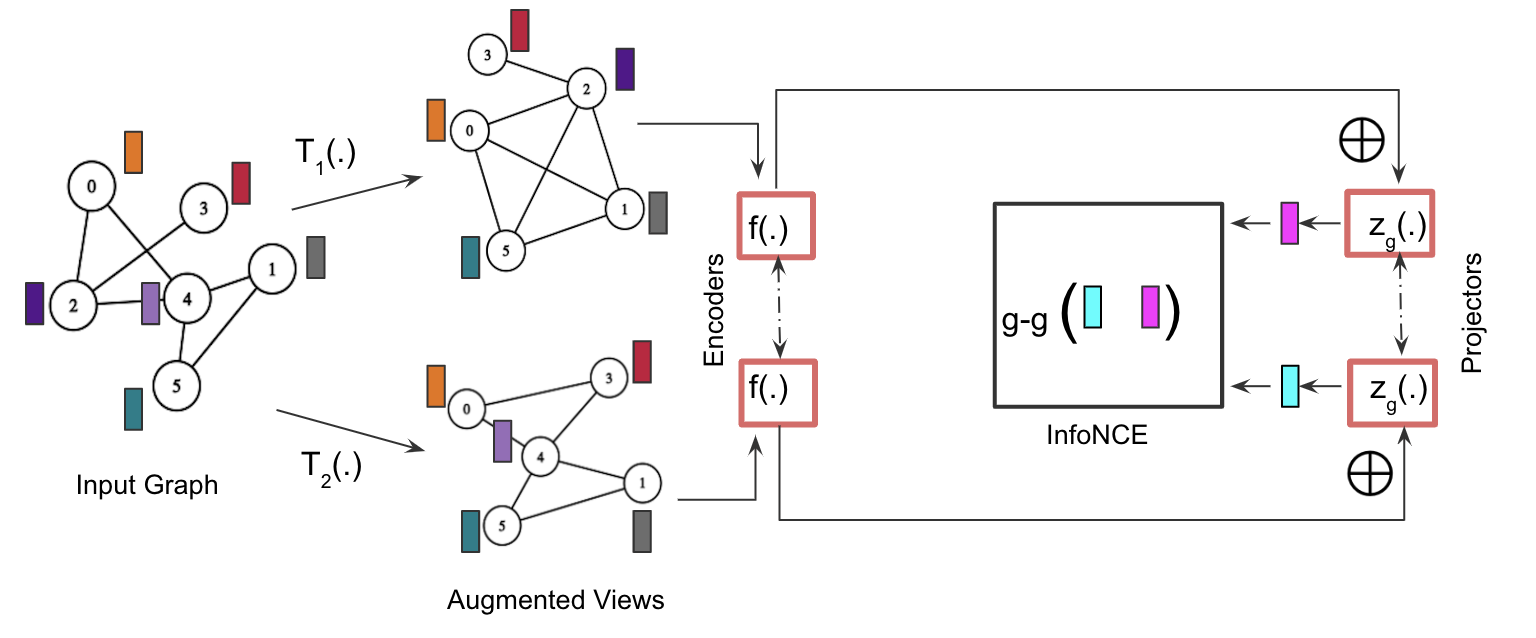}}
\caption{GraphCL design with shared encoder $f$, shared graph feature projector $z_g$, g-g contrastive mode and InfoNCE objective.} 
\label{fig:gcl_graphcl}
\end{center}
\vskip -0.2in
\end{figure*}

\textbf{BGRL:} Figure \ref{fig:gcl_bgrl_g2l} illustrates the design for a BGRL \citep{thakoor2021bootstrapped} inspired framework. The input graph $\gG=(\gV, \gE, w)$ is augmented by $T_1, T_2$ to generate views $\widetilde{\gG}_1, \widetilde{\gG}_2$. Unlike the previously mentioned frameworks, these graphs are encoded using an online and target GNNs $f_1, f_2$ to compute node embeddings. The online encoder $f_1$ computes node-level embeddings for both the views and projects them using an MLP $z_{n1}$. The target encoder $f_2$ and MLP projector $z_{n2}$ repeat the same procedure, followed by computing the graph-level features\footnote{Unlike MVGRL where MLP-based projections $z_{g1}, z_{g2}$ are employed, BGRL skips the graph feature projection step.} using a permutation invariant function $\bigoplus$ (eg: mean). Finally, these graph and node-level features and contrasted using the BL objective under `g-l' mode. The parameters of $f_2$ are updated based on an exponential moving average of $f_1$ parameters. This design is independent of negative samples and avoids the computational overheads of employing graph and node features across the training batch.

\begin{figure*}[ht]
\vskip 0.2in
\begin{center}
\centerline{\includegraphics[width=0.9\textwidth]{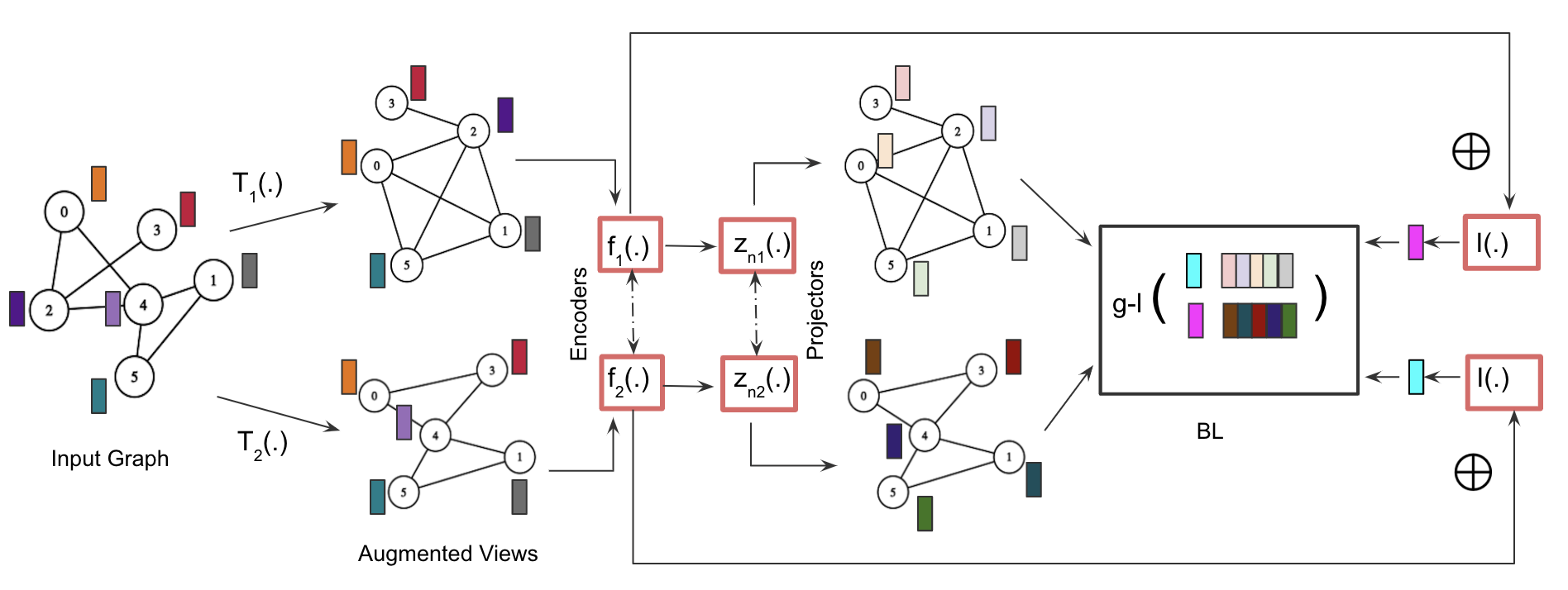}}
\caption{BGRL design with online, target encoders $f_1, f_2$, their node projectors $z_{n1}, z_{n2}$, g-l contrastive mode and BL objective.} 
\label{fig:gcl_bgrl_g2l}
\end{center}
\vskip -0.2in
\end{figure*}

\subsection{Evaluation protocols}

Generally, GCL frameworks differ in the choice of augmentors, encoder designs, contrastive modes and objective functions. Thus, it is unclear how the performance of an augmentor can be fairly tested. To address this issue, we fix as many configurations as possible and focus solely on the impact of the augmentor. Our reasoning is based on the well-established benchmark study by \citet{zhu2021empirical}. Adhering to this approach, we leverage the designs of widely adopted frameworks such as GRACE \citep{zhu2020deep}, MVGRL \citep{hassani2020contrastive}, GraphCL \citep{you2020graph} and BGRL \citep{thakoor2021bootstrapped}, where topological augmentation is necessary and evaluate on unsupervised node and graph classification tasks. These frameworks were chosen so as to extensively experiment on encoder designs, contrastive modes and objectives. We leverage graph convolution network (GCN) \citep{kipf2016semi} and graph isomorphism network (GIN) \cite{xu2018powerful} based encoders for node and graph-based tasks respectively. Also, based on the empirical benchmark observations in \citet{zhu2021empirical}, we perform feature masking with a constant masking ratio of $0.3$ in all the experiments and perform an extensive grid-search over the hyper-parameters to find their ideal values. The perturbation ratios $\gamma_1, \gamma_2$ are selected from $\{0.0, 0.1, 0.2, 0.3, 0.4, 0.5\}$ to prevent excessive graph corruption. The number of layers in the encoder are chosen from $\{ 2, 4, 8 \}$, hidden feature dimensions are chosen from $\{ 128, 256, 512 \}$, learning rate is chosen from $\{10^{-2}, 10^{-3}, 10^{-4}\}$, weight decay is set to $10^{-5}$ and the maximum epoch count is set to $2000$. An early stopping strategy with a patience interval of 50 epochs with respect to contrastive loss is also employed. Finally, we choose the Adam optimizer for learning the embeddings. For testing, we follow a linear evaluation protocol \citep{velickovic2019deep, zhu2021empirical} where the embeddings learned by the encoders in unsupervised settings are classified using a logistic regression classifier with a train, validation and test split of $10\%, 10\%, 80\%$ respectively. The process is repeated with $10$ random splits and the classification accuracies are reported. The selected hyper-parameters based on a grid search are reported in Table \ref{table:hp_grace}, \ref{table:hp_mvgrl}, \ref{table:hp_graphcl}, \ref{table:hp_bgrl} and can be reproduced with the provided code and instructions.

\begin{table}[H]
\centering
\caption{Hyper-parameters for GRACE-based design for node-classification}
\label{table:hp_grace}
\vskip 0.15in
\begin{center}
\begin{small}
\begin{sc}
\begin{tabular}{lccccccccr}
\toprule
Dataset & mode & loss & $\tau$ & $\gamma_1$ & $\gamma_2$ & L & lr & wd & hidden dim \\
\midrule
CORA & l-l & InfoNCE & $0.2$ & $0.5$ & $0.4$ & $2$ & $10^{-4}$ & $10^{-5}$ & $256$\\
AMAZON-PHOTO & l-l & InfoNCE & $0.2$ & $0.3$ & $0.4$ & $2$ & $10^{-2}$ & $10^{-5}$ & $256$\\ 
PUBMED & l-l & InfoNCE & $0.2$ & $0.4$ & $0.1$ & $2$ & $10^{-3}$ & $10^{-5}$ & $256$\\
COAUTHOR-CS & l-l & InfoNCE & $0.4$ & $0.5$ & $0.4$ & $2$ & $10^{-4}$ & $10^{-5}$ & $256$\\
COAUTHOR-PHY & l-l & InfoNCE & $0.4$ & $0.5$ & $0.5$ & $2$ & $10^{-2}$ & $10^{-5}$ & $128$\\
\bottomrule
\end{tabular}
\end{sc}
\end{small}
\end{center}
\vskip -0.1in
\end{table}

\begin{table}[H]
\centering
\caption{Hyper-parameters for MVGRL-based design for node-classification}
\label{table:hp_mvgrl}
\vskip 0.15in
\begin{center}
\begin{small}
\begin{sc}
\begin{tabular}{lccccccccr}
\toprule
Dataset & mode & loss & $\gamma_1$ & $\gamma_2$ & L & lr & wd & hidden dim \\
\midrule
CORA & g-l & JSD & $0.0$ & $0.4$ &  $2$ & $10^{-3}$ & $10^{-5}$ & $512$\\
AMAZON-PHOTO & g-l & JSD & $0.0$ & $0.3$ & $2$ & $10^{-4}$ & $10^{-5}$ & $512$\\
PUBMED & g-l & JSD & $0.1$ & $0.3$ & $2$ & $10^{-3}$ & $10^{-5}$ & $512$\\
COAUTHOR-CS & g-l & JSD & $0.0$ & $0.2$ & $2$ & $10^{-4}$ & $10^{-5}$ & $512$\\
COAUTHOR-PHY & g-l & JSD & $0.1$ & $0.4$ & $2$ & $10^{-3}$ & $10^{-5}$ & $512$\\
\bottomrule
\end{tabular}
\end{sc}
\end{small}
\end{center}
\vskip -0.1in
\end{table}

\begin{table}[H]
\centering
\caption{Hyper-parameters for GraphCL-based design for graph-classification}
\label{table:hp_graphcl}
\vskip 0.15in
\begin{center}
\begin{small}
\begin{sc}
\begin{tabular}{lccccccccr}
\toprule
Dataset & mode & loss & $\tau$ &$ \gamma_1$ & $\gamma_2$ & L & lr & wd & hidden dim \\
\midrule
PROTEINS & g-g & InfoNCE & $0.5$ & $0.2$ & $0.2$ & $2$ & $10^{-2}$ & $10^{-5}$ & $128$\\
IMDB-BINARY & g-g & InfoNCE & $0.5$ & $0.3$ & $0.5$ & $2$ & $10^{-2}$ & $10^{-5}$ & $128$\\
MUTAG & g-g & InfoNCE & $0.5$ & $0.2$ & $0.2$ & $2$ & $10^{-3}$ & $10^{-5}$ & $128$\\
IMDB-MULTI & g-g & InfoNCE & $0.5$ & $0.2$ & $0.2$ & $2$ & $10^{-2}$ & $10^{-5}$ & $128$\\
NCI1 & g-g & InfoNCE & $0.2$ & $0.2$ & $0.3$ & $2$ & $10^{-3}$ & $10^{-5}$ & $128$\\
\bottomrule
\end{tabular}
\end{sc}
\end{small}
\end{center}
\vskip -0.1in
\end{table}

\begin{table}[H]
\centering
\caption{Hyper-parameters for BGRL-based design for graph-classification}
\label{table:hp_bgrl}
\vskip 0.15in
\begin{center}
\begin{small}
\begin{sc}
\begin{tabular}{lccccccccr}
\toprule
Dataset & mode & loss & $\gamma_1$ & $\gamma_2$ & L & lr & wd & hidden dim \\
\midrule
PROTEINS & g-l & BL & $0.1$ & $0.1$ & $2$ & $10^{-2}$ & $10^{-5}$ & $256$\\
IMDB-BINARY & g-l & BL & $0.5$ & $0.3$ & $2$ & $10^{-3}$ & $10^{-5}$ & $128$\\
MUTAG & g-l & BL & $0.2$ & $0.2$ & $2$ & $10^{-3}$ & $10^{-5}$ & $128$\\
IMDB-MULTI & g-l & BL & $0.2$ & $0.2$ & $2$ & $10^{-3}$ & $10^{-5}$ & $128$\\
NCI1 & g-l & BL & $0.3$ & $0.1$ & $2$ & $10^{-3}$ & $10^{-5}$ & $256$\\
\bottomrule
\end{tabular}
\end{sc}
\end{small}
\end{center}
\vskip -0.1in
\end{table}

\subsection{Baselines}

We report the baseline node and graph classification performance from previously published reports \citep{hassani2020contrastive, you2020graph, zhu2021graph, xu2021infogcl}. For node classification, we leverage raw features \citep{velickovic2019deep}, DeepWalk \citep{perozzi2014deepwalk}, Graph Auto Encoder (GAE) \citep{kipf2016variational}, Deep Graph Infomax (DGI) \citep{velickovic2019deep}, GCN networks and Graph Attention Networks (GAT) \citep{velivckovic2017graph} (see Table \ref{table:node_baselines}). For graph classification, we leverage Graphlet Kernel (GK) \citep{shervashidze2009efficient}, Weisfeiler-Lehman Kernel (WL) \citep{shervashidze2011weisfeiler}, Deep Graph Kernel (DGK) \citep{yanardag2015deep}, Multi-scale Laplacian Graph Kernel (MLG) \citep{kondor2016multiscale}, node2vec \citep{grover2016node2vec}, sub2vec \citep{adhikari2018sub2vec}, graph2vec \citep{narayanan2017graph2vec}, InfoGraph \citep{sun2019infograph}, GIN, GCN and GAT networks (see Table \ref{table:graph_baselines}).

\begin{table*}[ht!]
\centering
\caption{Baseline node classification accuracies from published reports. }
\label{table:node_baselines}
\vskip 0.15in
\begin{center}
\begin{small}
\begin{sc}
\begin{tabular}{c|c|c|c|c|c}
\toprule
Method & CORA & Amazon-Photo & PubMed & Coauthor-CS & Coauthor-Phy \\
\midrule
Raw Feat & $47.90 \pm 0.40$ & $ 78.53 \pm 0.00$ & $ 69.10 \pm 0.30$ & $90.37 \pm 0.00$ & $93.58 \pm 0.00$  \\
DeepWalk & $67.20 \pm 0.00$ & $89.44 \pm 0.11$ & $65.30 \pm 0.00$ & $84.61 \pm 0.22$ & $91.77 \pm 0.15$ \\
DeepWalk+feat & $70.70 \pm 0.60$ & $90.05 \pm 0.08$ & $74.30 \pm 0.90$ & $87.70 \pm 0.04$ & $ 94.90 \pm 0.09$ \\
GAE & $71.50 \pm 0.40$ & $91.62 \pm 0.13$ & $72.10 \pm 0.50$ & $90.01 \pm 0.71$ & $94.92 \pm 0.07$ \\
DGI & $82.30 \pm 0.60$ & $ 91.61 \pm 0.22$ & $76.80 \pm 0.60$ & $92.15 \pm 0.63$ &  $94.51 \pm 0.52$ \\
\midrule
GCN & $81.50 \pm 0.00$ & $92.42 \pm 0.22$ & $79.0 \pm 0.00$ & $ 93.03 \pm 0.31$ & $95.65 \pm 0.16$  \\
GAT & $83.0 \pm 0.70$ & $92.56 \pm 0.35$ & $79.0 \pm 0.30$ & $92.31 \pm 0.24$ & $95.47 \pm 0.15$ \\
\bottomrule
\end{tabular}
\end{sc}
\end{small}
\end{center}
\vskip -0.1in
\end{table*}

\begin{table*}[ht!]
\centering
\caption{Baseline graph classification accuracies from published reports. }
\label{table:graph_baselines}
\vskip 0.15in
\begin{center}
\begin{small}
\begin{sc}
\begin{tabular}{c|c|c|c|c|c}
\toprule
Method & MUTAG & PROTEINS & IMDB-BINARY & IMDB-MULTI & NCI1 \\
\midrule
GK & $81.70 \pm 2.10$ & $-$ & $65.87 \pm 0.98$ & $43.90 \pm 0.40$ & $-$ \\
WL & $80.72 \pm 3.00$ & $72.92 \pm 0.56$ & $72.30 \pm 3.44$ & $47.0 \pm 0.50$ & $80.01 \pm 0.50$ \\
DGK  & $87.44 \pm 2.72$ & $73.30 \pm 0.82$ & $66.96 \pm 0.56$ & $44.60 \pm 0.50$ & $80.31 \pm 0.46$\\
MLG & $87.9 \pm 1.60$ & $-$ & $66.6 \pm 0.30$ & $41.2 \pm 0.00$ & $80.8 \pm 1.30$ \\
\midrule
node2vec  & $72.63 \pm 10.20$ & $57.49 \pm 3.57$ & $-$ & $-$ & $54.89 \pm 1.61$ \\
sub2vec & $61.05 \pm 15.80$ & $53.03 \pm 5.55$ & $55.26 \pm 1.54$ & $36.70 \pm 0.80$ & $52.84 \pm 1.47$ \\
graph2vec  & $83.15 \pm 9.25$ & $73.30 \pm 2.05$ & $71.10 \pm 0.54$ & $ 50.40 \pm 0.90$ & $73.22 \pm 1.81$ \\
InfoGraph & $89.01 \pm 1.13$ & $74.44 \pm 0.31$ & $73.03 \pm 0.87$ & $49.70 \pm 0.50$ & $76.20 \pm 1.06$ \\
\midrule
GIN & $89.4 \pm 5.60$ & $76.2 \pm 2.80$ & $ 75.1 \pm 5.10$ & $52.3 \pm 2.80$ & $ 82.7 \pm 1.70$ \\
GCN & $85.6 \pm 5.80$ & $75.2 \pm 3.60$ & $74.0 \pm 3.4$ & $51.9 \pm 3.8$ & $80.2 \pm 2.0$ \\
GAT & $89.4 \pm 6.10$ & $ 74.7 \pm 4.00$ & $70.5 \pm 2.30$ & $47.8 \pm 3.10$ & $66.6 \pm 2.20$ \\
\bottomrule
\end{tabular}
\end{sc}
\end{small}
\end{center}
\vskip -0.1in
\end{table*}

\section{Augmentor Overheads}
\label{appendix:aug_overhead}

We establish an augmentor overhead benchmark in terms of memory usage and latency for an extensive range of techniques on the benchmark datasets. Specifically, we measure:

\textbf{MEMORY:} The memory (in MB) that is required solely by the augmentation phase.

\textbf{LATENCY(CPU):} The wall-clock time (in seconds) for performing the augmentation.

\textbf{LATENCY(w/ GPU):} The wall-clock time (in seconds) for performing the augmentation when GPU is available.

We run $10$ experiments for each augmentor-dataset combination with a fixed perturbation ratio of $0.5$ (and use a batch size of $128$ for graph datasets) to report the mean and standard deviation of the metrics in Table \ref{table:node_aug_stats}, \ref{table:graph_aug_stats}. From Table \ref{table:node_aug_stats}, we can observe that simple NodeDropping and EdgeDropping schemes are computationally the most efficient, with EdgeDroppingEVC and diffusion-based approaches being the least. The code for EdgeDroppingEVC has been adopted from the original paper by \citet{zhu2021graph} and leverages the NetworkX library to compute centrality scores. Furthermore, since this approach lacks GPU support, the hardware acceleration is minimal. In fact, the overheads of data transfers from GPU to CPU and the construction of NetworkX graphs significantly increase the memory consumption and latencies for this technique. Similar patterns can be observed for other augmentors as well where the GPU-CPU overheads dominate the parallelization benefits. On the contrary, diffusion-based approaches leverage GPU acceleration for heavy matrix operations and gain significant speedups to achieve reasonable latencies. A similar pattern can be observed in Table \ref{table:graph_aug_stats} for graph datasets with a batch size of $128$. Note that some of the results show $0$ variance due to numerical rounding.

The overheads of $rLap$ in Table \ref{table:node_aug_stats}, \ref{table:graph_aug_stats} lie within a very narrow margin of the most efficient techniques such as EdgeDropping, NodeDropping and RandomWalkSubgraph. These low latencies are achieved by designing the augmentor in C++ and minimizing the interaction with Python APIs. Additionally, we compile Eigen without MKL optimizations and also do not support a GPU kernel for $rLap$ yet, which explains the lack of performance improvement when hardware acceleration is available. However, we tend to achieve minimal resource consumption by just leveraging the C++ APIs to the fullest. We leave the development of GPU support for $rLap$ as future work.

\begin{table}[H]
\centering
\caption{Statistics of resource consumption by augmentors on node classification datasets.}
\label{table:node_aug_stats}
\vskip 0.15in
\begin{center}
\begin{small}
\begin{sc}
\begin{tabular}{llll|l}
\toprule
         augmentor &      dataset &                memory &           latency(cpu) &          latency(w/ gpu) \\
\midrule

      EdgeAddition &         CORA &       $3.23 \pm 0.09$ &      $0.0024 \pm 0.0$ &  $0.0033 \pm 0.0001$ \\
      EdgeDropping &         CORA &     $0.51 \pm 0.0539$ &      $0.0004 \pm 0.0$ &  $0.0019 \pm 0.0001$ \\
EdgeDroppingDegree &         CORA &     $3.85 \pm 1.1307$ &     $0.003 \pm 0.001$ &  $0.0046 \pm 0.0032$ \\
   EdgeDroppingEVC &         CORA &      $39.58 \pm 0.14$ &   $0.2214 \pm 0.0038$ &  $0.3542 \pm 0.0021$ \\
    EdgeDroppingPR &         CORA &     $3.18 \pm 0.0872$ &   $0.0023 \pm 0.0001$ &     $0.0037 \pm 0.0001$ \\
   MarkovDiffusion &         CORA &  $228.73 \pm 38.1103$ &   $0.7966 \pm 0.1115$ &  $0.2539 \pm 0.0519$ \\
      NodeDropping &         CORA &     $1.31 \pm 1.3729$ &      $0.0006 \pm 0.0$ &  $0.0026 \pm 0.0008$ \\
      PPRDiffusion &         CORA &    $140.0 \pm 0.1342$ &   $0.2531 \pm 0.0017$ &  $0.5933 \pm 0.0589$ \\
RandomWalkSubgraph &         CORA &      $3.26 \pm 0.102$ &       $0.001 \pm 0.0$ &  $0.0027 \pm 0.0001$ \\
         
              rLap &         CORA &       $0.38 \pm 0.04$ &   $0.0037 \pm 0.0005$ &  $0.0048 \pm 0.0002$ \\
              \midrule
              EdgeAddition & AMAZON-PHOTO &    $20.94 \pm 5.4471$ &   $0.0368 \pm 0.0017$ &  $0.0071 \pm 0.0003$ \\
      EdgeDropping & AMAZON-PHOTO &     $5.64 \pm 0.5953$ &   $0.0032 \pm 0.0015$ &  $0.0022 \pm 0.0001$ \\
EdgeDroppingDegree & AMAZON-PHOTO &    $26.13 \pm 4.2223$ &   $0.0391 \pm 0.0007$ &  $0.0041 \pm 0.0001$ \\
   EdgeDroppingEVC & AMAZON-PHOTO &      $2264.1 \pm 0.004$ &       $1.554 \pm 0.012$ &   $1.4576 \pm 0.024$ \\
    EdgeDroppingPR & AMAZON-PHOTO &    $11.32 \pm 1.4999$ &   $0.0094 \pm 0.0001$ &     $0.0043 \pm 0.0003$ \\
   MarkovDiffusion & AMAZON-PHOTO &   $158.42 \pm 3.0019$ &   $30.6898 \pm 0.127$ &  $1.7687 \pm 0.0044$ \\
      NodeDropping & AMAZON-PHOTO &       $1.53 \pm 0.19$ &   $0.0017 \pm 0.0012$ &  $0.0028 \pm 0.0001$ \\
      PPRDiffusion & AMAZON-PHOTO &   $389.33 \pm 2.1808$ &   $1.6942 \pm 0.0247$ &  $0.8229 \pm 0.0112$ \\
RandomWalkSubgraph & AMAZON-PHOTO &     $7.26 \pm 0.1356$ &   $0.0024 \pm 0.0001$ &  $0.0029 \pm 0.0007$ \\
              rLap & AMAZON-PHOTO &     $32.6 \pm 0.5119$ &   $0.0848 \pm 0.0023$ &  $0.0891 \pm 0.0014$ \\
              \midrule

      EdgeAddition &       PUBMED &    $10.93 \pm 0.7785$ &   $0.0126 \pm 0.0002$ &  $0.0045 \pm 0.0001$ \\
      EdgeDropping &       PUBMED &     $1.89 \pm 0.1136$ &    $0.002 \pm 0.0003$ &  $0.0022 \pm 0.0001$ \\
EdgeDroppingDegree &       PUBMED &    $15.45 \pm 2.1238$ &     $0.015 \pm 0.001$ &  $0.0032 \pm 0.0003$ \\
   EdgeDroppingEVC &       PUBMED &  $2236.23 \pm 0.7537$ &   $1.2787 \pm 0.0152$ &  $1.3101 \pm 0.0156$ \\
    EdgeDroppingPR &       PUBMED &     $4.79 \pm 0.2071$ &   $0.0059 \pm 0.0005$ &  $0.0043 \pm 0.0003$ \\
   MarkovDiffusion &       PUBMED &   $152.87 \pm 2.0995$ &  $44.4244 \pm 0.1306$ &   $2.5587 \pm 0.0054$ \\
      NodeDropping &       PUBMED &      $1.78 \pm 0.172$ &   $0.0021 \pm 0.0002$ &     $0.0036 \pm 0.0002$ \\
      PPRDiffusion &       PUBMED &   $438.42 \pm 1.8595$ &  $19.0975 \pm 0.3605$ &  $2.5488 \pm 0.0135$ \\
RandomWalkSubgraph &       PUBMED &     $1.74 \pm 0.1625$ &   $0.0044 \pm 0.0003$ &   $0.003 \pm 0.0001$ \\
              rLap &       PUBMED &    $10.59 \pm 3.1908$ &   $0.0269 \pm 0.0008$ &  $0.0315 \pm 0.0039$ \\
\midrule
      EdgeAddition &  COAUTHOR-CS &    $16.71 \pm 3.0409$ &   $0.0286 \pm 0.0015$ &  $0.0057 \pm 0.0003$ \\
      EdgeDropping &  COAUTHOR-CS &     $3.02 \pm 0.5758$ &    $0.007 \pm 0.0018$ &  $0.0025 \pm 0.0003$ \\
EdgeDroppingDegree &  COAUTHOR-CS &     $15.1 \pm 2.3078$ &   $0.0363 \pm 0.0014$ &  $0.0037 \pm 0.0001$ \\
   EdgeDroppingEVC &  COAUTHOR-CS &   $1044.6 \pm 0.5745$ &  $10.2603 \pm 0.0266$ &  $10.6018 \pm 0.1096$ \\
    EdgeDroppingPR &  COAUTHOR-CS &      $5.09 \pm 0.359$ &   $0.0114 \pm 0.0003$ &   $0.0052 \pm 0.001$ \\
   MarkovDiffusion &  COAUTHOR-CS & $177.9667 \pm 7.0002$ &  $66.5972 \pm 0.1857$ &  $3.3583 \pm 0.1758$ \\
      NodeDropping &  COAUTHOR-CS &     $2.33 \pm 0.1269$ &   $0.0035 \pm 0.0004$ &  $0.0033 \pm 0.0001$ \\
      PPRDiffusion &  COAUTHOR-CS &   $416.58 \pm 3.9992$ &  $46.9031 \pm 1.4907$ &  $2.2614 \pm 0.1293$ \\
RandomWalkSubgraph &  COAUTHOR-CS &     $5.63 \pm 0.1792$ &   $0.0066 \pm 0.0002$ &   $0.003 \pm 0.0001$ \\
  
              rLap &  COAUTHOR-CS &     $13.6 \pm 0.9798$ &   $0.0508 \pm 0.0007$ & $0.0602 \pm 0.002$ \\
              \midrule
      EdgeAddition & COAUTHOR-PHY &    $44.59 \pm 8.3973$ &   $0.0975 \pm 0.0048$ &  $0.0121 \pm 0.0019$ \\
      EdgeDropping & COAUTHOR-PHY &    $11.57 \pm 1.5685$ &   $0.0098 \pm 0.0009$ &      $0.002 \pm 0.0001$ \\
EdgeDroppingDegree & COAUTHOR-PHY &     $20.7 \pm 0.1483$ &   $0.1219 \pm 0.0025$ &     $0.0046 \pm 0.0003$ \\
   EdgeDroppingEVC & COAUTHOR-PHY &  $2474.86 \pm 0.4294$ &  $28.2138 \pm 0.0304$ & $28.9077 \pm 0.2304$ \\
    EdgeDroppingPR & COAUTHOR-PHY &    $16.96 \pm 1.8205$ &   $0.0311 \pm 0.0002$ &   $0.0041 \pm 0.002$ \\
   MarkovDiffusion & COAUTHOR-PHY &       $316.0 \pm 0.4$ & $316.1449 \pm 0.1584$ & $15.1789 \pm 0.2203$ \\
      NodeDropping & COAUTHOR-PHY &     $4.08 \pm 1.7831$ &   $0.0067 \pm 0.0004$ &   $0.0046 \pm 0.0001$ \\
      PPRDiffusion & COAUTHOR-PHY &     $695.0 \pm 2.397$ &  $124.368 \pm 2.9031$ &  $10.394 \pm 0.2053$ \\
RandomWalkSubgraph & COAUTHOR-PHY &     $7.85 \pm 0.9563$ &    $0.014 \pm 0.0003$ &  $0.0057 \pm 0.0047$ \\
            
              rLap & COAUTHOR-PHY &    $43.9 \pm 11.3004$ &    $0.191 \pm 0.0258$ &  $0.2119 \pm 0.0064$ \\
\bottomrule
\end{tabular}
\end{sc}
\end{small}
\end{center}
\vskip -0.1in
\end{table}

\begin{table}[H]
\centering
\caption{Statistics of resource consumption by augmentors on graph classification datasets.}
\label{table:graph_aug_stats}
\vskip 0.15in
\begin{center}
\begin{small}
\begin{sc}
\begin{tabular}{llll|l}
\toprule
         augmentor &      dataset &                memory &           latency(cpu) &          latency(w/ gpu) \\
\midrule
      EdgeAddition &  IMDB-BINARY &     $4.78 \pm 0.3628$ &   $0.0307 \pm 0.0003$ &  $0.0191 \pm 0.0121$ \\
      EdgeDropping &  IMDB-BINARY &       $2.15 \pm 0.75$ &   $0.0079 \pm 0.0038$ &  $0.0038 \pm 0.0004$ \\
EdgeDroppingDegree &  IMDB-BINARY &     $7.51 \pm 0.6978$ &   $0.0346 \pm 0.0002$ &   $0.0121 \pm 0.001$ \\
   EdgeDroppingEVC &  IMDB-BINARY &  $2101.35 \pm 2.3127$ &   $0.7977 \pm 0.0068$ &  $0.9294 \pm 0.0138$ \\
    EdgeDroppingPR &  IMDB-BINARY &     $2.27 \pm 0.0458$ &   $0.0184 \pm 0.0002$ &  $0.0166 \pm 0.0001$ \\
   MarkovDiffusion &  IMDB-BINARY &  $195.43 \pm 60.3246$ &   $8.2739 \pm 0.0466$ &   $0.4589 \pm 0.006$ \\
      NodeDropping &  IMDB-BINARY &      $2.03 \pm 0.064$ &   $0.0046 \pm 0.0001$ &  $0.0032 \pm 0.0001$ \\
      PPRDiffusion &  IMDB-BINARY &   $170.13 \pm 22.767$ &    $0.658 \pm 0.0068$ &  $0.5295 \pm 0.0105$ \\
RandomWalkSubgraph &  IMDB-BINARY &     $2.09 \pm 0.1136$ &   $0.0063 \pm 0.0001$ &  $0.0096 \pm 0.0016$ \\            
              rLap &  IMDB-BINARY &     $6.96 \pm 0.3583$ &   $0.0423 \pm 0.0027$ &  $0.0412 \pm 0.0009$ \\
              \midrule
      EdgeAddition &   IMDB-MULTI &     $5.25 \pm 1.0249$ &   $0.0359 \pm 0.0017$ &  $0.0193 \pm 0.0001$ \\
      EdgeDropping &   IMDB-MULTI &     $2.18 \pm 0.1327$ &      $0.0064 \pm 0.0$ &  $0.0046 \pm 0.0001$ \\
EdgeDroppingDegree &   IMDB-MULTI &     $8.72 \pm 1.1232$ &   $0.0365 \pm 0.0001$ &  $0.0157 \pm 0.0001$ \\
   EdgeDroppingEVC &   IMDB-MULTI &  $2119.91 \pm 2.3364$ &    $0.806 \pm 0.0104$ &  $0.8579 \pm 0.0062$ \\
    EdgeDroppingPR &   IMDB-MULTI &     $4.45 \pm 0.7159$ &   $0.0221 \pm 0.0003$ &  $0.0236 \pm 0.0004$ \\
   MarkovDiffusion &   IMDB-MULTI &   $28.02 \pm 97.7836$ &   $5.5597 \pm 0.0245$ &   $0.386 \pm 0.0046$ \\
      NodeDropping &   IMDB-MULTI &     $2.17 \pm 0.1005$ &   $0.0054 \pm 0.0001$ &  $0.0045 \pm 0.0012$ \\
      PPRDiffusion &   IMDB-MULTI &   $98.67 \pm 41.6684$ &   $0.3907 \pm 0.0183$ &  $0.7974 \pm 0.0134$ \\
RandomWalkSubgraph &   IMDB-MULTI &     $4.07 \pm 0.9144$ &   $0.0075 \pm 0.0001$ &  $0.0107 \pm 0.0006$ \\
              rLap &   IMDB-MULTI &     $7.93 \pm 0.6165$ &   $0.0421 \pm 0.0023$ &  $0.0415 \pm 0.0009$ \\
              \midrule
      EdgeAddition &        MUTAG &       $3.22 \pm 0.14$ &   $0.0018 \pm 0.0001$ &  $0.0046 \pm 0.0001$ \\
      EdgeDropping &        MUTAG &     $1.61 \pm 0.0943$ &      $0.0004 \pm 0.0$ &  $0.0025 \pm 0.0001$ \\
EdgeDroppingDegree &        MUTAG &     $4.01 \pm 0.7713$ &   $0.0021 \pm 0.0001$ &  $0.0047 \pm 0.0004$ \\
   EdgeDroppingEVC &        MUTAG &  $2118.92 \pm 5.3619$ &   $0.0445 \pm 0.0005$ &  $0.0683 \pm 0.0061$ \\
    EdgeDroppingPR &        MUTAG &     $3.35 \pm 0.1118$ &   $0.0025 \pm 0.0001$ &   $0.006 \pm 0.0004$ \\
   MarkovDiffusion &        MUTAG &   $164.0 \pm 24.8769$ &   $0.3848 \pm 0.0067$ &  $0.1774 \pm 0.0075$ \\
      NodeDropping &        MUTAG &     $1.61 \pm 0.0943$ &      $0.0008 \pm 0.0$ &  $0.0033 \pm 0.0001$ \\
      PPRDiffusion &        MUTAG &   $117.04 \pm 0.2973$ &   $0.0829 \pm 0.0019$ &   $0.602 \pm 0.0101$ \\
RandomWalkSubgraph &        MUTAG &      $2.7 \pm 0.3821$ &      $0.0009 \pm 0.0$ &  $0.0044 \pm 0.0008$ \\          
              rLap &        MUTAG &     $1.62 \pm 0.3682$ &      $0.0021 \pm 0.0$ &  $0.0041 \pm 0.0001$ \\
              \midrule
      EdgeAddition &     PROTEINS &     $6.68 \pm 0.3124$ &   $0.0288 \pm 0.0014$ &  $0.0153 \pm 0.0013$ \\
      EdgeDropping &     PROTEINS &       $2.58 \pm 0.06$ &   $0.0058 \pm 0.0012$ &   $0.004 \pm 0.0001$ \\
EdgeDroppingDegree &     PROTEINS &    $10.48 \pm 1.6024$ &   $0.0319 \pm 0.0003$ &  $0.0127 \pm 0.0002$ \\
   EdgeDroppingEVC &     PROTEINS &  $2095.08 \pm 0.6554$ &   $1.1132 \pm 0.0209$ &  $1.3092 \pm 0.2126$ \\
    EdgeDroppingPR &     PROTEINS &      $5.27 \pm 0.064$ &   $0.0195 \pm 0.0018$ &  $0.0182 \pm 0.0001$ \\
   MarkovDiffusion &     PROTEINS &  $203.03 \pm 84.7664$ &  $25.7943 \pm 0.1671$ &  $1.1789 \pm 0.0087$ \\
      NodeDropping &     PROTEINS &      $2.0 \pm 0.6678$ &   $0.0079 \pm 0.0019$ &  $0.0062 \pm 0.0001$ \\
      PPRDiffusion &     PROTEINS &  $295.79 \pm 26.6581$ &   $4.8792 \pm 0.0261$ &   $1.6547 \pm 0.549$ \\
RandomWalkSubgraph &     PROTEINS &      $2.23 \pm 0.064$ &   $0.0074 \pm 0.0016$ &  $0.0099 \pm 0.0001$ \\
              rLap &     PROTEINS &     $6.45 \pm 0.3442$ &   $0.0405 \pm 0.0024$ &  $0.0407 \pm 0.0009$ \\
              \midrule
      EdgeAddition &         NCI1 &     $4.26 \pm 0.5122$ &   $0.0517 \pm 0.0006$ &   $0.046 \pm 0.0015$ \\
      EdgeDropping &         NCI1 &     $1.72 \pm 0.6447$ &   $0.0108 \pm 0.0001$ &  $0.0085 \pm 0.0002$ \\
EdgeDroppingDegree &         NCI1 &     $7.04 \pm 0.6296$ &   $0.0574 \pm 0.0013$ &  $0.0389 \pm 0.0009$ \\
   EdgeDroppingEVC &         NCI1 &  $2071.43 \pm 5.9791$ &   $1.9465 \pm 0.1448$ &  $2.3067 \pm 0.0682$ \\
    EdgeDroppingPR &         NCI1 &     $4.11 \pm 0.1136$ &   $0.0548 \pm 0.0064$ &  $0.0603 \pm 0.0017$ \\
   MarkovDiffusion &         NCI1 &   $141.825 \pm 6.632$   &  $37.7575 \pm 0.3818$ &  $2.0278 \pm 0.0061$   \\
      NodeDropping &         NCI1 &     $2.22 \pm 0.1249$ &    $0.019 \pm 0.0027$ &  $0.0219 \pm 0.0008$ \\
      PPRDiffusion &         NCI1 &   $157.16 \pm 8.8786$ &   $8.8779 \pm 0.0595$ &   $2.4582 \pm 0.011$ \\
RandomWalkSubgraph &         NCI1 &     $2.68 \pm 0.5862$ &   $0.0181 \pm 0.0008$ &  $0.0229 \pm 0.0014$ \\
              rLap &         NCI1 &     $2.82 \pm 0.5474$ &    $0.074 \pm 0.0003$ &  $0.0801 \pm 0.0011$ \\
\bottomrule
\end{tabular}
\end{sc}
\end{small}
\end{center}
\vskip -0.1in
\end{table}
\pagebreak

\section{Additional Experiments}
\label{app:add_exp}

\subsection{PPR Diffusion and rLap}

In our experiments, we observed that PPRDiffusion on medium scale graphs and MVGRL design leads to OOM and requires sub-graph sampling to proceed with the GCL phases. To this end, we observed a decrease in unsupervised node classification performance when the size of the sub-graph decreased (see Table \ref{table:results_mvgrl_subgraph_size}). As the nodes for a sub-graph are randomly selected, one can't guarantee that the sub-graph will capture useful structural information for the GNN encoders to exploit. In a worst case scenario, if all the nodes of a sub-graph are disconnected, the GNN encoder essentially acts as an MLP. Thus, as the size of the sub-graph reduces, we can expect a heavy corruption of topological information in the augmented views.

\begin{table*}[ht!]
\centering
\caption{Evaluation (in accuracy) on benchmark node datasets with \textbf{MVGRL} based design and PPR Diffusion with varying sub-graph sizes. PPRDiffusion indicates the full graph view and PPRDiffusion-$x$ indicates a sampled sub-graph with $x$ nodes as the graph view. }
\label{table:results_mvgrl_subgraph_size}
\vskip 0.15in
\begin{center}
\begin{small}
\begin{sc}
\begin{tabular}{c|c|c|c|c|c}
\toprule
Augmentor & CORA & Amazon-Photo & PubMed & Coauthor-CS & Coauthor-Phy \\
\midrule
PPRDiffusion & $84.12 \pm 2.72$  & $90.65 \pm 1.01$  & OOM & OOM & OOM \\
PPRDiffusion-$8192$ & $-$  & $-$  & $82.7 \pm 0.85$ & $90.9 \pm 1.06$ & $94.03 \pm 0.5$ \\
PPRDiffusion-$4096$ & $-$ & $87.38 \pm 0.91$ & $80.79 \pm 1.4$ & $89.0 \pm 0.83$ & $92.52 \pm 0.39$ \\
PPRDiffusion-$2048$ & $80.93 \pm 2.33$ & $85.65 \pm 1.41$ & $76.4 \pm 0.88$ & $88.43 \pm 0.77$ & $89.24 \pm 0.69$ \\
\bottomrule
\end{tabular}
\end{sc}
\end{small}
\end{center}
\vskip -0.1in
\end{table*}

\subsubsection{Reducing computational overheads}

For graphs such as COAUTHOR-PHY, the computational overheads of PPRDiffusion are relatively significant when compared to other augmentors (see Table \ref{table:node_aug_stats}). The standard practice to avoid such bottlenecks is to cache the output of the diffusion operator $\mS^{PPR}$ and reuse it for sampling sub-graphs. To further reduce the overheads of computing $\mS^{PPR}$ prior to caching, we leverage Theorem \ref{thm:theta_schur}.

\textbf{Sketch:} The COAUTHOR-PHY graph $\gG = (\gV, \gE)$ represents $|\gV| = 34,493$ nodes and $|\gE| = 495,924$ edges. To randomly sample a sub-graph of size $8192$ from its diffused variant $\mS^{PPR}$, one can compute a randomized Schur complement $\mR_{\gamma|\gV|}$ using $rLap$ with $\gamma=0.5$, apply the PPR diffusion operator on $\mR_{\gamma|\gV|}$ to obtain $\widetilde{\mS}^{PPR}_{\gamma}$ and cache it. $\widetilde{\mS}^{PPR}_{\gamma}$ represents $\approx 17,000$ nodes from which a sub-graph of size $8192$ can be sampled with relatively less overheads.

We benchmark the computational overheads of:

\begin{itemize}
    \item PPRDiffusion on COAUTHOR-PHY graph followed by random sampling of a sub-graph with $8192$ nodes.
    \item $rLap$ on COAUTHOR-PHY graph with $\gamma=0.5, o_v=rand, o_n=asc$ followed by PPRDiffusion and random sampling of a sub-graph with $8192$ nodes.
\end{itemize}

The results in Figure \ref{fig:ppr_rlapppr_overheads} demonstrate the effectiveness of the latter approach in achieving $\approx 0.25 \times$ reduction in memory consumption, $\approx 10 \times$ reduction in CPU only latencies and $\approx 5 \times$ reduction in latencies when GPU is available. We also compared the unsupervised node classification accuracy and observed a slight performance improvement of $\approx 1\%$. 

\textbf{Limitations and Tradeoffs:} From Algorithm \ref{alg:rlap}, notice that after every outer loop iteration, the number of edges in the graph representing the current Schur complement state will be less than or equal to the number of edges at the beginning of the iteration. To quantify the impact of diffusion on these randomized Schur complements, we measured the edge counts of sub-graphs of size $8192$ randomly sampled from the following graphs \footnote{We consider the $ o_v=rand, o_n=asc$ variant of $rLap$, which is the default setting. }:

\begin{itemize}
    \item \textbf{COAUTHOR-CS:PPR}: PPRDiffusion applied on COAUTHOR-CS.
    \item \textbf{COAUTHOR-CS:rLapPPR}: $rLap$ on COAUTHOR-CS with $\gamma=0.5$ followed by PPRDiffusion.
    \item \textbf{COAUTHOR-PHY:PPR}: PPRDiffusion applied on COAUTHOR-PHY.
    \item \textbf{COAUTHOR-PHY:rLapPPR}: $rLap$ on COAUTHOR-PHY with $\gamma=0.5$ followed by PPRDiffusion.
\end{itemize}

The results in Figure \ref{fig:rlap_rlapppr_subgraph_edge_0_5} show that the randomly sampled sub-graph from \textbf{COAUTHOR-CS:rLapPPR} has $\approx 2.5\times$ the number of edges in \textbf{COAUTHOR-CS:PPR}. Similarly, the sub-graph from \textbf{COAUTHOR-PHY:rLapPPR} has $\approx 1.5\times$ the number of edges in \textbf{COAUTHOR-PHY:PPR}. Thus, although one can save memory and time to address the bottlenecks of PPRDiffusion, the GNN encoders would eventually consume additional memory due to the message passing operations on relatively more edges. 

In such situations, if training latencies take precedence over memory constraints, then applying $rLap$ and proceeding with PPRDiffusion will lead to better training throughput. Else, based on the edge count pattern in Figure \ref{fig:rlap_rlapppr_subgraph_edge_gamma_variants}, the practitioner can use a lower $\gamma$ value and consider the trade-offs between augmentation latencies and available memory as per their setup.

\begin{figure}[ht]
\vskip 0.2in
\begin{center}
\begin{tabular}{cc}
\centering
  \includegraphics[width=70mm]{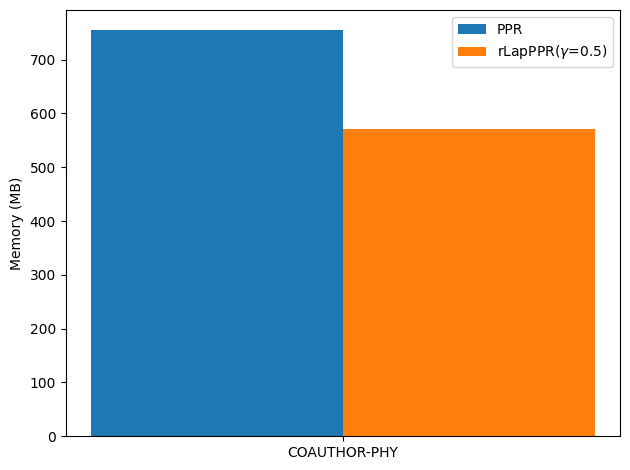} &   \includegraphics[width=70mm]{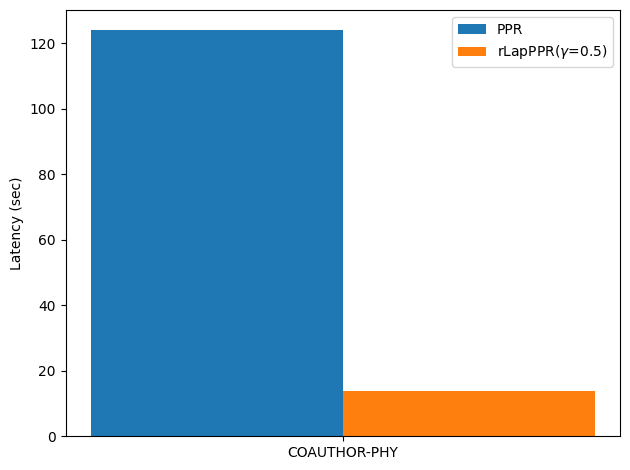} \\
(a) Memory (MB) overhead & (b) Latency (sec) overhead  \\[2pt]
 \includegraphics[width=70mm]{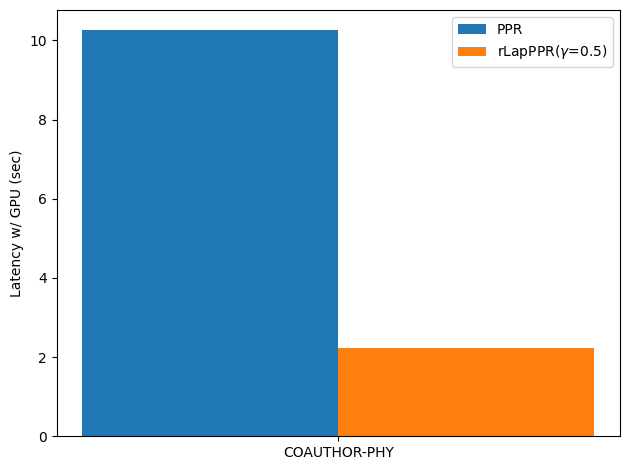} &   \includegraphics[width=70mm]{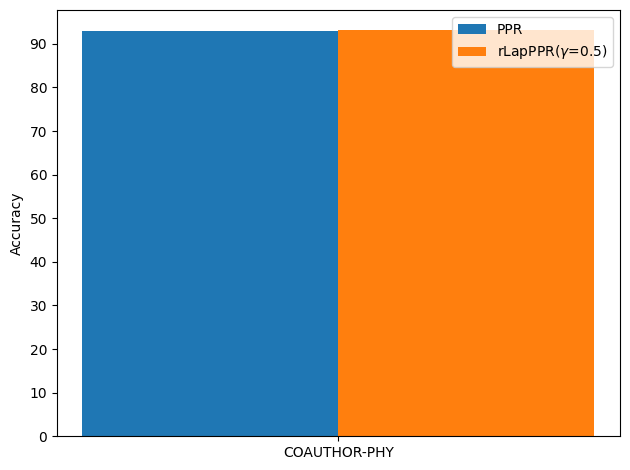} \\
(c)  Latency w/ GPU (sec) overhead  & (d) Node classification accuracy \\[2pt]
\end{tabular}
\caption{Computational overheads and unsupervised node classification accuracy of PPRDiffusion and $rLap$ + PPRDiffusion with $\gamma=0.5$ on COAUTHOR-PHY dataset.}
\label{fig:ppr_rlapppr_overheads}
\end{center}
\vskip -0.2in
\end{figure}

\begin{figure}[ht]
\vskip 0.2in
\begin{center}
\centerline{\includegraphics[width=0.5\textwidth]{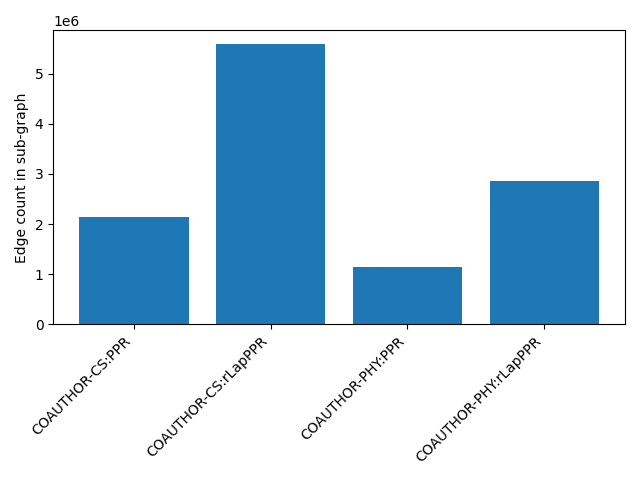}}
\caption{Plots of edge counts of randomly sampled sub-graphs of size $8192$ with $\gamma=0.5$ for $rLap$ based techniques.} 
\label{fig:rlap_rlapppr_subgraph_edge_0_5}
\end{center}
\vskip -0.2in
\end{figure}

\begin{figure}[ht]
\vskip 0.2in
\begin{center}
\begin{tabular}{cc}
\centering
  \includegraphics[width=70mm]{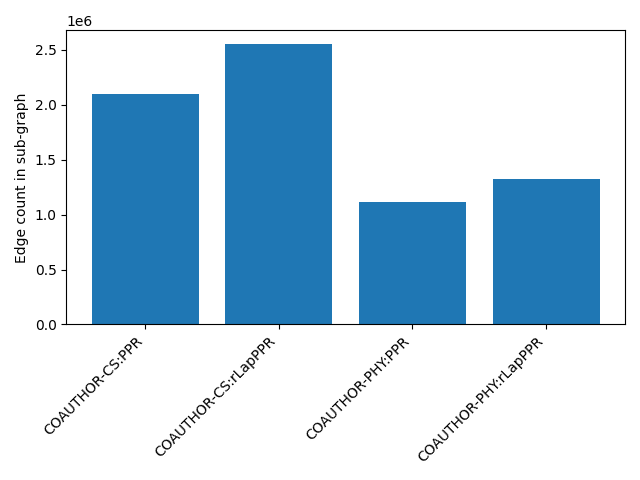} &   \includegraphics[width=70mm]{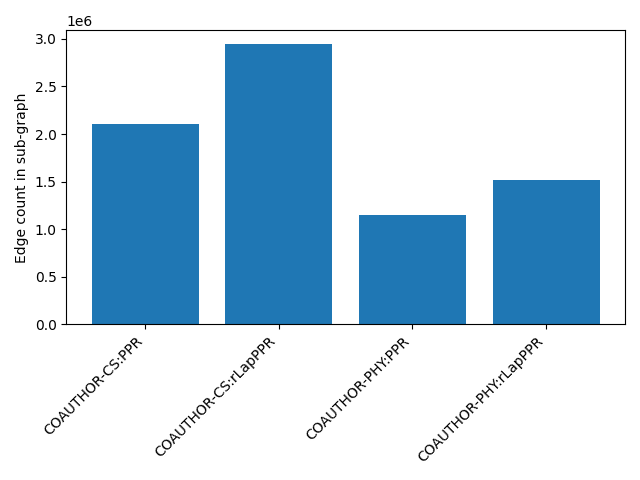} \\
(a) $\gamma=0.1$ & (b) $\gamma=0.2$  \\[2pt]
 \includegraphics[width=70mm]{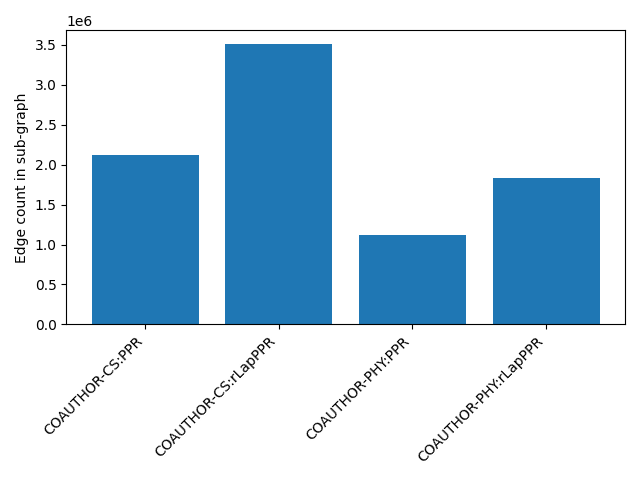} &   \includegraphics[width=70mm]{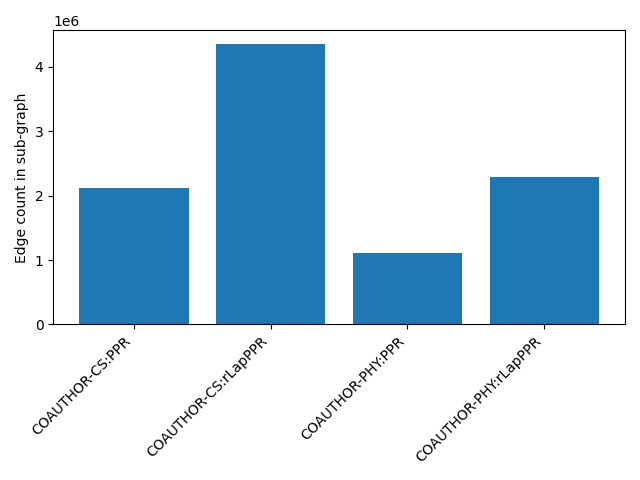} \\
(c)  $\gamma=0.3$  & (d) $\gamma=0.4$ \\[2pt]
\end{tabular}
\caption{Plots of edge counts of randomly sampled sub-graphs of size $8192$ with varying $\gamma$ for $rLap$ based techniques.}
\label{fig:rlap_rlapppr_subgraph_edge_gamma_variants}
\end{center}
\vskip -0.2in
\end{figure}


\end{document}